\documentclass{article}
\usepackage[letterpaper, left=1in, right=1in, top=1in, bottom=1in]{geometry}
\usepackage{authblk}

\PassOptionsToPackage{hypertexnames=false}{hyperref}  % compatible with cref for multiple algo
\usepackage{parskip}

\usepackage[colorlinks=true, linkcolor=blue!70!black, citecolor=blue!70!black,urlcolor=black,breaklinks=true]{hyperref}
\usepackage{microtype}
\usepackage{hhline}

\makeatletter
\newcommand{\neutralize}[1]{\expandafter\let\csname c@#1\endcsname\count@}
\makeatother

\usepackage{algorithm}

\usepackage{natbib}
\bibpunct{(}{)}{;}{a}{,}{,}

\usepackage{amsthm}
\usepackage{mathtools}
\usepackage{amsmath}
\usepackage{bbm}
\usepackage{amsfonts}
\usepackage{amssymb}

\usepackage{xpatch}

\newtheorem*{theorem*}{Theorem}
\newtheorem*{lemma*}{Lemma}
\newtheorem{theorem}{Theorem}
\newtheorem{lemma}{Lemma}
\newtheorem{proposition}[theorem]{Proposition}

\newtheorem{example}{Example}
\theoremstyle{remark}

\theoremstyle{definition}

\newtheorem{remark}{Remark}

\AtBeginEnvironment{remark}{%
  \pushQED{\qed}%
}
\AtEndEnvironment{remark}{\popQED\endexample}

\makeatletter
  \renewenvironment{proof}[1][Proof]%
  {%
   \par\noindent{\bfseries\upshape {#1.}\ }%
  }%
  {\qed\newline}
  \makeatother

\xpatchcmd{\proof}{\itshape}{\normalfont\proofnameformat}{}{}
\newcommand{\proofnameformat}{\bfseries}

\usepackage[nameinlink,capitalize]{cleveref}

\crefformat{equation}{#2(#1)#3}
\Crefformat{equation}{#2(#1)#3}
\usepackage{mdframed}
\usepackage{lipsum}
\newcommand{\Pp}{\mathbb P}

\Crefformat{figure}{#2Figure #1#3}
\Crefname{assumption}{Assumption}{Assumptions}
\Crefformat{assumption}{#2Assumption #1#3}
\usepackage[customcolors]{hf-tikz}
\usepackage{crossreftools}
\Crefformat{figure}{#2Figure #1#3}
\Crefname{assumption}{Assumption}{Assumptions}
\Crefformat{assumption}{#2Assumption #1#3}
\usepackage[customcolors]{hf-tikz}
\usepackage{crossreftools}
\usepackage{xparse}

\ExplSyntaxOn
\DeclareDocumentCommand{\XDeclarePairedDelimiter}{mm}
 {
  \__egreg_delimiter_clear_keys: % reset to the default
  \keys_set:nn { egreg/delimiters } { #2 }
  \use:x % we want to expand the values of the token variables set with the keys
   {
    \exp_not:n {\NewDocumentCommand{#1}{sO{}m} }
     {
      \exp_not:n { \IfBooleanTF{##1} }
       {
        \exp_not:N \egreg_paired_delimiter_expand:nnnn
         { \exp_not:V \l_egreg_delimiter_left_tl }
         { \exp_not:V \l_egreg_delimiter_right_tl }
         { \exp_not:n { ##3 } }
         { \exp_not:V \l_egreg_delimiter_subscript_tl }
       }
       {
        \exp_not:N \egreg_paired_delimiter_fixed:nnnnn 
         { \exp_not:n { ##2 } }
         { \exp_not:V \l_egreg_delimiter_left_tl }
         { \exp_not:V \l_egreg_delimiter_right_tl }
         { \exp_not:n { ##3 } }
         { \exp_not:V \l_egreg_delimiter_subscript_tl }
       }
     }
   }
 }

\keys_define:nn { egreg/delimiters }
 {
  left      .tl_set:N = \l_egreg_delimiter_left_tl,
  right     .tl_set:N = \l_egreg_delimiter_right_tl,
  subscript .tl_set:N = \l_egreg_delimiter_subscript_tl,
 }

\cs_new_protected:Npn \__egreg_delimiter_clear_keys:
 {
  \keys_set:nn { egreg/delimiters } { left=.,right=.,subscript={} }
 }

\cs_new_protected:Npn \egreg_paired_delimiter_expand:nnnn #1 #2 #3 #4
 {% Fix the spacing issue with \left and \right (D. Arsenau, P. Stephani and H. Oberdiek)
  \mathopen{}
  \mathclose\c_group_begin_token
   \left#1
   #3
   \group_insert_after:N \c_group_end_token
   \right#2
   \tl_if_empty:nF {#4} { \c_math_subscript_token {#4} }
 }
\cs_new_protected:Npn \egreg_paired_delimiter_fixed:nnnnn #1 #2 #3 #4 #5
 {
  \mathopen{#1#2}#4\mathclose{#1#3}
  \tl_if_empty:nF {#5} { \c_math_subscript_token {#5} }
 }
\ExplSyntaxOff

\XDeclarePairedDelimiter{\supnorm}{
  left=\lVert,
  right=\rVert,
  subscript=\infty
  }
\usepackage[utf8]{inputenc} % allow utf-8 input
\usepackage[T1]{fontenc}    % use 8-bit T1 fonts
\usepackage{url}            % simple URL typesetting
\usepackage{booktabs}       % professional-quality tables
\usepackage{amsfonts}       % blackboard math symbols
\usepackage{nicefrac}       % compact symbols for 1/2, etc.
\usepackage{microtype}      % microtypography
\usepackage{authblk}
\usepackage{tocloft}            % TOC spacing

\usepackage{enumitem}

\usepackage{breakcites}
\usepackage{dsfont}
\usepackage{etoolbox}
\usepackage{comment}
\newtoggle{draft}
\togglefalse{draft}

\usepackage{color-edits}
\usepackage{mathrsfs}

\usepackage{algorithm}
\usepackage{verbatim}
\usepackage[noend]{algpseudocode}

\usepackage{multicol}

\usepackage{colortbl}
\usepackage{bbm}
\usepackage{setspace}

\usepackage{transparent}

\usepackage{inconsolata}
\usepackage[scaled=.90]{helvet}
\usepackage{xspace}

\usepackage{pifont}

\usepackage{tikz-cd}

    \newcommand{\E}{{\mathbb E}}
    
    \newcommand{\R}{{\mathbb R}}

\newcommand{\X}{X}

    \renewcommand{\b}[1]{\boldsymbol{\mathbf{#1}}}

    \newcounter{rcnt}[section]

    \def\argmin{\mathop{\rm argmin}}

    \newcommand{\sign}{\operatorname{sign}}

\newcommand{\ST}{\operatorname{ST}}

    \newcommand{\ud}{\,\mathrm d}

\newcommand\blfootnote[1]{
    \begingroup
    \renewcommand\thefootnote{}\footnote{#1}
    \addtocounter{footnote}{-1}
    \endgroup
}

    \def\ddefloop#1{\ifx\ddefloop#1\else\ddef{#1}\expandafter\ddefloop\fi}
    \def\ddef#1{\expandafter\def\csname c#1\endcsname{\ensuremath{\mathcal{#1}}}}
    \ddefloop ABCDEFGHIJKLMNOPQRSTUVWXYZ\ddefloop
\title{Grokking through the Lens of Minimum-Norm Interpolation}
\author[1]{Gil Kur*}
\author[2]{Ileana Rugina}
\author[2,3]{Clémentine Carla Juliette Dominé}
\author[2]{Marco Mondelli*}
\affil[1]{ETH Z\"urich}
\affil[2]{Institute of Science and Technology Austria}
\affil[3]{Harvard University}
\begin{document}
\date{}
\maketitle

\begin{abstract}
Grokking shows that fitting the training data and learning the underlying
signal can occur at very different stages. However, existing theories offer limited quantitative insight into how this delayed generalization depends on inductive bias and signal structure. Our work addresses the gap by developing a statistical theory that characterizes how regularization geometry and signal sparsity govern generalization near interpolation. In particular, we focus on the prototypical setting of high-dimensional regression and identify regimes in which sparsity-promoting regularization makes exact interpolation much more
accurate than approximate fitting. In strongly overparameterized noiseless
problems, we prove a zero--one generalization law and construct a
family of convex norms whose interpolators transition from the trivial risk of
the all-zero predictor to exact recovery, while keeping the training error equal to $0$. Furthermore,
when feature dimension and sample size are proportional, we provide a precise characterization of training and generalization errors along $\ell_r$-regularization paths. This in turn allows us to quantify the
generalization gain that remains near interpolation: we show that this gain increases as the norm becomes more sparsity-promoting
and as the target becomes sparser, with a sharp drop in generalization reached for noiseless data and $\ell_1$ regularization. Experiments on diagonal linear networks and transformers trained on modular arithmetic demonstrate the generality of our theoretical predictions. Finally, beyond grokking, our work reveals a statistical instability in minimum-norm
interpolation: small perturbations in the regularization strength can lead to drastically different generalization, while preserving small training error.\blfootnote{* $=$ Equal contribution}\blfootnote{Emails:} \blfootnote{1. \texttt{gilkur@mit.edu}},\blfootnote{2. \texttt{Ileana.Rugina@ist.ac.at}} \blfootnote{3. \texttt{cdomine@fas.harvard.edu}}\blfootnote{4. \texttt{marco.mondelli@ist.ac.at}}
% Numerical experiments in regularized regression and diagonal linear networks support these orderings. Transformers trained on modular arithmetic show qualitatively similar changes in late-stage generalization when their output scale is varied. These results connect generalization near interpolation to norm geometry and signal structure, offering a statistical account of an important aspect of grokking.
\end{abstract}

\section{Introduction}

Overparameterized models can fit their training data exactly, even in
the presence of noise, and nevertheless generalize well
\citep{belkin2018understand,muthukumar2020harmless}.
This phenomenon, known as benign overfitting, has been rigorously
studied for minimum-norm interpolators (MNIs) in linear regression
\citep{bartlett2020benign,tsigler2023benign}
and reproducing kernel Hilbert spaces
\citep{liang2020just,liang2020multiple},
clarifying conditions under which interpolation is compatible with accurate prediction. Grokking provides an extreme manifestation of this compatibility
between interpolation and generalization: a model attains
near-zero training error yet generalizes poorly for many additional
training steps; then, the test error drops sharply while the
training error remains close to  zero \citep{power2022grokking,nanda2023progress}.
Initially observed in small algorithmic problems, the phenomenon has been reported in image, language, and
molecular prediction tasks \citep{liu2023omnigrok,humayun2024deep}.
This raises the question of which properties of the
data and the learning procedure allow solutions with almost
indistinguishable training errors to have dramatically different generalization.

Existing theoretical work explains grokking by analyzing linear estimators and classification
\citep{levi2024grokking,beck2024grokking},
transitions between lazy and rich training
\citep{kumar2024grokking,lyu2024dichotomy,mohamadi2024why}, and slow norm
minimization under weight decay \citep{boursier2025theoretical}. However, 
turning these explanations into an end-to-end guarantee
requires proving early fitting, a prolonged period of poor
generalization, and a final improvement in the test error.
\citet{xu2026grok} establish such guarantees, together with
quantitative bounds on generalization
delay.
Their analysis nevertheless concerns over-parameterized linear
regression with $\ell_2$ regularization.

This restriction is significant since nonlinear parameterization
fundamentally changes the bias of training.
For example, in diagonal linear networks, the scale and shape of
initialization determine an implicit regularizer that is a hybrid  between ridge $\ell_2$-like and sparsity-promoting $\ell_1$-like
geometries \citep{woodworth2020kernel,azulay2021implicit}.
Furthermore, existing analyses of gradient flow in diagonal linear networks with
vanishing initialization establish convergence to the
$\ell_1$-MNI \citep{pesme2023saddle}
and connect training dynamics to
$\ell_1$-regularization paths \citep{berthier2023incremental,berthier2025diagonal}. Recently, \citet{tikeng2025grokking} have connected grokking to
sparsity- and low-rank-promoting regularization.
These results establish the importance of matching regularization
to the structure of the target, but leave open how this match
governs the sharpness of generalization near interpolation.
In particular, a quantitative account of how much generalization
can improve for a small reduction in training error, and how
this improvement varies with the regularizing norm and signal
sparsity, is still missing.

To address this gap, we develop a statistical theory that studies generalization
along regularization paths induced by different norms, and we  compare
near-interpolating solutions with the corresponding MNI. More precisely, we consider a linear model and the regularized square-root loss, which allows the optimum to be an exact interpolator when the regularization is below a critical threshold (see \eqref{eq:minprob} for details). Our results establish when approaching exact interpolation
produces a sharp drop in generalization error, and
show how this behavior depends on the
geometry of the norm and the sparsity of the problem. These theoretical predictions then shed light into late-stage generalization under non-linear training of diagonal linear networks and transformer models. We summarize our main contributions as follows:

\begin{itemize}[itemsep=2pt,leftmargin=*]

 \item \textbf{A zero--one law for generalization, under exact interpolation.}
 When the number of features significantly exceeds the sample size, we
 construct a family of convex $\ell_1$--$\ell_\infty$ hybrid norms
 for noiseless sparse regression. As the norm geometry varies, we show that
 the generalization performance exhibits a zero--one
 transition between exact recovery and trivial error,
 while every minimizer keeps interpolating exactly
 (Theorem~\ref{thm:cube-grokking}). %The construction builds on a corresponding zero--one law for ordinary $\ell_1$ regularization (Lemma~\ref{thm:0-1grok}). 
 This gives a sharp separation between
 fitting the training data and recovering the underlying signal.

 \item \textbf{Quantifying the roles of norm geometry and sparsity.}
 When the number of features scales proportionally to the sample
 size, we use the convex Gaussian min--max theorem (CGMT) to
 characterize training and generalization error of $\ell_r$-regularized
 square-root regression (Lemma~\ref{thm:main-nonint}), as well as the generalization error of $\ell_r$-MNIs
 (Lemma~\ref{thm:min-norm-interpolator}). Using this characterization, we provide quantitative conditions leading to a sharp drop in generalization near interpolation. In particular, we show that this drop is more prominent \emph{(i)} as the norm passes from $\ell_2$ to $\ell_1$ (Theorem~\ref{thm:sparse-decreasing-slope}), with the most extreme example for $\ell_1$ regularization and noiseless data (Theorem \ref{thm:r1-critical-infinite-slope}), and \emph{(ii)} as the regression problem becomes sparser (Theorem~\ref{prop:sparsity-monotonicity}). %The construction builds on a short Gaussian-comparison proof at the leading basis-pursuit weak-recovery scale \citep{donoho2009}.

\item \textbf{Evidence from training neural networks.}
We test these predictions in %norm-regularized regression and 
two-layer diagonal linear networks, whose nonlinear parameterization
induces an initialization-dependent implicit bias.
The results support the predicted ordering of late-stage
generalization gains: stronger sparsity-promoting bias and sparser
targets yield larger improvements near interpolation.
Experiments with transformers trained on modular arithmetic also show larger gains in generalization at smaller output
scales, qualitatively consistent with this picture (Figure~\ref{fig:dln_gain_monotone}).

\end{itemize}

Beyond grokking, our results highlight a surprising statistical instability of the regularized square-root loss around its corresponding MNI: small perturbations of the norm geometry yield a zero--one generalization transition while preserving an exact fit (Theorem \ref{thm:cube-grokking}) and, along the $\ell_1$ regularization path, generalization changes much faster than training error near interpolation (Theorem \ref{thm:r1-critical-infinite-slope}).

  %Furthermore, our result shows that both the $\ell_1$ and the $\ell_2$ MNIs behave differently from the $\ell_r$ MNI when $r \in (1,2)$, a surprising result. Specifically, in the proportional regime, grokking with high noise can occur for the $\ell_r$ MNI only when $r \in (1,2)$. \gk{marco please add a remark on this for the $\ell_1$-MNI if the reviewers would ask a prove we can do that in the rebutall.}
%The construction further implies that arbitrarily small, fixed
%multiplicative perturbations of the regularizing norm can change
%the asymptotic normalized risk from zero to one on the same data.
%This reveals a statistical instability of minimum-norm interpolation:
%

\section{Related work}

\paragraph{Grokking, implicit bias, and regularization.}
Initialization scale, layer imbalance, and target or output
rescaling shape feature learning
\citep{chizat2019lazy,saxe2013exact, domine2024lazy,jarvis2025theory,anguita2026theory,kumar2024grokking,nam2025position}.
In particular, \citet{kunin2024get, kumar2024grokking} show that unbalanced
initializations and output rescaling  can accelerate feature learning and reduce
grokking delays in transformers. 
Data availability also influences learning regime and grokking profile: a ``Goldilocks zone'' allows
eventual structured generalization while retaining an initial
memorization phase \citep{liu2022towards,kumar2024grokking}.
For modular addition, \citet{mohamadi2024why} establish a
sample-complexity separation between permutation-equivariant
kernel methods and quadratic networks. 
Early oscillations in the training loss have been used to predict
grokking \citep{notsawo2023predicting}, while \citet{jeffares2024deep}
connect its emergence empirically to benign overfitting. Related work has also focused on sparse subnetworks \citep{merrill2023tale}, norm-efficient circuits \citep{varma2023explaining}, and
low-rank
weights \citep{yunis2024approaching}.
Focusing on regularization, \citet{tikeng2025grokking} analyze
an initial loss-minimization phase followed by a slower
dynamics driven by regularization, with applications to sparse recovery regularized via the
$\ell_1$ norm and matrix factorization regularized via the nuclear norm. %Additional related work on grokking focuses on sparse subnetworks \citep{merrill2023tale}, %norm-efficient circuits \citep{varma2023explaining}, and
%low-rank
%weights \citep{yunis2024approaching} and the comparison of 
$\ell_1$ and $\ell_2$ regularization are compared by \cite{zunkovic2024grokking} in the context of linear classification.
%\citep{zunkovic2024grokking}. \citet{xu2024benign} prove early overfitting followed by near-optimal generalization in ReLU networks trained on noisy XOR
%cluster data.
Most recently, \citet{xu2026grok} quantify how initialization
and weight decay affect the generalization delay for ridge regression.
Our work goes beyond ridge regularization: motivated by explicit
characterizations of implicit bias
\citep{woodworth2020kernel,azulay2021implicit},
we study regularization paths in high-dimensional regression
and quantify how generalization gains near interpolation
depend on norm geometry and signal sparsity.

% and signal sparsity. Diagonal-network experiments support the predicted orderings, while Transformers trained on modular arithmetic exhibit qualitatively consistent behavior.

%, which provide a simple and analytically tractable setting for examining how changes in the learning regime affect implicit bias. While lazy training favors dense minimum-\(\ell_2\)-norm solutions, rich training favors sparse minimum-\(\ell_1\)-norm solutions \cite{azulay2021implicit,woodworth2020kernel}.
%Expressing this implicit bias as an explicit regularizer \cite{anguita2026theory} reveals a smooth transition between these regimes, allowing us to control the onset of grokking in our regression problem. This transition is governed by the overall scale of the weights at initialization, as illustrated in Fig.~\ref{fig:LABEL}.

%\cite{gunasekar2018implicit}

\paragraph{Precise asymptotics for regularized regression and minimum-norm interpolators.} A rich line of work has studied high-dimensional
regularized estimators and MNIs via different techniques. Approximate message passing yields exact asymptotics for Lasso
\citep{bayati2012lasso}, robust $M$-estimation
\citep{donoho2016high}, $\ell_1$-MNI
\citep{li2021minimum} and $\ell_r$-regularized least squares
\citep{weng2018overcoming}. In particular, \citet{weng2018overcoming} study the effects of norm
exponent and sparsity on the optimally tuned risk in the
low-noise limit; in contrast, we focus on the generalization
gain near interpolation at a fixed noise level.
\cite{kur2024minimum,kur2026minimum} relate the generalization of MNIs to
$2$-uniform convexity via % a geometric approach based on 
the local theory of Banach spaces. \citet{zhou2024optimistic}
derive uniform bounds connecting population risk to training error
and model complexity, including for near-interpolating predictors.
Our strategy builds on Gordon's Gaussian comparison
inequalities and the convex Gaussian min--max theorem (CGMT) \citep{gordon1985,stojnic2013framework,thrampoulidis2015gaussian}.
These methods provide a characterization of the square-root
Lasso and its generalizations
\citep{oymak2013squared,thrampoulidis2015lasso},
convex regularized $M$-estimators \citep{thrampoulidis2018precise},
as well as uniform distributional guarantees for the Lasso
\citep{miolane2021distribution}.
Gaussian comparison methods also establish the generalization error of interpolators minimizing separable convex penalties
\citep{varma2025optimal} and yield sharp rates for $\ell_r$-MNIs \citep{wang2022tight,donhauser2022fast}. 
Our contribution is a unified characterization of $\ell_r$-MNIs and $\ell_r$-regularized regression with the square-root loss, which we then use to understand the conditions causing a sharp drop in the generalization error near interpolation.

\section{Preliminaries}
For $a,b\in\R^p$ and $r\ge1$, we  write $\langle a,b\rangle := \sum_{j=1}^p a_jb_j$ and $\|b\|_{r}:=\left(\sum_{j=1}^p|b_j|^r\right)^{1/r}.$
We consider a linear regression model
\begin{equation}\label{eq:model}
y=X\beta_*+\varepsilon,
\end{equation}
where $\beta_*\in \mathbb R^p$ is the vector of regression coefficients, $X\in \mathbb R^{n\times p}$ is the design matrix with independent $N(0,1/n)$ entries, $y\in \mathbb R^n$ is the vector of observations, and the noise vector $\varepsilon\sim N(0,\sigma^2I_n)$ is independent of $X$. We study estimators optimizing the $\ell_r$-regularized square-root loss
\begin{equation}
 \widehat\beta_{\kappa,r}
  \in \operatorname*{arg\,min}_{b\in\mathbb R^p}
  \left\{
    \|y-Xb\|_2+
  \frac{\kappa \lambda(n, p)}{\sqrt{n}}\|b\|_{r}
  \right\}.
  \label{eq:minprob}
\end{equation}
For $r=1$, this estimator is the square-root Lasso
\citep{belloni2011square}. We note that, when the coefficient $\kappa \lambda(n, p)$ in front of the $\ell_r$ regularization is smaller than a critical value (see e.g.\ Theorem 4.2 in \cite{friedlander2008exact}), $\widehat\beta_{\kappa,r}$ coincides with the $\ell_r$-MNI given by
\begin{equation}
 \widehat\beta_r
 \in
 \argmin_{b\in\mathbb R^p}
 \left\{
\|b\|_r :Xb=y
 \right\}.
 \label{eq:min-norm-interpolator}
\end{equation}
For convenience, we express this coefficient as  $\kappa\lambda(n, p)$, so the critical value of the regularization leading to interpolation is achieved when $\kappa$ is a universal constant that does not scale with $n, p$. The precise form of $\lambda(n, p)$ depends on the scaling between $n$ and $p$, and it is clarified later. 
%Note that $\kappa$ may depend on $n,p,r$, the cannoical scaling is usually $\kappa = \min\{p^{1-1/r},(2\log(p/n))^{-1/2}\}$.

We measure the generalization error of an estimator via its mean squared error (MSE). Thus, the training and generalization error of an estimator $\beta$ are respectively given by 
\begin{equation}
 \mathcal E(\beta):= %\frac{1}{\sqrt{n}}\|X\beta-X\beta_*\|_2=
 \frac{1}{\sqrt{n}}\|y-X\beta\|_2,\qquad
\mathsf{MSE}(\beta):=\frac{1}{p}\|\beta-\beta_*\|_2^2.
 \label{eq:lt-errors}
\end{equation}
%with $\beta_*$ being the ground-truth vector of regression coefficients in \eqref{eq:model}. 
Let $0_p\in\mathbb R^p$ be the all-$0$ vector and $\mathsf{MSE}(0_p)$ its corresponding MSE. We denote normalized MSE and training error by $\mathsf{MSE}_{s}(\beta):=\mathsf{MSE}(\beta)/\mathsf{MSE}(0_p)$ and $\mathcal E_{s}(\beta):=\mathcal E(\beta)/\mathcal E(0_p)$. Throughout the paper, $c,C \in (0,1)$ are absolute constants and $c(\alpha),C(\alpha)$ constants that only depends on $\alpha$. %, which may change from line to line. 
We use %$\asymp,\lesssim\gtrsim $ denote equality/inueqlaity up to mulitplicative aboslute constnat. Here 
$O(\cdot),o(\cdot),\Omega(\cdot),\omega(\cdot)$ for asymptotic notation in $n, p$, and $\mathrm{conv}\{\cdot\}$ denotes the convex hull operator.
% We denote by $B_r^m$ the unit $\ell_r$ ball in $\R^m$ and by $S^{m-1}$ the unit sphere in $m$ dimensions \mm{do we need this in the body?}.  

%We denote by $\delta_p=n/p\to\delta$, here $\rightarrow$  means as $n$ and $p$ grow together. We also $g_m\sim N(0,I_m)$ to an instorpic guassian verctor in $\R^m$, and use that $\E \|g_m\|_{2} = \sqrt{m} + O(1/\sqrt{m})$. $S^{d-1}$ is the unit sphere. We use the notation of $L = \log(p/n)$, and throughout this work $\lambda > 0$ and a \textbf{fixed} constant that deoes not depend on $p,n$.
%We consider \(1\le r<\infty\). As $n,p\to\infty$, let $\delta_p=n/p\to\delta\in(0,\infty)$, and suppose
%\[
% y=X\beta+\varepsilon, \qquad X_{ij}\stackrel{\mathrm{iid}}\sim N(0,1/n), \qquad \varepsilon\sim N(0,\sigma^2I_n),
%\]
%where $X$ and $\varepsilon$ are independent.   
\section{Main results}

Our results cover two asymptotic regimes.
In the strongly overparameterized regime ($p\gg n$), Section~\ref{sec:01} establishes that, while  the training error remains negligible, the MSE exhibits a zero-one transition.
In the proportional regime ($p\asymp n$), Section~\ref{sec:lp-norm-extension} gives precise characterizations of $\ell_r$-regularized estimators and minimum-norm interpolators.
Building on these characterizations, Section~\ref{sec:mse-slope} then shows that smaller norm exponents and sparser regression problems yield sharper generalization gains near interpolation.  

\subsection{A zero--one law for \texorpdfstring{$\ell_1$}{l1}
regularization when \texorpdfstring{$p\gg n$}{p >> n}}\label{sec:01}

We consider the noiseless Gaussian model in \eqref{eq:model}, with
$y=\X\beta_*$ and
\begin{equation}
 n,p\to\infty,\qquad
 p/n\longrightarrow\infty,\qquad
 \log p=o(n).
 \label{eq:zero-one-dimensions}
\end{equation}
Set $L:=2\log(p/n)$ and $\lambda(n,p)=\sqrt L$ in
\eqref{eq:minprob}. Set $\nu=\nu_{n,p}>0$ with
$\nu^2=L-\log L+O(1)$ and deterministic $\eta=\eta_{n,p}\in(0,1)$
such that
\begin{equation}
 k_*:=\left\lfloor(1-\eta)\frac n{\nu^2}\right\rfloor \approx \frac{n}{2\log(p/n)},
 \qquad \eta=o(1),\qquad
 L^{-1}+\frac Ln=o(\eta).
 \label{eq:signal}
\end{equation}
Let $\beta_*\in\R^p$ be any fixed $k_*$ sparse vector (i.e. $\|\beta_*\|_0 = k_*$) with $+1$ or $-1$ entries.
We remark that $k_*$ is also known as the leading weak-recovery sparsity
scale \citep{donoho2009}.
% The assumptions on $\eta$ are compatible with the entire dimensional
% range; for example, $\eta=\sqrt{L^{-1}+L/n}$ is admissible for all
% sufficiently large $n,p$.

Our first result shows a discontinuity in the MSE along the regularization path in the noiseless case. 

\begin{lemma}[Zero--one generalization]
\label{thm:0-1grok}
Under the assumptions above, fix $\kappa>0$, $\kappa\ne1$, and consider
\eqref{eq:minprob} with $r=1$ and $\sigma=0$. Then, the minimizer $\widehat\beta_{\kappa,1}$ satisfies
% \mm{this is the main result}
\begin{equation}
 \left(
\mathsf{MSE}_s(\widehat\beta_{\kappa,1}),
\mathcal E_s(\widehat\beta_{\kappa,1})
 \right)
 \xrightarrow{\Pp}
 \left(\mathbf 1_{\{\kappa>1\}},\mathbf 1_{\{\kappa>1\}}\right).
 \label{eq:zero-one}
\end{equation}
\end{lemma}
% \mm{formula below is more of a comment}
%More precisely, with probability $1-\exp(cn\eta)$, it holds
%\begin{equation}
% \begin{cases}
%  \widehat\beta_{\kappa,1}=\beta_*, & 0<\kappa<1,\\[4pt]
%  \displaystyle\frac{\|\widehat\beta_{\kappa,1}\|_2^2}{k_*}
%  \le C_\kappa L^{-1}, & \kappa>1.
% \end{cases}
% \label{eq:zero-one-regimes}
%\end{equation}
%This result is illustrated numerically in Figure \mm{ADD}. 
In words, \eqref{eq:zero-one} shows that a discontinuity in both training and generalization error is possible despite the convexity of the norm. 
% This is surprising, since such a discontinuity is not expected to occur when the norm is $r$-uniform smooth, i.e., it has some curvature properties, such as in the $\ell_r$ case. \mm{could be justified better}
The proof combines Gaussian comparison inequalities with the connection between symmetric Gaussian polytopes $Q_{n,p}:=\operatorname{conv}\{\pm X_j:1\le j\le p\}\subset\mathbb R^n$ and $\ell_1$-MNI \citep{donoho2005neighborliness,kurpathak2026gaussianity}.
The Minkowski functional of $Q_{n,p}$ gives the minimum $\ell_1$ budget needed to interpolate the data, and a sharp estimate of its spherical mean guides the calibration of the sparsity $k_*$; see \eqref{eq:sketch-norm-representation}.
For $\kappa<1$, this calibration leaves a strict margin in Gordon's comparison theorem, ensuring that every nonzero perturbation of the ground truth increases the objective  in \eqref{eq:minprob}, hence yielding unique exact recovery.
For $\kappa>1$, a clipping argument and properties of the Gaussian distribution give an $\ell_1$--$\ell_2$ norm comparison for $\widehat\beta_{\kappa,1}$.
Then, a restricted-isometry argument and a comparison of the objective in \eqref{eq:minprob} at $\widehat\beta_{\kappa,1}$ and at the zero estimator show that the normalized $\ell_2$ energy of $\widehat\beta_{\kappa,1}$ vanishes.

\paragraph{Zero--one generalization with zero training error.}
The preceding lemma does not separate fitting from generalization:
training and generalization error undergo the same transition. To keep exact
interpolation while generalization exhibits a zero--one law, we slightly modify the $\ell_1$ norm.
Specifically, we replace $\ell_1$ by its infimal convolution
with a scaled $\ell_\infty$ norm:
\begin{equation}
 N_\kappa(b):=\inf_{u+v=b}
 \left\{\|u\|_1+
 \frac{p}{\kappa}\sqrt{\frac{2}{\pi L}}\,\|v\|_\infty\right\}.
 \label{eq:cube-hybrid-norm}
\end{equation}
Its unit ball is
$\operatorname{conv}\left\{B_1^p,\frac{\kappa}{p}
\sqrt{\pi L/2}\,B_\infty^p\right\}$,
and its geometric distance to the $\ell_1$-norm is within a factor $O_\kappa(\sqrt L)$, as
\begin{equation}
 \frac1\kappa\sqrt{\frac{2}{\pi L}}\,\|b\|_1
 \le N_\kappa(b)\le\|b\|_1,
 \qquad \forall b\in\R^p.
 \label{eq:cube-hybrid-comparison}
\end{equation}
Leveraging this norm allows us to prove  the following extreme example of grokking. 

%In words, as $\kappa$ varies along the regularization path, the idea is to make the solution ``dense'' (non$\|\cdotTheta(d)$-effective sparsity) and , and.

\begin{theorem}[Zero--one generalization at interpolation]
\label{thm:cube-grokking}
Under the same assumptions, fix $\kappa>0$, $\kappa\ne1$, and let
\begin{equation}
 \widetilde\beta_\kappa=
 \operatorname*{arg\,min}_{b\in\R^p}
 \left\{\|y-\X b\|_2+
 \frac{\kappa\sqrt L}{2\sqrt n}\,N_\kappa(b)\right\}.
 \label{eq:cube-hybrid-estimator}
\end{equation}
Then, it holds
\begin{equation}
\left(\mathsf{MSE}_s(\widetilde\beta_\kappa),
\mathcal E_{s}(\widetilde\beta_\kappa)\right)
 \xrightarrow{\Pp}
 \left(\mathbf 1_{\{\kappa>1\}},0\right).
 \label{eq:cube-zero-one}
\end{equation}
\end{theorem}
The idea is that the solution $\widetilde\beta_\kappa$ in \eqref{eq:cube-hybrid-estimator} is \emph{(i)} $k_*$-sparse when $\kappa<1$, and \emph{(ii)} dense (all entries have roughly the same size) and interpolating  when $\kappa>1$. The result is illustrated in the left plot of Figure \ref{fig:lr_monotone_curves}, which clearly displays how the normalized MSE jumps from $0$ to $1$ while the training error remains extremely small.     Full proofs, as well as sketches, of the results in this subsection appear in  Appendix \ref{app:zero-one-proof}, where we also discuss  the technical novelty of our approach (see Remarks \ref{R:WR}-\ref{R:CGMT} therein).

\subsection{Precise characterization of estimators when $p\asymp n$}
\label{sec:lp-norm-extension}

We next study the regime in which the sample size $n$ and the feature dimension $p$ grow at the same rate. Our aim is to characterize both training error and MSE for the family of \(\ell_r\) regularizers, and Lemmas \ref{thm:main-nonint}-\ref{thm:min-norm-interpolator} below give a scalar description of these two regimes in terms of a pair of fixed point equations. This provides the starting point to establish under what conditions the MSE exhibits a sharp drop once the training error is already close to $0$, as done in Section \ref{sec:mse-slope}.

More formally, throughout this subsection and Section~\ref{sec:mse-slope}, we consider the model in \eqref{eq:model} and assume
$n,p\to\infty$ with $\delta_p=n/p\to\delta\in(0,\infty)$. We set $\lambda(n, p) = p^{1-1/r}$ in \eqref{eq:minprob}, so that $\widehat\beta_{\kappa,r}$ coincides with the MNI $\widehat\beta_{r}$ in \eqref{eq:min-norm-interpolator} when $\kappa$ is smaller than a universal constant independent of $n, p$.
The empirical
distribution of $\beta_*$ is assumed to converge to $P_B$ in $W_q$, with $q=\max\{2,r\}$.
%Thus, $P_B$ describes the limiting distribution of the signal coordinates; the coordinates themselves need not be random.
%
The scalar description given by Lemmas \ref{thm:main-nonint}-\ref{thm:min-norm-interpolator} is expressed through
\begin{equation}\label{eq:sdes}
 f_{\tau,\alpha}
 :=\tau\eta_{\alpha,r}\left(Z\right),
 \qquad Z = Z(\tau,B):=\frac B\tau+G,\qquad B\sim P_B,\quad G\sim N(0,1),
\end{equation}
where $B$ and $G$ are independent and the proximal operator $\eta_{\alpha,r}(z)$ is given by $\argmin_{u\in\R}
 \left\{\frac12(u-z)^2+\frac\alpha r|u|^r\right\}$ (for $r=1$, $\eta_{\alpha,1}(z):=\operatorname{sign}(z)(|z|-\alpha)_+$). We also  define
\begin{equation}
 H_r(\tau,\alpha)
 \hspace{-.1em}:=\hspace{-.1em}\E\left[\hspace{-.1em}\left(\hspace{-.1em}\eta_{\alpha,r}(Z)
 \hspace{-.1em}-\hspace{-.1em}\frac B\tau\right)^2\right],\quad %\label{eq:Hdef}\\
 D_r(\tau,\alpha)
 \hspace{-.1em}:=\hspace{-.1em}\E\left[\eta'_{\alpha,r}(Z)\right],\quad %\label{eq:Ddef}\\
 U_r(\tau,\alpha)
 \hspace{-.1em}:=\hspace{-.1em}\E\left[\left|\eta_{\alpha,r}(Z))\right|^r\right],
 \label{eq:Hdef}
\end{equation}
%For $r=1$, we set %the derivative in $\eta'_{\alpha,r}$  is understood almost everywhere, or equivalently, 
with $D_1(\tau,\alpha)=\Pp(|Z|>\alpha)$. The scalar estimator in \eqref{eq:sdes} applies
shrinkage to a signal coordinate observed with effective Gaussian
noise of standard deviation $\tau$. The functions in
\eqref{eq:Hdef} summarize its behavior:
$\tau^2H_r(\tau,\alpha)=\mathbb E(f_{\tau,\alpha}-B)^2$ is the
MSE, $\tau^rU_r(\tau,\alpha)
=\mathbb E|f_{\tau,\alpha}|^r$ is the $r$-th absolute moment and, as we will see, $D_r(\tau,\alpha)$ determines the training error $\cE(\widehat\beta_{\kappa,r})$ via the expression $\tau(1-D_r(\tau,\alpha)/\delta)$. The parameters $\tau$ and $\alpha$ will be identified via a set of two fixed point equations.

\paragraph{Regularized regression with square-root loss.} 
 We start by characterizing the estimator in \eqref{eq:minprob} when
its training error $\cE(\widehat\beta_{\kappa,r})$ has a strictly positive limit.

\begin{lemma}[Precise characterization for nonzero training error]
\label{thm:main-nonint}
Fix $\kappa>0$. Let $\tau_\kappa,\alpha_\kappa>0$ satisfy
\begin{align}
 \tau_\kappa^2
 &=\sigma^2+\frac{\tau_\kappa^2}{\delta}
 H_r(\tau_\kappa,\alpha_\kappa),\qquad
 \kappa
 =\begin{cases}
 \alpha_\kappa,&r=1,\\
 \alpha_\kappa U_r(\tau_\kappa,\alpha_\kappa)^{(r-1)/r},
 &1<r<\infty.
 \end{cases}
 \label{eq:se1}
\end{align}
Define
\begin{equation}
 R_\kappa:=\tau_\kappa
 \left(1-\frac{D_r(\tau_\kappa,\alpha_\kappa)}{\delta}\right).
 \label{eq:Rdef}
\end{equation}
If $R_\kappa>0$, then, almost surely, the minimizer
$\widehat\beta_{\kappa,r}$ in \eqref{eq:minprob} satisfies
\begin{align}
 \mathcal{E}(\widehat\beta_{\kappa,r})
 &\to R_\kappa,\qquad\qquad
\mathsf{MSE}(\widehat\beta_{\kappa,r})
 \to \tau_\kappa^2H_r(\tau_\kappa,\alpha_\kappa)
 =\delta(\tau_\kappa^2-\sigma^2).\label{eq:resclaim}
\end{align}
\end{lemma}
In words, the first
equation in \eqref{eq:se1} determines the effective noise level from
the noise variance and the MSE. The second equation
calibrates the shrinkage parameter $\alpha_\kappa$ to the
regularization strength $\kappa$.
Solving for $(\tau_\kappa, \alpha_\kappa)$ determines both
limiting training error and the MSE via \eqref{eq:resclaim}.

The argument uses the CGMT framework to compare \eqref{eq:minprob} with an auxiliary
optimization involving independent Gaussian vectors. We first show that its population counterpart has the unique minimizer $f_{\tau_\kappa,\alpha_\kappa}$,
characterized by the fixed-point equations \eqref{eq:se1}.
We then prove the stability of the minimizers of the auxiliary problem. This in turn allows us to use the CGMT to transfer
their asymptotic behavior to the original estimator, thus yielding \eqref{eq:resclaim}. The full proof is deferred to Appendix~\ref{app:pfnon-int}, with an overview in Appendix~\ref{app:sketchnonint}. Beyond the limits in \eqref{eq:resclaim}, the argument also gives convergence in $W_2$ of the joint empirical distribution of $(\widehat\beta_{\kappa,r}, \beta_*)$, which implies convergence of the empirical averages of functions with quadratic growth, see Remark \ref{rem:joint-empirical-laws} in Appendix~\ref{app:pfnon-int}. 

\paragraph{Minimum-norm interpolation.}
For $\delta<1$, we then provide a similar characterization for the minimum-$\ell_r$-norm
solution that fits the observations exactly. 

\begin{lemma}[Precise characterization for the MNI]
\label{thm:min-norm-interpolator}
Let $\delta\in(0,1)$. Suppose that
$\tau_{\mathrm{int}},\alpha_{\mathrm{int}}>0$ satisfy
\begin{align}
 \tau_{\mathrm{int}}^2
 &=\sigma^2+
 \frac{\tau_{\mathrm{int}}^2}{\delta}
 H_r(\tau_{\mathrm{int}},\alpha_{\mathrm{int}}),\qquad
 D_r(\tau_{\mathrm{int}},\alpha_{\mathrm{int}})=\delta.
 \label{eq:int-state-mse}
\end{align}
Then, almost surely, the minimum-norm interpolator
$\widehat\beta_r$ in \eqref{eq:min-norm-interpolator} satisfies
\begin{equation}
\mathsf{MSE}(\widehat\beta_{r})
 \to
 \tau_{\mathrm{int}}^2
 H_r(\tau_{\mathrm{int}},\alpha_{\mathrm{int}})
 =
 \delta\bigl(\tau_{\mathrm{int}}^2-\sigma^2\bigr). %,\qquad  p^{-1/r}\bigl\|   \widehat\beta_r \bigr\|_{r} \to \|f_{\mathrm{int}}\|_r = \tau_{\mathrm{int}} U_r(\tau_{\mathrm{int}},\alpha_{\mathrm{int}})^{1/r}.
 \label{eq:int-mse}
\end{equation}
\end{lemma}

We note that the first scalar equation in \eqref{eq:int-state-mse} coincides with the one in \eqref{eq:se1}; the second scalar equation in \eqref{eq:int-state-mse} corresponds to setting $R_\kappa=0$ in \eqref{eq:Rdef}, enforcing exact interpolation in the characterization.
The proof first shows that a sufficiently small positive
regularization parameter enforces minimum-norm interpolation,
and then adapts the CGMT analysis of Lemma~\ref{thm:main-nonint}
to the case $D_r=\delta$.
The full argument is in Appendix~\ref{app:pfint}.

\paragraph{Relation to previous work.}

Specializing Theorem 3.1 of  \citet{thrampoulidis2018precise} to the  square-root loss and $\ell_r$ regularization gives the limit in probability of $\mathsf{MSE}(\widehat\beta_{\kappa,r})$. Lemma \ref{thm:main-nonint} additionally establishes the limit $R_\kappa$ of the normalized training error, which is needed for our
analysis of the interpolation threshold in
Section~\ref{sec:mse-slope}. Furthermore, it gives almost-sure convergence
and the joint empirical $W_2$ limit discussed in Remark \ref{rem:joint-empirical-laws}. 

For minimum-norm interpolation, Lemma \ref{thm:min-norm-interpolator} extends
\citet{li2021minimum} which is restricted to $r=1$.
For $r\ge 1$, specializing Theorem~3.1 of
\citet{varma2025optimal} to isotropic features and the potential
$\psi(t)=|t|^r/r$ gives the limit in probability of $\mathsf{MSE}(\widehat\beta_{r})$ for $\beta_*$ having i.i.d.\ coordinates. 
Lemma \ref{thm:min-norm-interpolator} relaxes this i.i.d.\ assumption, and it gives almost-sure convergence.

We finally note that the application of results by \cite{thrampoulidis2018precise,varma2025optimal} still requires showing uniqueness of some parameters involved in the characterization. Our analysis directly gives uniqueness, after assuming existence of the solution of the fixed point equations; furthermore, existence at and just above the interpolation threshold follows from Lemma \ref{prop:branch-attachment} (Appendix \ref{app:aux}) when $\sigma>0$, $\delta\in (0, 1)$ and $r\in (1, 2)$. %Finally, we highour strategy presented in Appendices \ref{app:pfnon-int}-\ref{app:pfint} and summarized in the sketch of Appendix \ref{app:sketchnonint}, which follows directly from our analysis. In addition, our argument in analysis provides a unified argument for both regularized regression and MNIs, which allows to prove the continuity statement in Proposition \ref{}. This is necessary for the analysis in Proposition \ref{}, and it does not follow from existing work. They appear in Appendices~\ref{app:pfnon-int} and~\ref{app:pfint}, respectively; Appendix~\ref{app:sketchnonint} provides an overview of the regularized-estimation proof.

%When $\sigma=0$, the assumption $\tau_{\mathrm{int}}>0$ describes regimes with positive limiting MSE. The noiseless exact-recovery boundary, where the effective noise scale vanishes, is treated separately in Section~\ref{sec:mse-slope}.

%These characterizations express the limiting MSE in the common form $\delta(\tau^2-\sigma^2)$, with the parameters determined by the relevant scalar equations. In Section~\ref{sec:mse-slope}, we study how this error changes as the regularization strength approaches the interpolation threshold, and how that change depends on the norm exponent $r$ and the signal distribution $P_B$.

\subsection{Generalization near interpolation: norm geometry and sparsity}
\label{sec:mse-slope}

We now use the characterization in Section~\ref{sec:lp-norm-extension}
to study how norm geometry (captured by the parameter $r$) and sparsity of $\beta_*$ affect generalization
near interpolation. We focus on comparing the decrease in MSE and training error, establishing conditions under which the MSE decreases more rapidly having fixed the training error.  

More formally, consider $\delta\in(0,1)$ and let $\overline R_r(\kappa)$ denote the limiting
normalized training error of a global minimizer: $\overline R_r(\kappa)$ is either $0$ or it equals
$R_\kappa>0$ in \eqref{eq:Rdef}. Since $\overline R_r(\kappa)$ is
nondecreasing in $\kappa$ (Lemma~\ref{lem:residual-monotonicity} in Appendix \ref{app:aux}), we define the critical regularization leading to interpolation as
\begin{equation}\label{eq:kappastar}
 \kappa_r^\star
 :=\sup\{\kappa\geq0:\overline R_r(\kappa)=0\}
 =\inf\{\kappa\geq0:\overline R_r(\kappa)>0\}.
\end{equation}
For $\kappa>\kappa_r^\star$, we denote the limiting MSE as $\overline{\mathsf{MSE}}_r(\kappa):=\delta\{\tau_\kappa^2-\sigma^2\}$ and define
\begin{equation}
 S_r
 :=\lim_{\kappa\downarrow\kappa_r^\star}
 \frac{\dfrac{d}{d\kappa}\overline{\mathsf{MSE}}_r(\kappa)}
      {\dfrac{d}{d\kappa}R_\kappa}.
 \label{eq:slope-definition}
\end{equation}
The quantity $S_r$ measures how much
generalization changes for a given change in training error: $S_r>0$ means that an exact
interpolation outperforms approximate one, and a larger value of $S_r$ is evidence of more significant grokking.
Throughout this section, we focus on the sparse family of regression problems where the empirical distribution of $\beta_*$ approaches the law of
\begin{equation}
 B=\frac{\xi_\rho W}{\sqrt\rho},
 \qquad \xi_\rho\sim\operatorname{Bernoulli}(\rho),
 \label{eq:sparse-prior-family}
\end{equation}
where $\xi_\rho$ and $W$ are independent. This choice is rather broad, and it includes Rademacher, Gaussian, bounded, and
finite-mixture random variables.

\begin{figure}[t]
  \centering
  \includegraphics[width=\linewidth]{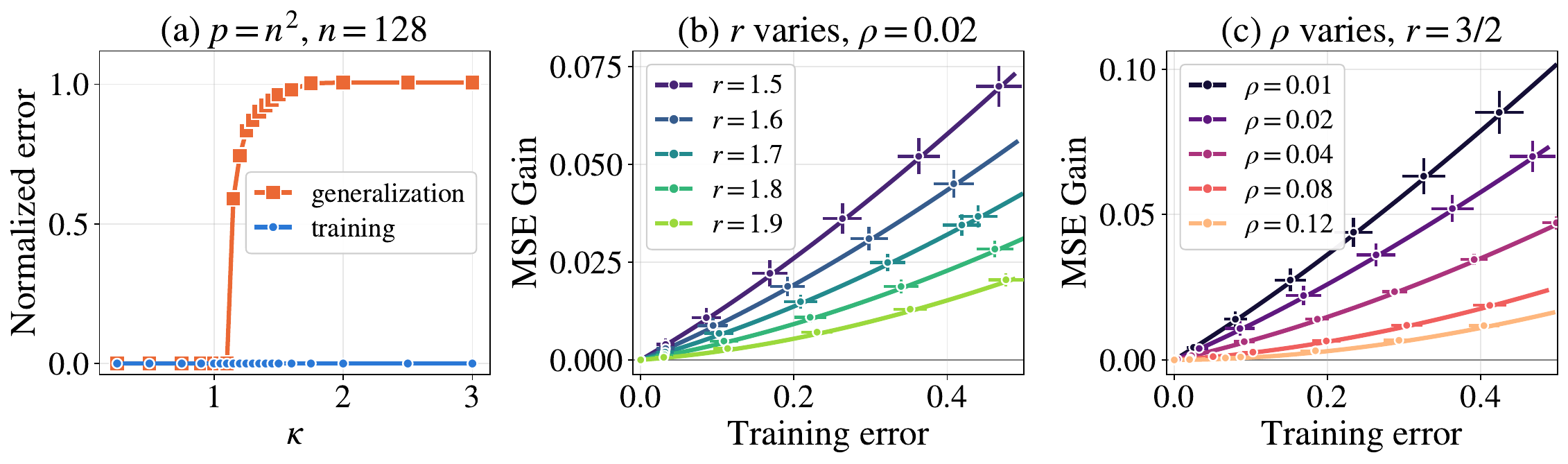}%
  \vspace{-1em}
  \caption{Training and generalization error (MSE) for regression regularized with the convex $\ell_1-\ell_\infty$ norm in \eqref{eq:cube-hybrid-norm} (Left) and with $\ell_r$ norms (Center, Right). \emph{Left:} We pick $\beta_*$ to be a $k_*$-sparse vector with uniformly distributed $\pm 1$ entries, consider the estimator in \eqref{eq:cube-hybrid-estimator} and plot normalized training (blue) and generalization (orange) errors as a function of regularization strength $\kappa$. We report the mean $\pm$ one standard error over $10$ seeds. The generalization error exhibits a zero--one transition around $\kappa=1$, while the training error remains close to $0$ (Theorem \ref{thm:cube-grokking}). \emph{Center, Right:} We pick 
 $B$ as in \eqref{eq:sparse-prior-family} with $W$ a Rademacher random variable, $\delta=0.2$, $\sigma=0.1$, $n=1000$ and $p=5000$. We consider the estimator $\widehat\beta_{\kappa,r}$ in \eqref{eq:minprob} and plot MSE gain $\overline{\mathsf{MSE}}_r(\kappa)-\overline{\mathsf{MSE}}_r(\kappa_r^\star)$ as a function of training error $\cE(\widehat\beta_{\kappa,r})$. Solid curves are theoretical predictions from Lemma~\ref{thm:main-nonint}, and dots are obtained by solving regularized regression numerically (mean $\pm$ one standard error over $20$ seeds). In the central plot, 
for a fixed $\rho=0.02$, the curves flatten as $r$ increases (Theorem~\ref{thm:sparse-decreasing-slope}). In the right plot, for a fixed $r=3/2$, the curves flatten as $\rho$ increases (Theorem~\ref{prop:sparsity-monotonicity}).}
  \label{fig:lr_monotone_curves}
  \vspace{-1em}
  \end{figure}

\paragraph{Role of norm geometry.} The result below shows that, for sparse regression problems, grokking is more pronounced as the regularization passes from $\ell_2$ to $\ell_1$, i.e., $S_r$ is decreasing in $r$.

\begin{theorem}[Smaller $r$ implies more grokking]
\label{thm:sparse-decreasing-slope}
Fix $1<r_-<r_+<2$ and $\sigma\ge 0$. For a fixed $\lambda>0$, let $B$ be given by \eqref{eq:sparse-prior-family} with $\rho=\lambda\delta$. Assume $\E W^2=1$ and $\E|W|^{2/(r_--1)}<\infty$. Then, there exist
$\lambda_0:=\lambda_0(r_-,r_+,W)>0$,
$\delta_0:=\delta_0(r_-,r_+,W,\lambda,\sigma)>0$ such that,
whenever $\lambda\in (0, \lambda_0)$ and
$\delta\in (0, \delta_0)$, $S_r$ is strictly positive and decreasing in $r$ on $(r_-,r_+)$.
\end{theorem}

In words, as $\delta\to 0$ (i.e., as the model gets more over-parameterized) and for a sufficiently sparse problem (i.e., $\rho\lesssim \delta$), exact
interpolation generalizes better than approximate fitting, and smaller norm exponents
yield a sharper decrease in the MSE. Remarkably, the result holds regardless of the amount of noise $\sigma$ in the data. We note that 
$\rho/\delta=\lambda$ is small, so the number of nonzero entries in $\beta_*$ is small compared to $n$, but the relative sparsity $\rho$ does not vanish in $n$  ($\rho=\Theta(1)$). The result is illustrated in the central plot of Figure \ref{fig:lr_monotone_curves}, and its proof is deferred to
Appendix~\ref{app:proof-sparse-decreasing-slope}. %) also establishes the existence of the limit in \eqref{eq:slope-definition}, as well as the continuous differentiability of the map $r\mapsto S_r$. 

We next show that, for noiseless data ($\sigma=0$) and $\ell_1$ regularization ($r=1$), the MSE improvement for a given change in the training error may be 
unbounded. This result aligns with Lemma \ref{thm:0-1grok} above, albeit in a different regime of $n, p$, and it is proved in Appendix~\ref{app:proof-r1-critical-infinite-slope}. 
\begin{theorem}[Unbounded improvement in MSE for $\ell_1$ regularization]
\label{thm:r1-critical-infinite-slope}
Fix $0<\rho<1$, and let $B$ be given by
\eqref{eq:sparse-prior-family}. Assume that $\E W^2=1$ and, as $t\downarrow 0$,
\begin{equation}
 \Pp(|W|\leq t)=o(t^2).
 \label{eq:r1-critical-small-ball}
\end{equation}
Set $r=1$, $\sigma=0$, and
\begin{equation}
 \delta=\delta_c(\rho)
 :=\min_{\alpha\geq0}
 \left\{\rho(1+\alpha^2)
 +(1-\rho)\E(|G|-\alpha)_+^2\right\},
 \qquad G\sim N(0,1).
 \label{eq:r1-critical-delta}
\end{equation}
Let $\alpha_c>0$ be the minimizer of
\eqref{eq:r1-critical-delta}. Then $\kappa_1^\star=\alpha_c$,
and the limit in \eqref{eq:slope-definition} satisfies $S_1=+\infty$.
\end{theorem}

In words, the limiting MSE and training error both
tend to zero as $\kappa\downarrow\kappa_1^\star$, but the training error does so at a faster rate. Hence, near interpolation, a small
reduction in training error yields a much larger reduction
in MSE. This results in delayed generalization captured by an unbounded $S_1$. Theorem \ref{thm:r1-critical-infinite-slope} identifies two sufficient conditions leading to such phenomenon: \emph{(i)} the nonzero part of the signal $\beta_*$ has little mass around $0$, %i.e. sparse, 
which follows from the fact that $W$ does not satisfy even a mild small ball condition as \eqref{eq:r1-critical-small-ball} entails; \emph{(ii)} the sample size equals the critical value $\delta_c(\rho)$ in \eqref{eq:r1-critical-delta}, which corresponds to the classical weak recovery
threshold for noiseless $\ell_1$ minimization with Gaussian designs
\citep{donoho2009,almt2014}. %The result is illustrated in the central plot of Figure \ref{fig:lr_monotone_curves}, and its proof is in
Remarkably, the situation is quite different in the presence of label noise: when $\sigma>0$, the $\ell_1$-MNI has a worse MSE than an approximate interpolator, i.e., $S_1<0$ (see Proposition \ref{prop:r1-negative-slope} in Appendix \ref{app:proof-r1-negative-slope} for the precise statement and proof).

We finally note that exact interpolation hurts generalization for $\ell_2$ regularization: for $\sigma>0$, allowing positive training error reduces the MSE, i.e., $S_2< 0$; furthermore, $S_2=0$ in the noiseless case (see Proposition~\ref{prop:r2-negative-slope} in
Appendix~\ref{app:proof-r2-negative-slope}
for the precise statement and proof). To reconcile this result with the analysis of grokking for ridge regularization in \cite{xu2026grok}, note that in their mechanism, generalization is delayed by
the slow decay of an initialization component irrelevant for
the training data. In contrast, our results
identify when regularization geometry supports such
a delay. 

%\begin{proposition}[Negative slope for noisy $\ell_1$ interpolation]
%\label{prop:r1-negative-slope}
%In the proportional regime model with $\delta\in(0,1)$, $\sigma>0$, and $0<\E B^2<\infty$. The limit $S_1$ in \eqref{eq:slope-definition} exists and is strictly negative. \end{proposition}
%The proof of this proposition appears in the appendix. We now conclude with the follwong corollary that has indepdent and somehow suprising form the statsical point of view.
%\begin{remark}[The noisy endpoints behave differently]
%\label{rem:noisy-endpoints}
%The positive slope in Theorem~\ref{thm:sparse-decreasing-slope}
%is an interior phenomenon within the family $1\le r\le2$.
%Indeed, in the proportional Gaussian model with any fixed
%$\sigma>0$, both endpoints satisfy $S_1<0$ and $S_2<0$
%(Propositions~\ref{prop:r1-negative-slope}
%and~\ref{prop:r2-negative-slope}).
%Thus, allowing a small positive training residual improves
%generalization at either endpoint, whereas exact interpolation
%is locally preferable in the sparse regimes established for
%$1<r<2$.
%In particular, the monotonicity in $r$ cannot be extended
%uniformly to $r=1$ at fixed noise and problem parameters.
%This also contrasts with the noiseless critical $\ell_1$ case,
%where Theorem~\ref{thm:r1-critical-infinite-slope}
%gives $S_1=+\infty$.
%\end{remark}

\paragraph{Role of sparsity.}

We next fix any $\ell_r$ norm and show that grokking is more pronounced as the regression problem becomes sparser. 

\begin{theorem}[More sparsity implies more grokking]
\label{prop:sparsity-monotonicity}
Fix $1<r<2$ and $\sigma\geq0$. For $0<\rho<1$, let $B$ be given by \eqref{eq:sparse-prior-family} with $\rho=\lambda\delta$. Assume $\E W^2=1$ and $\E|W|^{r/(r-1)}<\infty$. For any $0<\lambda_-<\lambda_+<\infty$, there exists
$\delta_0:=\delta_0(r,W,\lambda_-,\lambda_+,\sigma)>0$
such that, whenever $\delta\in (0, \delta_0)$, $S_r$ is  strictly
decreasing in $\rho$ on
$[\lambda_-\delta,\lambda_+\delta]$. Moreover, there exists $\lambda_c=\lambda_c(r,W)>0$
such that, if $\lambda_+<\lambda_c$, then $\delta_0$
may be chosen so that $S_r>0$ throughout this interval.
\end{theorem}

The result is illustrated in the right plot of Figure \ref{fig:lr_monotone_curves}, and its proof is in 
Appendix~\ref{app:proof-sparsity-monotonicity}. We finally note that, if there is no structure in the regression problem (i.e., $B$ is Gaussian), then exact interpolation hurts generalization  (i.e., $S_r<0$ for every $1\leq r<2$), see Proposition~\ref{prop:gaussian-negative-slope} in
Appendix~\ref{app:proof-gaussian-negative-slope}
for the precise statement and proof.

\section{Numerical results}
\label{sec:Numerical_Experiments}

We next ask whether the orderings predicted by
Theorems~\ref{thm:sparse-decreasing-slope}
and~\ref{prop:sparsity-monotonicity} persist beyond regression with an explicit regularization. The initialization scale is known to shape the implicit bias of
gradient-based training
\citep{gunasekar2018implicit,azulay2021implicit,kunin2024get,kumar2024grokking},
so it can be used to tune the geometry of the effective regularizer. We use it
first in two-layer diagonal linear networks (DLNs) \citep{azulay2021implicit},
where the implicit regularizer is known in closed form. We then use it in
transformers trained on modular arithmetic \citep{kumar2024grokking}.

\paragraph{Diagonal linear networks.}
We parameterize $\beta=u_+\circ v_+-u_-\circ v_-$ and run full-batch gradient
descent on $\cE(\beta)$ with initialization
$u_{\pm,i}(0)^2+v_{\pm,i}(0)^2=\sqrt k/2$. \citet{azulay2021implicit} show that when gradient flow converges to a zero-loss solution $\hat\beta(\infty)$, this solution minimizes an initialization-dependent regularizer:
\[
 \hat\beta(\infty)=\argmin_{X\beta=y}Q_k(\beta),
 \qquad
 Q_k(\beta)=\sum_{j=1}^p q_k(\beta_j),
 \qquad
 q_k(x)=\frac12\int_0^x
 \operatorname{arcsinh}(2z/\sqrt{k})dz.
\]
In words, as $k$ goes from $0$ to $\infty$, the initialization scale $k$ moves $Q_k$ continuously between $\ell_1$-like
and $\ell_2$-like geometries. %To quantify this, define the local homogeneity exponent $r_{\mathrm{eff}}(x;k):=xq_k'(x)/q_k(x)$. For fixed $x\neq0$, this exponent increases from $1$ to $2$ as $k$ ranges from $0$ to $\infty$. 
Furthermore, Figure \ref{fig:azulay_early_stop} in
Appendix~\ref{app:Experiment} shows that the performance of the solution $\hat \beta(t)$ obtained via early stopping
closely matches that of $Q_k$-regularized regression, upon tuning the regularization strength according to the stopping time $t$. Overall, this provides numerical evidence that training lies on a
regularization path of the type analyzed in Section~\ref{sec:mse-slope}. 

In the first two plots of Figure \ref{fig:dln_gain_monotone}, we take $B$ as in \eqref{eq:sparse-prior-family}, with $W$ a Rademacher random variable, and plot the residual generalization error ($\mathsf{MSE}(\hat\beta(t))-\mathsf{MSE}(\hat\beta(\infty))$) as a function of the training error $\cE(\beta(t))$. The experiment clearly shows that larger gains in MSE near interpolation are exhibited for small values of $k$ (left plot) and of $\rho$ (central plot). This is in agreement with Theorems~\ref{thm:sparse-decreasing-slope}
and~\ref{prop:sparsity-monotonicity}, which predict a sharper drop in MSE for $\ell_1$-like geometry (associated to small $k$) and sparse problems (associated to small $\rho$). 
%On the left panel, . The experiment clearly shows that 
%To measure the approach to $\beta(\infty)$, we use the training residual $\varepsilon(t):=\cE(\beta(t))$. Let $t_\varepsilon$ denote the first time at which $\varepsilon(t)\le\varepsilon$. We then define the
%generalization gain still remaining at residual $\varepsilon$ as
%\[
% \Delta\mathsf{MSE}(\varepsilon)
% :=\mathsf{MSE}(\beta(t_\varepsilon))-\mathsf{MSE}%(\beta(\infty)).
%\]
%This quantity is the dynamical analogue of $S_r$.
%Figure~\ref{fig:dln_gain_monotone} (Left) fixes $\rho$. There, 
%increasing $k$ flattens the curve, i.e., smaller values of $k$ lead to larger gains in MSE near interpolation. This is in agreement with the prediction of Theorem \ref{thm:sparse-decreasing-slope}, as small values of $k$ are associated to an $\ell_1$-like geometry. In the central panel, we plot the same quantity for different values of $k$. The experiment 
%moves the implicit bias toward $\ell_1$ and enlarges the late-stage drop in
%MSE, while increasing $k$ flattens the curves. This matches the predicted
%decrease of $S_r$ in $r$.
%Figure~\ref{fig:dln_gain_monotone} (Center) fixes $k$. There, increasing
%$\rho$ likewise weakens the gain, as predicted for less sparse signals.
%Together, these experiments show that the monotonicities predicted for
%explicit $\ell_r$ regularization carry over to diagonal-network training.

\begin{figure}[!t]
\centering
\includegraphics[width=\linewidth]{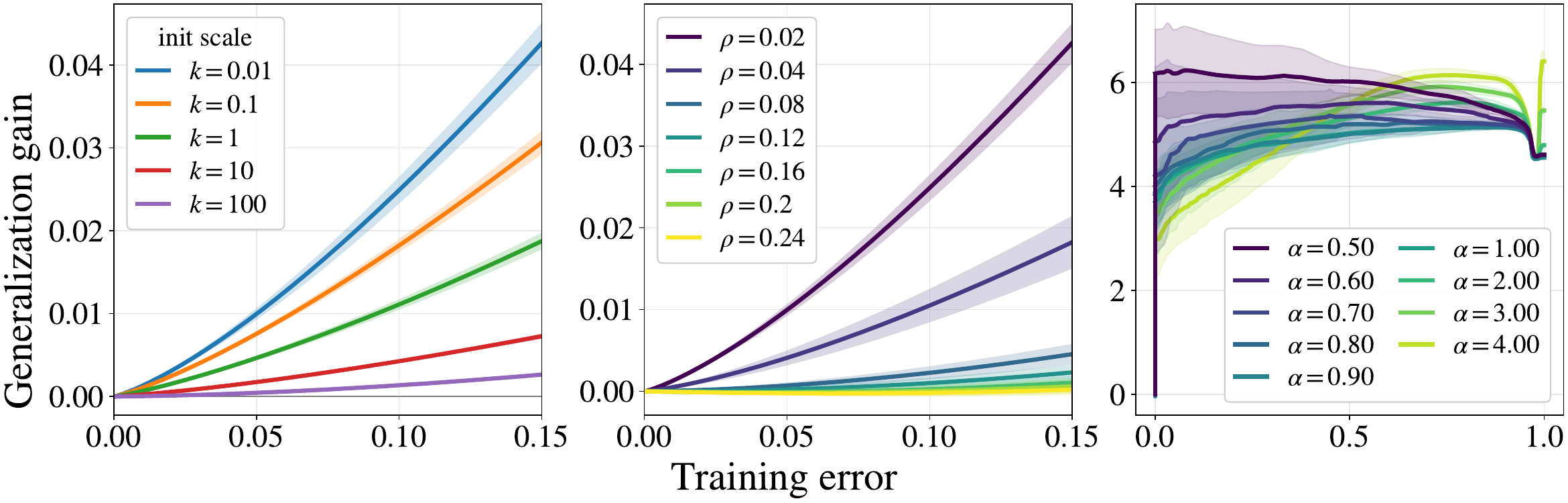}
\caption{Residual generalization error as a function of the training error.
\textbf{Diagonal linear networks} (Left, Center) are trained by gradient
descent on the square-root loss ($\delta=0.2$, $\sigma=0.03$, $n=2000$, $p=10^4$; mean $\pm$ one
standard deviation over $20$ seeds).
\textit{Left:} at fixed $\rho=0.02$, the gain in generalization near interpolation decreases as $k$ increases,
mirroring the decrease of $S_r$ in $r$
(Theorem~\ref{thm:sparse-decreasing-slope}).
\textit{Center:} at fixed $k=0.01$, the gain in generalization near interpolation decreases as $\rho$ increases
(Theorem~\ref{prop:sparsity-monotonicity}).
\textbf{Transformers} (Right) are trained by AdamW on the cross-entropy loss for a modular arithmetic task (mean $\pm$ one standard deviation over 5 seeds). We regulate the
output scale $\alpha$ controlling the learning regime and implicit bias. %Where available, the weight-decay-compensated run is used instead of the original ($\alpha$-lr recipe) run. 
As expected, the gain in generalization near interpolation decreases with $\alpha$.  For implementation details, see
Appendix~\ref{app:Experiment}.}
\vspace{-1em}
\label{fig:dln_gain_monotone}
\end{figure}

\paragraph{Transformer on a modular arithmetic task.}
We train a transformer on a modular arithmetic task following
\citet{power2022grokking,kumar2024grokking}, varying the output scale
\(\alpha\) to probe changes in the learning regime and implicit bias. %For\(\alpha<1\), we use weight-decay-compensated runs (labeled ``wd-comp'') inplace of the  original \(\alpha\)-dependent learning-rate recipe. Thisadjustment controls for the coupling between learning ratand effectiveweight decay, allowing us to better isolate the effect of output scale\citep{kumar2024grokking}. 
The right plot of Figure~\ref{fig:dln_gain_monotone} displays a decrease in late-stage
generalization gains as \(\alpha\) increases. This is qualitatively consistent with
the predicted ordering that stronger sparsity-promoting bias yields larger
gains near interpolation.

\section{Conclusions and perspectives}
We studied why predictors with nearly indistinguishable training errors can have drastically different generalization errors.
In the strongly overparameterized regime ($p\gg n$), we construct a family of norms whose MNIs exhibit a zero--one generalization transition as the norm geometry varies, while maintaining exact interpolation.
In the proportional regime ($p\asymp n$), we quantify the generalization gain for a given change in training error near interpolation, and we show that this gain increases as the target becomes sparser and the norm becomes more sparsity-promoting.
Experiments in diagonal networks and transformers trained on a modular arithmetic task support the predicted dependence on initialization and sparsity. %\mm{would be nice to end with future work which can be taken also from the limitation paragraph we had}

%%%%%%%%%%%
A natural next step is to quantify the correspondence observed between early-stopped diagonal networks and regression with their initialization-dependent implicit regularizer.
This would help connect our analysis of regularization paths to finite-time training dynamics, distinguishing the magnitude of the generalization gain from the time required to realize it.
For transformers, extending this perspective requires identifying the representation in which target sparsity is relevant and testing its role through controlled variations of task structure and output scale.

\subsection*{AI use statement}

We used generative AI tools to polish the writing, implement straightforward parts of the code, fill in a few proof details after we had identified the key methods and arguments, check proofs for correctness, and improve the coverage of related work. We did not use generative AI tools to develop the paper's core ideas, define the problem setting, or prove the main theoretical results. We take full responsibility for the final content of this work, including all text, claims, code, and other artifacts produced with the assistance of generative AI.

\subsection*{Acknowledgments}
This research was funded in whole or in part by the Austrian Science Fund (FWF) 10.55776/COE12. M. M.\ is also funded by the European Union (ERC, INF$^2$, project number 101161364). Views and opinions expressed are however those of the author(s) only and do not necessarily reflect those of the European Union or the European Research Council Executive Agency. Neither the European Union nor the granting authority can be held responsible for them. Gil Kur conducted the very initial part of the research during his visit to the IDEAL Institute, hosted by Lev Reyzin, which was supported by NSF ECCS-2217023.

\bibliographystyle{plainnat}
\bibliography{biblio}

\appendix

\section{Proofs of Lemma~\ref{thm:0-1grok} and Theorem~\ref{thm:cube-grokking}}
\label{app:zero-one-proof}
\label{app:lemma1-theorem1}
\subsection{Sketch of the proof of Lemma~\ref{thm:0-1grok}}
\label{ss:proof-sketch}

\paragraph{Calibrating the ground truth.}
Let $S=\operatorname{supp}(\beta_*)$ and $R:=\|y\|_2$.
Recall that $L=2\log(p/n)$ and that $k_*$ is defined in
\eqref{eq:signal}, so $\|\beta_*\|_1=\|\beta_*\|_2^2=k_*$.
We work under the assumptions of Section~\ref{sec:01} and keep
$\kappa>0$ fixed. Denote the objective by
\begin{equation}
 F_\kappa(b):=\|y-\X b\|_2+\tau_\kappa\|b\|_1,
 \qquad \tau_\kappa:=\kappa\sqrt{L/n}.
 \label{eq:objective}
\end{equation}
Define the symmetric Gaussian polytope
\begin{equation}
 Q_{n,p}:=\X B_1^p
 =\operatorname{conv}\{\pm\X_j:1\le j\le p\}
 \subset\R^n.
 \label{eq:sketch-polytope}
\end{equation}
This is almost surely the unit ball of a norm $\|\cdot\|_{Q_{n,p}}$,
since $\X$ has full row rank. The Minkowski functional of $Q_{n,p}$ satisfies
\begin{equation}
 \|z\|_{Q_{n,p}}
 =\inf_{\X b=z}\|b\|_1,
 \qquad z\in\R^n.
 \label{eq:sketch-norm-representation}
\end{equation}
Thus, $F_\kappa$ trades residual length against the minimum
$\ell_1$ budget needed to produce a prediction. The mean spherical norm of $Q_{n,p}$ is
\[
 M(Q_{n,p})
 :=\int_{S^{n-1}}\|\omega\|_{Q_{n,p}}\,
       \ud\sigma_{n-1}(\omega),
\]
where $\sigma_{n-1}$ is normalized surface measure on the unit sphere in $\R^n$.
By \eqref{eq:sketch-norm-representation}, this is the average
minimum $\ell_1$ budget required to interpolate a unit response.
In \citet{kurpathak2026gaussianity} it is proven that \footnote{The cited work
provides a sharper expansion. The coarser $O(1)$ remainder
is sufficient for our argument, since $\eta L\to\infty$. Also note that the assumptions $L^{-1}+L/n=o(\eta)$ and $\eta=o(1)$ imply
$L/n=o(\eta)$ and $n^{-1/2}=o(\eta)$.} 
\[
 \frac{n}{\bigl(\E_{\X}M(Q_{n,p})\bigr)^2}
 =L-\log L+O(1).
\]
This motives our deterministic
calibration $\nu^2=L-\log L+O(1)$. To see this, as $k_*=(1-\eta)n/\nu^2+O(1)$ and
$R/\sqrt{k_*}=1+O_{\Pp}(n^{-1/2})$, the calibration gives
\[
 \frac{R}{k_*}
 =\frac{\nu}{\sqrt n}
 \left[1+\frac\eta2
       +O\!\left(\eta^2+\frac Ln\right)
       +O_{\Pp}(n^{-1/2})\right].
\]
 Consequently,
\begin{equation}
 \frac{R}{k_*}
 =\frac{\nu}{\sqrt n}
       \left(1+\frac{\eta}{2}+o_{\Pp}(\eta)\right),
 \qquad
 \frac{F_\kappa(\beta_*)}{F_\kappa(0_p)}
 =\frac{\tau_\kappa k_*}{R}
 \xrightarrow{\Pp}\kappa.
 \label{eq:sketch-protrusion}
\end{equation}
Thus, the planted representation produces slightly more response
length per unit $\ell_1$ budget than the calibration scale
$\nu/\sqrt n$.
The objective comparison points to exact recovery for $\kappa<1$
and vanishing coefficient energy relative to the signal for
$\kappa>1$.
However, comparing the ground truth with zero is not sufficient:
competing predictors may mix planted and off-support coordinates.
We control all such competitors, working around $\beta_*$ below
the transition and around zero above it.

\paragraph{Exact recovery for $\kappa<1$.}
Let $s=(\beta_*)_S$, which is also the sign vector on the support.
The directional derivative of the $\ell_1$ norm at $\beta_*$ is
\begin{equation}
 D(h):=\langle s,h_S\rangle+\|h_{S^c}\|_1
 =\sup_{z\in\partial\|\beta_*\|_1}\langle z,h\rangle.
 \label{eq:directional-derivative}
\end{equation}
The subgradient inequality therefore gives
\[
 F_\kappa(\beta_*+h)-F_\kappa(\beta_*)
 \ge\|\X h\|_2+\tau_\kappa D(h).
\]
The dilated subdifferential $\nu\partial\|\beta_*\|_1$ is a
shifted coordinate cube, so its expected squared distance from
a standard Gaussian vector can be computed explicitly.
Specifically, \eqref{eq:lt-distance-moment} in the proof of
Lemma~\ref{lem:width-margin} gives
\[
 \E \operatorname{dist}(g_p,\nu\partial\|\beta_*\|_1)^2
 =k_*(1+\nu^2)+(p-k_*)h_0(\nu)
 \le n\left(1-\frac{\eta}{2}\right),
\]
for all sufficiently large dimensions, where $h_0$ is the Gaussian
excess function in Lemma~\ref{lem:h0}.
The final inequality follows from the sparsity calibration of $k_*$
and the bound $(p/n)h_0(\nu)=O(L^{-1})=o(\eta)$.
Thus, the expected squared distance lies below $n$ by a margin of
order $\eta n$.
By Jensen's inequality and \eqref{eq:lt-gaussian-length}, this yields
\[
 \E\|g_n\|_2
 -\E\operatorname{dist}(g_p,\nu\partial\|\beta_*\|_1)
 \gtrsim\eta\sqrt n.
\]
Applying Gordon's comparison with a deterministic offset
(the first probability bound in Lemma~\ref{lem:lt-comparison},
with $U=S^{n-1}$) to $\sqrt n\,\X$ with offset
$f=-\nu D$, as in the proof of Lemma~\ref{lem:width-margin},
therefore gives, with probability at least $1-e^{-cn\eta^2}$,
\begin{equation}
 \|\X h\|_2+\frac{\nu}{\sqrt n}D(h)
 \gtrsim\eta\|h\|_2
 \qquad\text{for every }h\in\R^p.
 \label{eq:sketch-offset-escape}
\end{equation}
Since $\alpha:=\kappa\sqrt L/\nu\to\kappa<1$, eventually
$0<\alpha<1$, and
\[
 \|\X h\|_2+\tau_\kappa D(h)
 =\alpha\left(\|\X h\|_2+\frac{\nu}{\sqrt n}D(h)\right)
    +(1-\alpha)\|\X h\|_2
 \gtrsim_\kappa\eta\|h\|_2.
\]
Hence, every nonzero perturbation increases the objective.
This proves unique exact recovery,
$\widehat\beta_{\kappa,1}=\beta_*$, establishing the
$\kappa<1$ regime of the lemma.

\paragraph{The null side: $\kappa>1$.}
Any improvement over zero must exploit unusually large initial
correlations.
Recall that $\tau_\kappa=\kappa\sqrt{L/n}$, and define
\begin{equation}
 q:=\X^{\mathsf T}\frac yR,\qquad
 \gamma:=\frac{1+\kappa}{2}\sqrt{L/n}<\tau_\kappa,\qquad
 Z_\gamma^2:=\sum_{j=1}^p(|q_j|-\gamma)_+^2.
 \label{eq:initial-correlations}
\end{equation}
Here $q$ records the initial feature--response correlations:
$q_j=\langle\X_j,y\rangle/R$ measures the alignment of feature $j$
with the response direction before fitting any coefficients.
By Cauchy--Schwarz,
\[
 \|y-\X b\|_2
 \ge\left\langle\frac yR,y-\X b\right\rangle
 =R-\langle q,b\rangle.
\]
Thus, $F_\kappa(b)\le R$ implies
$\tau_\kappa\|b\|_1\le\langle q,b\rangle$.
Clipping each coordinate of $q$ to $[-\gamma,\gamma]$ decomposes
$q$ into a vector with $\ell_\infty$ norm at most $\gamma$
and a remainder with $\ell_2$ norm $Z_\gamma$.
H\"older's and Cauchy--Schwarz inequalities therefore yield
\begin{equation}
 F_\kappa(b)\le R
 \quad\Longrightarrow\quad
 \tau_\kappa\|b\|_1\le\langle q,b\rangle
 \le\gamma\|b\|_1+Z_\gamma\|b\|_2.
 \label{eq:sketch-zero-plane}
\end{equation}
In the proof of Lemma~\ref{lem:initial-excess} below, Gaussianity
gives $Z_\gamma\lesssim\sqrt{k_*/n}$ with probability tending to one;
see \eqref{eq:lt-clipping}.
Subtracting $\gamma\|b\|_1$ in \eqref{eq:sketch-zero-plane}
and using $\tau_\kappa-\gamma=\frac{\kappa-1}{2}\sqrt{L/n}$ gives
\begin{equation}
 \|b\|_1
 \lesssim\frac{1}{\kappa-1}\sqrt{\frac{k_*}{L}}\,\|b\|_2
 \lesssim_\kappa\frac{k_*}{\sqrt n}\|b\|_2
 \qquad(F_\kappa(b)\le R).
 \label{eq:sketch-effective-sparsity}
\end{equation}
The last comparison uses $k_*L/n\to1$.
Thus, every nonzero vector with $F_\kappa(b)\le R$ has effective
sparsity $\|b\|_1^2/\|b\|_2^2$ bounded by
$O_\kappa(k_*^2/n)=O_\kappa(n/L^2)$.

The extra factor of $L$ below the scale $n/L$ is important:
by \eqref{eq:sketch-effective-sparsity},
\[
 \frac Ln\|b\|_1^2
 \lesssim_\kappa\frac{k_*}{n}\|b\|_2^2
 =o(\|b\|_2^2)
 \qquad(F_\kappa(b)\le R).
\]
Thus, the correction term in the global Euclidean lower bound
\eqref{eq:lt-global-lower}, as well as
$\tau_\kappa^2\|b\|_1^2$, can be absorbed into a fixed positive
multiple of $\|b\|_2^2$.
An effective-sparsity bound of only $O_\kappa(n/L)$ would give
corrections of order $\|b\|_2^2$, without ensuring that their
constants are small enough for absorption.

For every $b$ with $F_\kappa(b)\le R$, we have
$\|y-\X b\|_2\le R-\tau_\kappa\|b\|_1$, with a nonnegative
right-hand side.
Squaring and applying the clipping bound gives
\begin{align*}
 \|\X b\|_2^2-\tau_\kappa^2\|b\|_1^2
 &\le2R\bigl(\langle q,b\rangle-\tau_\kappa\|b\|_1\bigr)\\
 &\le2RZ_\gamma\|b\|_2
 \lesssim\frac{k_*}{\sqrt n}\|b\|_2,
\end{align*}
where we used $R\asymp\sqrt{k_*}$ on the same high-probability
event.
By \eqref{eq:sketch-effective-sparsity} and the global lower bound
in Lemma~\ref{lem:effective-rip}, for all sufficiently large $n,p$,
\[
 \|\X b\|_2^2
 \ge\left(\frac14-C_\kappa\frac{k_*}{n}\right)\|b\|_2^2
 \ge\frac18\|b\|_2^2,
 \qquad
 \tau_\kappa^2\|b\|_1^2
 \lesssim_\kappa\frac{k_*}{n}\|b\|_2^2,
\]
where the second estimate follows from
\eqref{eq:sketch-effective-sparsity} and $k_*L/n\to1$.
Since $k_*/n\to0$, the quadratic penalty term can be absorbed
into the Euclidean lower bound.
The squared comparison therefore implies
$\|b\|_2\lesssim_\kappa k_*/\sqrt n$, and
\eqref{eq:sketch-effective-sparsity} then gives
$\|b\|_1\lesssim_\kappa k_*^2/n$.
These bounds also imply
\[
 \langle q,b\rangle
 \le\gamma\|b\|_1+Z_\gamma\|b\|_2
 \lesssim_\kappa\frac{k_*^{3/2}}{n}
 \lesssim_\kappa R\frac{k_*}{n}.
\]
Returning to $\|y-\X b\|_2\ge R-\langle q,b\rangle$ and using
$\|y-\X b\|_2\le R$, we obtain, uniformly over
$F_\kappa(b)\le R$,
\begin{equation}
 \|b\|_2^2+\|b\|_1\lesssim_\kappa\frac{k_*^2}{n},
 \qquad
 1-C_\kappa\frac{k_*}{n}
 \le\frac{\|y-\X b\|_2}{R}\le1.
 \label{eq:sketch-sublevel-energy}
\end{equation}
Since $\|\beta_*\|_\infty=1$, the first bound also implies
\[
 \bigl|\mathsf{MSE}_s(b)-1\bigr|
 \le\frac{\|b\|_2^2+2\|b\|_1}{k_*}
 \lesssim_\kappa\frac{k_*}{n}.
\]
Every minimizer satisfies
$F_\kappa(\widehat\beta_{\kappa,1})\le F_\kappa(0_p)=R$.
Since
\[
 \frac{k_*}{n}\approx\frac1L
 =\frac1{2\log(p/n)}\longrightarrow0,
\]
both normalized errors differ from one by
$O_\kappa(L^{-1})$ on the same high-probability event.
Consequently,
\[
 \left(
 \mathsf{MSE}_s(\widehat\beta_{\kappa,1}),
 \mathcal E_s(\widehat\beta_{\kappa,1})
 \right)
 \xrightarrow{\Pp}(1,1),
\]
completing the proof.
\begin{remark}[Connection to weak recovery]
\label{R:WR}
The recovery part of our proof of Lemma~\ref{thm:0-1grok}
gives a short, quantitative application of the Gaussian-distance
approach to sparse recovery and generalized LASSO
\citep{chandrasekaran2012convex,almt2014,oymak2013squared}.
A direct calculation of the distance to the dilated $\ell_1$
subdifferential establishes recovery at the leading
Donoho--Tanner weak-recovery scale
$k_*\approx \tfrac{n}{2\log(p/n)}$ \citep{donoho2009}, without counting
the faces of the random polytope $Q_{n,p}$.

% The argument uses Gordon's classical Gaussian comparison
% \citep{gordon1985}, rather than its convex strengthening underlying
% CGMT \citep{thrampoulidis2015gaussian}, and yields a uniform
% linear growth bound for the regularized objective around the
% ground truth.
\end{remark}

\begin{remark}[Between geometric and CGMT approaches]\label{R:CGMT}
Our argument combines the geometric viewpoint on minimum-norm
interpolation \citep{kur2026new}
with classical Gaussian comparison.
Rather than deriving a full asymptotic characterization of the
optimizer through CGMT, we control the objective directly around
the ground truth and the zero predictor.
This makes explicit the geometric and distributional ingredients
responsible for the transition.
A natural direction is to determine whether the same zero--one
law holds for non-Gaussian designs, such as matrices with
independent Rademacher entries. Developing geometric or universality arguments that replace
our Gaussian-specific arguments would be of independent interest.
\end{remark}
\subsection{Sketch of the proof of Theorem~\ref{thm:cube-grokking}}
\label{ss:hybrid-proof-sketch}

We use the notation of the preceding sketch. The hybrid norm adds a
second way to fit the response: a component with small $\ell_\infty$
norm that completes the prediction without contributing appreciably
to its normalized coefficient energy. We first show that this forces
exact interpolation, and then recover the same zero--one transition
as in Lemma~\ref{thm:0-1grok}.

\paragraph{The cost of completing a prediction.}
The additional geometric input is that the image $\X B_\infty^p$
is approximately Euclidean. For the completion cost in
\eqref{eq:lt-completion-norm},
\[
 f(z)=\sqrt{\frac2\pi}\frac p{\sqrt n}
       \min_{\X v=z}\|v\|_\infty,
\]
Lemma~\ref{lem:lt-cube} gives, with probability at least
$1-2e^{-n/2}$, simultaneously for all $z\in\R^n$,
\begin{equation}
 (1-\epsilon_{n,p})\|z\|_2\le f(z)
 \le(1+\epsilon_{n,p})\|z\|_2,
 \qquad
 \|v\|_2\lesssim\sqrt{\frac np}\,\|z\|_2,
 \label{eq:sketch-hybrid-completion}
\end{equation}
where $\epsilon_{n,p}=C\sqrt{n/p}=o(k_*/n)$ and $v$ is any
minimum-$\ell_\infty$ completion of $z$.
Thus, the normalized cost of fitting a response is close to its
Euclidean length, while the required coefficient energy is small
because $p\gg n$.

\paragraph{Exact interpolation for every $\kappa$.}
The objective in \eqref{eq:cube-hybrid-estimator}, denoted by
$J_\kappa$ in \eqref{eq:lt-hybrid-objective}, is
\[
 J_\kappa(b)=\|y-\X b\|_2+\frac{\tau_\kappa}{2}N_\kappa(b).
\]
For any $b$ with nonzero residual $e=y-\X b$, choose a
minimum-$\ell_\infty$ completion $v_e$ satisfying $\X v_e=e$.
The scaling in $N_\kappa$ gives
$\tau_\kappa N_\kappa(v_e)\le f(e)$.
Since $b+v_e$ interpolates, the triangle inequality yields
\begin{equation}
 J_\kappa(b+v_e)-J_\kappa(b)
 \le-\|e\|_2+\frac12f(e)
 \le-\frac{1-\epsilon_{n,p}}2\|e\|_2<0.
 \label{eq:sketch-hybrid-exact-fit}
\end{equation}
Hence every minimizer interpolates exactly and therefore minimizes
$N_\kappa$ subject to $\X b=y$; see
\eqref{eq:lt-optimizer-identity}.

For an optimal decomposition $b=u+v$, fixing $u$ and minimizing over
$v$ subject to $\X v=y-\X u$ leaves the objective
$H_\kappa(u)=f(y-\X u)+\tau_\kappa\|u\|_1$ from
\eqref{eq:lt-reduced-objective}. In particular,
\begin{equation}
 \min_{\X b=y}N_\kappa(b)
 =\frac1{\tau_\kappa}\min_u H_\kappa(u),
 \qquad
 \widetilde\beta_\kappa=u+v,
 \label{eq:sketch-hybrid-reduction}
\end{equation}
where $u$ minimizes $H_\kappa$ and $v$ is a minimum-$\ell_\infty$
completion of $y-\X u$.
This correspondence holds for every minimizer by
\eqref{eq:lt-reduction}. It remains to control $u$; the bound in
\eqref{eq:sketch-hybrid-completion} then controls $v$.

\paragraph{Exact recovery for $\kappa<1$.}
For every perturbation $h$, the subgradient inequality and
\eqref{eq:sketch-hybrid-completion} give
\[
 H_\kappa(\beta_*+h)-H_\kappa(\beta_*)
 \ge(1-\epsilon_{n,p})\|\X h\|_2+\tau_\kappa D(h).
\]
Using $\alpha$ from the preceding sketch, the right-hand side equals
\[
 \alpha\left(\|\X h\|_2+\frac\nu{\sqrt n}D(h)\right)
 +(1-\epsilon_{n,p}-\alpha)\|\X h\|_2
 \gtrsim_\kappa\eta\|h\|_2.
\]
Here we used Lemma~\ref{lem:width-margin} and the fact that
$\alpha\to\kappa<1$ and $\epsilon_{n,p}\to0$.
Thus $u=\beta_*$ is the unique minimizer of $H_\kappa$.
Its residual is zero, so its minimum-$\ell_\infty$ completion is
$v=0$. Therefore $\widetilde\beta_\kappa=\beta_*$ uniquely, with
probability at least $1-e^{-cn\eta^2}-2e^{-n/2}$.

\paragraph{The null side: $\kappa>1$.}
We work on the intersection of the events in
Lemmas~\ref{lem:initial-excess}, \ref{lem:effective-rip},
and~\ref{lem:lt-cube}, which has probability tending to one.
As in the preceding sketch, $R\asymp\sqrt{k_*}$ and
$Z_\gamma\lesssim\sqrt{k_*/n}$ on this event.
For any $w$ with $H_\kappa(w)\le H_\kappa(0_p)=f(y)$,
the cube estimate gives
\[
 (1-\epsilon_{n,p})\|y-\X w\|_2+\tau_\kappa\|w\|_1
 \le(1+\epsilon_{n,p})R.
\]
Combining this with $\|y-\X w\|_2\ge R-\langle q,w\rangle$
and the clipping inequality yields
\[
 \bigl(\tau_\kappa-(1-\epsilon_{n,p})\gamma\bigr)\|w\|_1
 \le(1-\epsilon_{n,p})Z_\gamma\|w\|_2+2\epsilon_{n,p}R.
\]
Since the coefficient on the left is at least
$\tau_\kappa-\gamma$, it follows that
\begin{equation}
 \|w\|_1\lesssim_\kappa
 \frac{k_*}{\sqrt n}\|w\|_2+\epsilon_{n,p}k_*.
 \label{eq:sketch-hybrid-coefficients}
\end{equation}
Squaring the residual comparison and applying the global Euclidean
lower bound \eqref{eq:lt-global-lower}, as in
\eqref{eq:lt-hybrid-quadratic}, then gives
\begin{equation}
 \left(\frac14-C_\kappa\frac{k_*}{n}\right)\|w\|_2^2
 \le C_\kappa\frac{k_*}{\sqrt n}\|w\|_2
       +C_\kappa\epsilon_{n,p}k_*.
 \label{eq:sketch-hybrid-boundary}
\end{equation}
For $\|w\|_2=Mk_*/\sqrt n$, dividing by $k_*^2/n$ shows that
this is impossible for a sufficiently large fixed $M=M_\kappa$,
since $k_*/n\to0$ and $\epsilon_{n,p}n/k_*\to0$.
Hence $H_\kappa(w)>H_\kappa(0_p)$ on that sphere.

Convexity now confines every minimizer $u$ of $H_\kappa$ to the
ball it bounds. Indeed, otherwise the segment from zero to $u$
would meet the sphere at $tu$, with $t\in(0,1]$, and
\[
 H_\kappa(tu)
 \le(1-t)H_\kappa(0_p)+tH_\kappa(u)
 \le H_\kappa(0_p),
\]
a contradiction. Thus $\|u\|_2\lesssim_\kappa k_*/\sqrt n$,
and \eqref{eq:sketch-hybrid-coefficients} implies
\begin{equation}
 \|u\|_2^2+\|u\|_1\lesssim_\kappa\frac{k_*^2}{n}.
 \label{eq:sketch-hybrid-u-energy}
\end{equation}

\paragraph{Completing the prediction and the risk bound.}
For every minimizer $u$, comparison with zero also gives
$f(y-\X u)\le H_\kappa(u)\le f(y)$, and therefore
$\|y-\X u\|_2\lesssim R\asymp\sqrt{k_*}$.
The coefficient-energy estimate in
\eqref{eq:sketch-hybrid-completion} yields
\[
 \frac{\|v\|_2^2}{k_*}\lesssim\frac np.
\]
Using $\|\beta_*\|_\infty=1$ and $\|\beta_*\|_2=\sqrt{k_*}$,
we conclude that
\begin{align*}
 \bigl|\mathsf{MSE}_s(u+v)-1\bigr|
 &\le\frac{\|u+v\|_2^2+2\|u\|_1}{k_*}
       +\frac{2\|v\|_2}{\sqrt{k_*}}\\
 &\lesssim_\kappa\frac{k_*}{n}+\sqrt{\frac np}
 =O_\kappa(L^{-1}).
\end{align*}
Thus the completed predictor has asymptotically trivial normalized
risk above the transition, despite fitting the response exactly.
Together with exact recovery below the transition, this proves,
for every fixed $\kappa\ne1$ and every minimizer,
\[
 \left(\mathsf{MSE}_s(\widetilde\beta_\kappa),
       \mathcal E_s(\widetilde\beta_\kappa)\right)
 \xrightarrow{\Pp}\left(\mathbf1_{\{\kappa>1\}},0\right).
\]

\subsection{Preparation for the proofs}
\label{subsec:lt-setup}

We use the notation of Section~\ref{sec:01} and the preceding sketch.
We first prove Lemma~\ref{thm:0-1grok}, bounding the coefficients of
all vectors whose objective is no larger than at zero.
For Theorem~\ref{thm:cube-grokking}, we use the same clipping and
Euclidean estimates, together with convexity, to control the
$\ell_1$ component. The cube-completion estimate controls the
$\ell_\infty$ component.

The dimensional assumptions and \eqref{eq:signal} imply
\begin{gather}
 k_*\approx\frac nL\longrightarrow\infty,\qquad
 \frac{k_*}{n}\approx\frac1L\longrightarrow0,\qquad
 n\eta^2\longrightarrow\infty,\notag\\
 \sqrt{\frac np}=e^{-L/4}=o(k_*/n).
 \label{eq:lt-scales}
\end{gather}
Indeed, $L\to\infty$ and $L=o(n)$ imply $n/L\to\infty$,
so the floor in $k_*$ has relative size $O(L/n)=o(1)$.
Also, $L^{-1}+L/n\ge2n^{-1/2}$, so the condition
$L^{-1}+L/n=o(\eta)$ yields $\eta\sqrt n\to\infty$.
We will also use $\eta L\to\infty$ and $L/n=o(\eta)$.

Moreover, the Gaussian model implies
\begin{equation}
 y\sim N(0,(k_*/n)I_n),\qquad
 \frac{R^2}{k_*}\overset{\mathrm{law}}=\frac{\chi_n^2}{n},\qquad
 \frac R{\sqrt{k_*}}\xrightarrow{\Pp}1.
 \label{eq:lt-response}
\end{equation}
For the hybrid norm in \eqref{eq:cube-hybrid-norm}, introduce the
abbreviations
\begin{equation}
 a:=\sqrt{\frac2\pi}\frac p{\sqrt n},\qquad
 m_\kappa:=\frac p\kappa\sqrt{\frac{2}{\pi L}},\qquad
 \tau_\kappa m_\kappa=a,
 \label{eq:lt-hybrid}
\end{equation}
and denote the objective in \eqref{eq:cube-hybrid-estimator} by
\begin{equation}
 J_\kappa(b):=\|y-\X b\|_2+\frac{\tau_\kappa}{2}N_\kappa(b).
 \label{eq:lt-hybrid-objective}
\end{equation}
We prove the limits in \eqref{eq:zero-one} and \eqref{eq:cube-zero-one}
for every minimizer. For the hybrid estimator, we also show that
$\X\widetilde\beta_\kappa=y$ exactly on a high-probability event.
All constants in this appendix are positive and independent of $n,p$;
subscripts indicate permitted dependence on fixed parameters.
The conclusions are for each fixed $\kappa\ne1$.

\subsection{Preliminaries}
\label{sec:preliminaries}
\label{app:gaussian-preliminaries}
We write $B_r^d$ for the unit $\ell_r$ ball and $S^{d-1}$ for the
Euclidean unit sphere. We use $\operatorname{dist}$ for Euclidean distance and
$\|A\|_{\mathrm F}:=(\sum_{i,j}A_{ij}^2)^{1/2}$ for the Frobenius norm.
For a nonempty compact set $T\subset\R^d$, define its Gaussian
width by
\[
 w(T):=\E\sup_{v\in T}\langle g_d,v\rangle,
 \qquad g_d\sim N(0,I_d).
\]
We use Gaussian concentration: for a standard Gaussian vector $Z$,
a $K$-Lipschitz function $f$, and $t>0$, each of
$\Pp\{f(Z)-\E f(Z)>t\}$ and
$\Pp\{f(Z)-\E f(Z)<-t\}$ is at most
$e^{-t^2/(2K^2)}$. Also, $\operatorname{Var}(f(Z))\le K^2$.
These imply
\begin{equation}
 \sqrt{n-1}\le \E\|g_n\|_2\le\sqrt n,
 \label{eq:lt-gaussian-length}
\end{equation}
since $\E\|g_n\|_2^2=n$ and the norm is one-Lipschitz.
See \citet{vershynin2018highdimensional} for these Gaussian facts.
\begin{lemma}[Gordon's comparison theorem]
\label{lem:lt-comparison}
Let $A\in\R^{m\times d}$ have independent $N(0,1)$ entries,
let $T\subset S^{d-1}$ and $U\subset S^{m-1}$ be deterministic,
nonempty compact sets, and let $f:T\to\R$ be deterministic
and continuous. Then
\begin{align}
 \E\inf_{v\in T}\sup_{u\in U}
 \{\langle u,Av\rangle-f(v)\}
 &\ge w(U)-\E\sup_{v\in T}
 \{\langle g_d,v\rangle+f(v)\},
 \label{eq:lt-gordon-general}\\
 \E\sup_{v\in T,\,u\in U}\langle u,Av\rangle
 &\le w(T)+w(U).
 \label{eq:lt-sup-comparison}
\end{align}
Moreover, for every $t>0$,
\begin{align*}
 \Pp\!\left\{
 \inf_{v\in T}\sup_{u\in U}
 \{\langle u,Av\rangle-f(v)\}
 <w(U)-\E\sup_{v\in T}
 \{\langle g_d,v\rangle+f(v)\}-t
 \right\}
 &\le e^{-t^2/2},\\
 \Pp\!\left\{
 \sup_{v\in T,\,u\in U}\langle u,Av\rangle
 >w(T)+w(U)+t
 \right\}
 &\le e^{-t^2/2}.
\end{align*}
\end{lemma}

\begin{lemma}
\label{lem:h0}
Let $h_0(t):=\E(|G|-t)_+^2$ for $G\sim N(0,1)$, and write
$\phi(t)=(2\pi)^{-1/2}e^{-t^2/2}$. For $t>0$,
\begin{equation}
 h_0(t)\le\frac{4\phi(t)}{t^3}.
 \label{eq:lt-excess-bound}
\end{equation}
Under the stated assumptions,
\begin{equation}
 \frac pn h_0(\nu)=O(L^{-1})=o(\eta),\qquad
 \frac{p}{k_*} h_0(\mu\sqrt L)
 \le C_\mu\frac{e^{-(\mu^2-1)L/2}}{\sqrt L}\longrightarrow0
 \quad(\mu>1\text{ fixed}).
 \label{eq:lt-excess-scales}
\end{equation}
\end{lemma}
\begin{lemma}[Width of effectively sparse vectors]
\label{lem:compressible-width}
For integers $1\le\ell\le p$, set
$T_\ell:=\{v\in S^{p-1}:\|v\|_1\le\sqrt\ell\}$. Then
\begin{equation}
 w(T_\ell)\le C\sqrt{\ell\log(ep/\ell)}.
 \label{eq:lt-width}
\end{equation}
\end{lemma}
\begin{lemma}
\label{lem:effective-rip}
Under the dimensional assumptions above, suppose that the rows of
$\sqrt n\,\X$ are independent, centered, isotropic sub-Gaussian
random vectors with sub-Gaussian norms bounded by a fixed $K$, namely
$\sup_{i\le n}\sup_{v\in S^{p-1}}
\|\sqrt n\,\langle\X_{i,:},v\rangle\|_{\psi_2}\le K$.
Then there exist constants $c_K,C_K>0$, depending only on $K$, such that,
for all sufficiently large dimensions, with probability at least
$1-e^{-c_K n}$, simultaneously for every $b\in\R^p$,
\begin{equation}
 \|\X b\|_2^2
 \ge \frac14\|b\|_2^2
      -C_K\frac Ln\|b\|_1^2.
 \label{eq:lt-global-lower}
\end{equation}
\end{lemma}
\subsection{Proof of Lemma~\ref{thm:0-1grok}}
\label{subsec:lt-lemma1}
\subsubsection{Recovery below the transition}
\label{sec:recovery}
The subdifferential appearing in \eqref{eq:directional-derivative} is
\begin{equation}
 \partial\|\beta_*\|_1
 =\{z\in\R^p:z_S=s,\ \|z_{S^c}\|_\infty\le1\}.
 \label{eq:lt-directional}
\end{equation}
The subgradient inequality gives
$\|\beta_*+h\|_1-k_*\ge D(h)$ for every $h \in \R^{p}$.

\begin{lemma}[Weighted recovery margin]
\label{lem:width-margin}
For all sufficiently large dimensions, with probability at least
$1-e^{-c n\eta^2}$, simultaneously for every $h\in\R^p$,
\begin{equation}
 \|\X h\|_2+\frac\nu{\sqrt n}D(h)
 \ge c\eta\|h\|_2.
 \label{eq:lt-weighted-margin}
\end{equation}
\end{lemma}
\begin{proof}
The squared distance to the dilated subdifferential has expectation
\begin{align}
 \E\operatorname{dist}^2(g_p,\nu\partial\|\beta_*\|_1)
 &=k_*(1+\nu^2)+(p-k_*)h_0(\nu),\notag\\
 \frac1n\E\operatorname{dist}^2(g_p,\nu\partial\|\beta_*\|_1)
 &\le1-\eta+\frac1{\nu^2}+\frac pn h_0(\nu)
 \le1-\frac\eta2.
 \label{eq:lt-distance-moment}
\end{align}
The first inequality uses $k_*\nu^2\le(1-\eta)n$; the second follows
from Lemma~\ref{lem:h0} and $\eta L\to\infty$. Write
\[
 \Delta:=\E\|g_n\|_2-
 \E\operatorname{dist}(g_p,\nu\partial\|\beta_*\|_1).
\]
Jensen's inequality and \eqref{eq:lt-gaussian-length} give
\[
 \Delta\ge\sqrt{n-1}-\sqrt{n(1-\eta/2)}
 \ge c\eta\sqrt n,
\]
since $\eta\sqrt n\to\infty$.
For each $g\in\R^p$, choose its Euclidean projection
$\nu z_g$ onto $\nu\partial\|\beta_*\|_1$. For every unit $v$,
\[
 \langle g,v\rangle-\nu D(v)
 \le\langle g-\nu z_g,v\rangle
 \le\operatorname{dist}(g,\nu\partial\|\beta_*\|_1).
\]
We now apply the first probability bound in Lemma~\ref{lem:lt-comparison}
to $A=\sqrt n\X$, $T=S^{p-1}$, $U=S^{n-1}$,
$f=-\nu D$, and $t=\Delta/2$. With probability at least
$1-e^{-\Delta^2/8}$,
\[
 \|\X h\|_2+\frac\nu{\sqrt n}D(h)
 \ge\frac{\Delta}{2\sqrt n}\|h\|_2
\]
for all $h$, by homogeneity. This proves the claim.
\end{proof}

Fix $0<\kappa<1$. On the event of Lemma~\ref{lem:width-margin},
use the coefficient $\alpha$ from the sketch. Since $\alpha\to\kappa$,
eventually $0<\alpha<1$ and, for every $h\ne0$,
\begin{align}
 F_\kappa(\beta_*+h)-F_\kappa(\beta_*)
 &\ge\|\X h\|_2+\tau_\kappa D(h)\notag\\
 &=\alpha\left(\|\X h\|_2+\frac\nu{\sqrt n}D(h)\right)
          +(1-\alpha)\|\X h\|_2\notag\\
 &\ge c_\kappa\eta\|h\|_2>0.
 \label{eq:lt-l1-growth}
\end{align}
Thus $\widehat\beta_{\kappa,1}=\beta_*$ uniquely, and both normalized
errors are exactly zero on this event.
The same event also gives uniqueness of the $\ell_1$-MNI:
if $\X h=0$ and $\|\beta_*+h\|_1\le k_*$, then $D(h)\le0$, contradicting
\eqref{eq:lt-weighted-margin} unless $h=0$.
\subsubsection{Clipping above the transition}
\label{sec:null-offsupport}
Fix $\kappa>1$ and use $q$, $\gamma$, and $Z_\gamma$ from
\eqref{eq:initial-correlations}. Set
\begin{equation}
 \mu:=\frac{1+\kappa}{2},\qquad \sqrt n\,\gamma=\mu\sqrt L.
 \label{eq:lt-correlations}
\end{equation}

\begin{lemma}[Initial correlation excess]
\label{lem:initial-excess}
There exists $c_\kappa>0$, depending only on $\kappa$, such that,
for all sufficiently large $n,p$, with probability at least
$1-2e^{-c_\kappa n}-e^{-k_*/2}$,
\begin{equation}
 \frac12\sqrt{k_*}\le R\le2\sqrt{k_*},\qquad
 Z_\gamma\le3\sqrt{\frac{k_*}{n}},\qquad
 \langle q,b\rangle\le\gamma\|b\|_1+Z_\gamma\|b\|_2
 \quad(b\in\R^p).
 \label{eq:lt-clipping}
\end{equation}
\end{lemma}
\begin{proof}
Let $P=I_p-\beta_*\beta_*^{\mathsf T}/k_*$.
Since $\|\beta_*\|_2^2=k_*$, $P$ is the orthogonal projection
onto $\beta_*^\perp$.
Decompose each row of $\X$ into its projection onto
$\operatorname{span}(\beta_*)$ and its projection onto
$\beta_*^\perp$.
Writing $B:=\X P$ and using $y=\X\beta_*$ gives
\[
 \X=\X(I_p-P)+\X P
 =\frac{y\beta_*^{\mathsf T}}{k_*}+B.
\]
For each row, the two projections are jointly Gaussian and have
zero cross-covariance, since
$n^{-1}(I_p-P)P=0$.
They are therefore independent. As the rows of $\X$ are independent,
$B$ is independent of $y$ and has independent $N(0,P/n)$ rows.
Consequently, conditional on $y$,
\begin{equation}
 q=\frac R{k_*}\beta_*+\frac g{\sqrt n},\qquad
 g:=\sqrt n\,B^{\mathsf T}\frac yR\sim N(0,P).
 \label{eq:lt-regression}
\end{equation}
Indeed, the conditional covariance of $g$ is
\[
 n\sum_{i=1}^n\frac{y_i^2}{R^2}\frac Pn=P.
\]

Let $\mathcal A_y$ denote the response event
\[
 \frac12\sqrt{k_*}\le R\le2\sqrt{k_*},
 \qquad \frac R{k_*}\le\gamma.
\]
Since $k_*L/n\to1$ and $\mu>1$, for all sufficiently large dimensions,
\[
 \mu\sqrt{\frac{k_*L}{n}}\ge\frac{1+\mu}{2}>1.
\]
Also, $R/\sqrt{k_*}$ has the same distribution as
$\|g_n\|_2/\sqrt n$.
Gaussian concentration for the Euclidean norm therefore gives
\[
 \begin{aligned}
 \Pp(\mathcal A_y^c)
 &\le
 \Pp\left(\frac R{\sqrt{k_*}}<\frac12\right)
 +\Pp\left(
    \frac R{\sqrt{k_*}}>
    \min\left\{2,\frac{1+\mu}{2}\right\}
   \right)\\
 &\le2e^{-c_\kappa n}.
 \end{aligned}
\]
Here the second threshold is a fixed number strictly greater
than one, so $c_\kappa>0$ can be chosen independently of $n,p$.

On $\mathcal A_y$, the conditional mean on $S$ lies in the
clipping cube: for $j\in S$,
$|(R/k_*)(\beta_*)_j|\le\gamma$.
Thus $(|q_j|-\gamma)_+\le |g_j|/\sqrt n$ on $S$, whereas
$q_j=g_j/\sqrt n$ off $S$. It follows that
\[
 Z_\gamma^2\le\frac1n\left(\|g_S\|_2^2+
          \sum_{j\notin S}(|g_j|-\mu\sqrt L)_+^2\right).
\]
Conditionally on $y$,
$\E[\|g_S\|_2^2\mid y]=\operatorname{tr}(P_{S,S})=k_*-1$,
and the off-support coordinates are independent standard Gaussians.
Lemma~\ref{lem:h0} therefore gives, on $\mathcal A_y$,
\[
 \E[Z_\gamma^2\mid y]
 \le\frac{k_*-1}{n}
      +\frac{p-k_*}{n}h_0(\mu\sqrt L)
 \le\frac{2k_*}{n}
\]
for all sufficiently large $n,p$.

In conditional distribution, realize $g=Ph$ with
$h\sim N(0,I_p)$.
The function
\[
 h\longmapsto\operatorname{dist}\left(
       (R/k_*)\beta_*+n^{-1/2}Ph,\gamma B_\infty^p
       \right)
\]
equals $Z_\gamma$ in this representation and it is
$n^{-1/2}$-Lipschitz as the distance to a fixed set is
one-Lipschitz and the operator norm of $P$ is at most one.
Its conditional mean is at most $\sqrt{2k_*/n}$ by Jensen's
inequality. Gaussian concentration therefore gives, on $\mathcal A_y$,
\[
 \Pp\left(Z_\gamma>3\sqrt{k_*/n}\mid y\right)
 \le
 \exp\left(-\frac{(3-\sqrt2)^2}{2}k_*\right)
 \le e^{-k_*/2}.
\]
Integrating over $y$ and combining this with the bound on
$\Pp(\mathcal A_y^c)$ gives the stated probability.

Finally, define the clipped vector coordinate-wise by
\[
 q_j^{\mathrm{clip}}
 :=\operatorname{sign}(q_j)\min\{|q_j|,\gamma\}.
\]
Then
\[
 q_j-q_j^{\mathrm{clip}}
 =\operatorname{sign}(q_j)(|q_j|-\gamma)_+,
 \qquad
 \|q^{\mathrm{clip}}\|_\infty\le\gamma,
 \qquad
 \|q-q^{\mathrm{clip}}\|_2=Z_\gamma.
\]
For every $b\in\R^p$, H\"older's and Cauchy--Schwarz inequalities give
\[
 \begin{aligned}
 \langle q,b\rangle
 &=\langle q^{\mathrm{clip}},b\rangle
   +\langle q-q^{\mathrm{clip}},b\rangle\\
 &\le\|q^{\mathrm{clip}}\|_\infty\|b\|_1
      +\|q-q^{\mathrm{clip}}\|_2\|b\|_2\\
 &\le\gamma\|b\|_1+Z_\gamma\|b\|_2.
 \end{aligned}
\]
This proves the last, deterministic inequality in
\eqref{eq:lt-clipping}.
To obtain the full clipping implication in
\eqref{eq:sketch-zero-plane}, Cauchy--Schwarz also gives
\[
 \|y-\X b\|_2
 \ge\left\langle\frac yR,y-\X b\right\rangle
 =R-\langle q,b\rangle.
\]
Hence, whenever $F_\kappa(b)\le R$,
\[
 \tau_\kappa\|b\|_1
 \le R-\|y-\X b\|_2
 \le\langle q,b\rangle
 \le\gamma\|b\|_1+Z_\gamma\|b\|_2,
\]
as claimed in the sketch.
\end{proof}

\subsubsection{Coefficient bounds above the transition}
\begin{proposition}[Comparison with the zero predictor]
\label{prop:sublevel-energy}
Fix $\kappa>1$. For all sufficiently large $n,p$, with probability
at least $1-3e^{-c_\kappa n}-e^{-k_*/2}$, simultaneously for every
$b\in\R^p$,
\begin{equation}
 F_\kappa(b)\le R
 \quad\Longrightarrow\quad
 \frac{\|b\|_2^2+\|b\|_1}{k_*}
 \le C_\kappa\frac{k_*}{n}.
 \label{eq:lt-robust-null}
\end{equation}
\end{proposition}
\begin{proof}
Let us intersect the event in Lemma~\ref{lem:initial-excess} with the event
in Lemma~\ref{lem:effective-rip}. Their intersection has the stated
probability by a union bound, after decreasing $c_\kappa$ if necessary.
Fix $b$ with $F_\kappa(b)\le R$, and put
$B_0=\|b\|_2$ and $U_0=\|b\|_1$.
By Cauchy--Schwarz and \eqref{eq:lt-clipping},
\[
 \|y-\X b\|_2\ge R-\langle q,b\rangle,
 \qquad
 \tau_\kappa U_0\le\langle q,b\rangle
 \le\gamma U_0+Z_\gamma B_0.
\]
Since $\tau_\kappa-\gamma=(\kappa-1)\sqrt{L/n}/2$ and
$k_*L/n\to1$, it follows that
\begin{equation}
 U_0\le\frac{Z_\gamma}{\tau_\kappa-\gamma}B_0
 \le C_\kappa\frac{k_*}{\sqrt n}B_0.
 \label{eq:lt-approx-compressibility}
\end{equation}
The objective comparison gives
$\|y-\X b\|_2\le R-\tau_\kappa U_0$, whose right-hand side
is nonnegative. Squaring and using clipping yields
\begin{align}
 \|\X b\|_2^2-\tau_\kappa^2U_0^2
 &\le2R\bigl(\langle q,b\rangle-\tau_\kappa U_0\bigr)\notag\\
 &\le2RZ_\gamma B_0
 \le C\frac{k_*}{\sqrt n}B_0.
 \label{eq:lt-squared-comparison}
\end{align}
Combining \eqref{eq:lt-global-lower} with these bounds and
$L/n\asymp1/k_*$ gives
\[
 \frac14B_0^2
 \le\frac{C_\kappa}{k_*}U_0^2
       +C\frac{k_*}{\sqrt n}B_0
 \le C_\kappa\frac{k_*}{n}B_0^2
       +C\frac{k_*}{\sqrt n}B_0.
\]
For sufficiently large $n,p$, the quadratic term on the right is
at most $B_0^2/8$. Thus $B_0\le C_\kappa k_*/\sqrt n$;
substitution into \eqref{eq:lt-approx-compressibility} gives
$U_0\le C_\kappa k_*^2/n$.
All inequalities hold for every $b$ on the same event, proving
\eqref{eq:lt-robust-null}.
\end{proof}

\subsubsection{Generalization and training errors}
The continuous objective $F_\kappa$ is coercive, so it attains its
minimum. Every minimizer $\widehat\beta_{\kappa,1}$ satisfies
$F_\kappa(\widehat\beta_{\kappa,1})\le F_\kappa(0_p)=R$.
Proposition~\ref{prop:sublevel-energy} therefore gives,
for every $b$ with $F_\kappa(b)\le R$,
\begin{equation}
 \|b\|_2^2+\|b\|_1\le C_\kappa\frac{k_*^2}{n},\qquad
 \left|\mathsf{MSE}_s(b)-1\right|
 \le\frac{\|b\|_2^2+2\|b\|_1}{k_*}
 \le C_\kappa\frac{k_*}{n}.
 \label{eq:lt-l1-energy}
\end{equation}
The middle inequality uses $|\langle\beta_*,b\rangle|\le\|b\|_1$.
For the training error, clipping and \eqref{eq:lt-l1-energy} give
\[
 \langle q,b\rangle
 \le\gamma\|b\|_1+Z_\gamma\|b\|_2
 \le C_\kappa\frac{k_*^{3/2}}{n}
 \le C_\kappa R\frac{k_*}{n}.
\]
The Cauchy--Schwarz inequality used above gives the lower bound
\[
 \|y-\X b\|_2
 \ge R-\langle q,b\rangle
 \ge R\left(1-C_\kappa\frac{k_*}{n}\right).
\]
For the upper bound, compare the objective at $b$ with the
objective at the zero predictor:
$F_\kappa(b)\le F_\kappa(0_p)=R$.
Dropping the nonnegative penalty gives
$\|y-\X b\|_2\le F_\kappa(b)\le R$.
These two comparisons yield the two-sided residual bound
\begin{equation}
 1-C_\kappa\frac{k_*}{n}
 \le\frac{\|y-\X b\|_2}{R}\le1
 \qquad(F_\kappa(b)\le R).
 \label{eq:lt-l1-residual}
\end{equation}
For sufficiently large $n,p$, the lower bound is positive,
so no minimizer interpolates.
By the definition of $\mathcal E_s$,
\eqref{eq:lt-l1-residual} directly gives
$|\mathcal E_s(b)-1|\le C_\kappa k_*/n$ on the same set.
Together with \eqref{eq:lt-l1-growth} and \eqref{eq:lt-l1-energy},
this proves both components of \eqref{eq:zero-one} and the coefficient
energy bound. The conclusions hold for every minimizer on a common
high-probability event for each fixed $\kappa\ne1$.

\subsection{Proof of Theorem~\ref{thm:cube-grokking}}
\label{subsec:lt-theorem1}
Recall from \eqref{eq:cube-hybrid-norm} that
\[
 N_\kappa(b)=\inf_{u+v=b}
 \left\{\|u\|_1+
 \frac{p}{\kappa}\sqrt{\frac{2}{\pi L}}\,\|v\|_\infty\right\}.
\]

\subsubsection{The infimal-convolution norm}
\begin{proposition}[Unit ball and norm comparison]
\label{lem:lt-norm}
For every $\kappa>0$, the infimum in \eqref{eq:cube-hybrid-norm}
is attained for every $b\in\R^p$, and $N_\kappa$ is a norm with unit ball
\[
 \{b\in\R^p:N_\kappa(b)\le1\}
 =\operatorname{conv}\left(
 B_1^p\cup\frac{\kappa}{p}\sqrt{\frac{\pi L}{2}}\,B_\infty^p
 \right).
\]
For each fixed $\kappa>0$ and all sufficiently large dimensions,
\begin{equation}
 B_1^p
 \subseteq\{b\in\R^p:N_\kappa(b)\le1\}
 \subseteq\kappa\sqrt{\frac{\pi L}{2}}\,B_1^p.
 \label{eq:lt-norm-inclusions}
\end{equation}
Equivalently,
\begin{equation}
 \frac1\kappa\sqrt{\frac{2}{\pi L}}\,\|b\|_1
 \le N_\kappa(b)\le\|b\|_1
 \qquad(b\in\R^p).
 \label{eq:lt-norm-comparison}
\end{equation}
\end{proposition}
\begin{proof}
For fixed $b$, the cost as a function of $u$, with $v=b-u$, is continuous
and tends to infinity as $\|u\|_2\to\infty$, so it attains its minimum.
A decomposition with cost at most one expresses $b$ as a convex
combination of a point in $B_1^p$, a point in
$m_\kappa^{-1}B_\infty^p$, and possibly zero. Conversely, grouping the
two types of terms in such a convex combination produces a decomposition
with cost at most one. Thus the unit sublevel set is the stated convex
hull. It is compact, convex, symmetric, and contains a neighborhood of
zero. Together with the positive homogeneity of the infimal convolution,
this identifies $N_\kappa$ as its Minkowski functional and proves that
$N_\kappa$ is a norm.

For fixed $\kappa$, $m_\kappa/p=\kappa^{-1}\sqrt{2/(\pi L)}<1$
for all sufficiently large $n,p$. Every decomposition then satisfies
\[
 \|u\|_1+m_\kappa\|v\|_\infty
 \ge\frac{m_\kappa}{p}\bigl(\|u\|_1+\|v\|_1\bigr)
 \ge\frac{m_\kappa}{p}\|b\|_1.
\]
Taking the infimum proves the lower bound in
\eqref{eq:lt-norm-comparison}; choosing $u=b$ and $v=0$ proves the
upper bound. Passing to unit balls gives \eqref{eq:lt-norm-inclusions}.
\end{proof}

For $z\in\R^n$, define the normalized cost of completing a prediction by
\begin{equation}
 f(z):=a\min_{\X v=z}\|v\|_\infty.
 \label{eq:lt-completion-norm}
\end{equation}
The minimum is attained because its constraint set is nonempty and
closed and the objective is coercive. Its unnormalized value is the
Minkowski functional of $\X B_\infty^p$, which is a norm since $\X$ has
full row rank. In particular, $f$ is even, convex, and positively
homogeneous.

\begin{lemma}[Uniform cube completion and its coefficient energy]
\label{lem:lt-cube}
There is an absolute $C>0$ such that, for all sufficiently large
dimensions, $\epsilon_{n,p}:=C\sqrt{n/p}<1/2$ and, with probability
at least $1-2e^{-n/2}$, the following bounds hold simultaneously for
every $z\in\R^n$:
\begin{equation}
 (1-\epsilon_{n,p})\|z\|_2\le f(z)\le(1+\epsilon_{n,p})\|z\|_2.
 \label{eq:lt-cube-round}
\end{equation}
On the same event, every minimum-$\ell_\infty$ completion satisfies
\begin{equation}
 v\in\operatorname*{arg\,min}_{\X w=z}\|w\|_\infty
 \quad\Longrightarrow\quad
 \|v\|_2\le C\sqrt{\frac np}\|z\|_2.
 \label{eq:lt-completion-energy}
\end{equation}
Moreover, $\epsilon_{n,p}=Ce^{-L/4}=o(k_*/n)$.
\end{lemma}
\begin{proof}
This is the Gaussian projection form of the Dvoretzky--Milman theorem
for the cube; see \citet[Theorem~11.3.3 and Example~11.3.6]{vershynin2018highdimensional}.
Its high-probability refinement (Exercise~11.3.5 there), obtained by
Gaussian concentration, gives in our normalization
\[
 a(1-\theta)B_2^n\subseteq\X B_\infty^p
                  \subseteq a(1+\theta)B_2^n,
 \qquad \theta=C_0\sqrt{n/p},
\]
with probability at least $1-2e^{-n/2}$, for an absolute $C_0>0$.
Here the scale is $a=n^{-1/2}w(B_\infty^p)$.
Passing to Minkowski functionals yields
\[
 \frac{\|z\|_2}{1+\theta}\le f(z)
                     \le\frac{\|z\|_2}{1-\theta}
 \qquad(z\in\R^n),
\]
which implies \eqref{eq:lt-cube-round} after choosing $C\ge2C_0$
and taking sufficiently large $n,p$.
Every minimum-$\ell_\infty$ completion satisfies
\[
 \|v\|_2\le\sqrt p\,\|v\|_\infty
 =\frac{\sqrt p}{a}f(z)
 \le C\sqrt{\frac np}\,\|z\|_2,
\]
proving \eqref{eq:lt-completion-energy} on the same event.
Finally, $\sqrt{n/p}=e^{-L/4}=o(L^{-1})$ and $k_*/n\approx L^{-1}$
give $\epsilon_{n,p}=o(k_*/n)$. All bounds are uniform in $z$.
\end{proof}

\subsubsection{Exact interpolation of every regularized minimizer}
Recall from \eqref{eq:lt-hybrid-objective} that
\[
 J_\kappa(b)=\|y-\X b\|_2+\frac{\tau_\kappa}{2}N_\kappa(b)
\]
is the objective defining $\widetilde\beta_\kappa$ in
\eqref{eq:cube-hybrid-estimator}.
We work on the event of Lemma~\ref{lem:lt-cube}.
For any $b\in\R^p$, set $e=y-\X b$ and choose a minimum-$\ell_\infty$
completion $v_e$ of $e$.
By \eqref{eq:cube-hybrid-norm}, \eqref{eq:lt-hybrid}, and
\eqref{eq:lt-completion-norm},
\[
 \tau_\kappa N_\kappa(v_e)
 \le\tau_\kappa m_\kappa\|v_e\|_\infty=f(e).
\]
Since $b+v_e$ interpolates, the triangle inequality and
\eqref{eq:lt-cube-round} give
\begin{align}
 J_\kappa(b+v_e)-J_\kappa(b)
 &\le-\|e\|_2+\frac{\tau_\kappa}{2}N_\kappa(v_e)\notag\\
 &\le-\|e\|_2+\frac12 f(e)
 \le-\frac{1-\epsilon_{n,p}}2\|e\|_2.
 \label{eq:lt-exact-fit-improvement}
\end{align}
Thus every noninterpolating vector is strictly suboptimal.
The objective $J_\kappa$ is coercive, and $N_\kappa$ attains its
minimum on the nonempty closed affine interpolation set.
Among interpolators, $J_\kappa$ is a positive multiple of $N_\kappa$.
Consequently,
\begin{equation}
 \operatorname*{arg\,min}_{b\in\R^p}J_\kappa(b)
 =\operatorname*{arg\,min}_{\X b=y}N_\kappa(b).
 \label{eq:lt-optimizer-identity}
\end{equation}
In particular, on this same event, every minimizer satisfies
\begin{equation}
 \X\widetilde\beta_\kappa=y,\qquad
 \mathcal E(\widetilde\beta_\kappa)
 =\mathcal E_s(\widetilde\beta_\kappa)=0.
 \label{eq:lt-exact-training}
\end{equation}
This establishes exact training fit before either generalization
regime is considered.

\subsubsection{Eliminating the \texorpdfstring{$\ell_\infty$}{l-infinity} component}
We write $b=u+v$ as in the definition of $N_\kappa$.
For fixed $u$, interpolation requires $\X v=y-\X u$, and the
smallest possible value of $\tau_\kappa m_\kappa\|v\|_\infty$
is $f(y-\X u)$. Thus, minimizing over $v$ leaves the objective
\begin{equation}
 H_\kappa(u):=f(y-\X u)+\tau_\kappa\|u\|_1.
 \label{eq:lt-reduced-objective}
\end{equation}
More precisely,
\begin{align}
 \min_{\X b=y}N_\kappa(b)
 &=\min_{\X(u+v)=y}\{\|u\|_1+m_\kappa\|v\|_\infty\}\notag\\
 &=\min_u\left\{\|u\|_1+\frac1{\tau_\kappa}f(y-\X u)\right\}
 =\frac1{\tau_\kappa}\min_u H_\kappa(u).
 \label{eq:lt-reduction}
\end{align}
The function $H_\kappa$ is continuous and coercive, so it attains
its minimum. If $b$ minimizes $N_\kappa$ subject to $\X b=y$,
every optimal decomposition $b=u+v$ must have $u$ minimizing
$H_\kappa$ and $v$ minimizing $\|v\|_\infty$ subject to
$\X v=y-\X u$; otherwise, replacing that component would lower
the joint cost in \eqref{eq:lt-reduction}.
Conversely, every such pair gives a minimum-$N_\kappa$ interpolator.
By \eqref{eq:lt-optimizer-identity}, every minimizer of $J_\kappa$
therefore has a decomposition
\begin{equation}
 \widetilde\beta_\kappa=u+v,\qquad
 u\in\operatorname*{arg\,min}_w H_\kappa(w),\qquad
 v\in\operatorname*{arg\,min}_{\X w=y-\X u}\|w\|_\infty.
 \label{eq:lt-residual-accounting}
\end{equation}
We will show that $u=\beta_*$ and $v=0$ when $\kappa<1$.
When $\kappa>1$, we bound $\|u\|_2^2+\|u\|_1$ using
$H_\kappa$, and $\|v\|_2$ using Lemma~\ref{lem:lt-cube}.
These bounds then control the MSE of $u+v$.

\subsubsection{Generalization below the transition}
Fix $0<\kappa<1$ and intersect the events in
Lemmas~\ref{lem:width-margin} and~\ref{lem:lt-cube}.
For every $h\in\R^p$, evenness of $f$ and the subgradient inequality give
\begin{align*}
 H_\kappa(\beta_*+h)-H_\kappa(\beta_*)
 &=f(\X h)+\tau_\kappa\bigl(\|\beta_*+h\|_1-k_*\bigr)\\
 &\ge(1-\epsilon_{n,p})\|\X h\|_2+\tau_\kappa D(h).
\end{align*}
Let us use the coefficient $\alpha$ from the sketch.
Since $\alpha\to\kappa<1$ and $\epsilon_{n,p}\to0$, eventually
$0<\alpha<1-\epsilon_{n,p}$. Hence the last expression equals
\[
 \alpha\left(\|\X h\|_2+\frac\nu{\sqrt n}D(h)\right)
 +(1-\epsilon_{n,p}-\alpha)\|\X h\|_2
 \ge c_\kappa\eta\|h\|_2,
\]
by \eqref{eq:lt-weighted-margin}.
It follows that $\beta_*$ is the unique minimizer of $H_\kappa$.
The minimum-$\ell_\infty$ completion of $y-\X\beta_*=0$ is uniquely
$v=0$. By \eqref{eq:lt-optimizer-identity} and \eqref{eq:lt-reduction},
every minimizer of $J_\kappa$ therefore equals $\beta_*$.
Thus, both $\mathsf{MSE}_s$ and $\mathcal E_s$ are exactly zero on an
event of probability at least
$1-e^{-cn\eta^2}-2e^{-n/2}$, by a union bound.

\subsubsection{Generalization above the transition}
Fix $\kappa>1$. Let $\Omega_\kappa$ be the intersection of the
events in Lemmas~\ref{lem:initial-excess},
\ref{lem:effective-rip}, and~\ref{lem:lt-cube}.
By a union bound, for sufficiently large $n,p$,
\[
 \Pp(\Omega_\kappa)
 \ge1-3e^{-c_\kappa n}-e^{-k_*/2}-2e^{-n/2}.
\]
We work on this event throughout.

We first show that, for a sufficiently large constant $M>0$
depending only on $\kappa$,
\[
 H_\kappa(w)>H_\kappa(0_p)
 \qquad\text{whenever}\qquad
 \|w\|_2=M\frac{k_*}{\sqrt n}.
\]
For any $w$ satisfying $H_\kappa(w)\le H_\kappa(0_p)=f(y)$,
\eqref{eq:lt-cube-round} gives
\begin{equation}
 (1-\epsilon_{n,p})\|y-\X w\|_2+\tau_\kappa\|w\|_1
 \le(1+\epsilon_{n,p})R.
 \label{eq:lt-hybrid-zero-comparison}
\end{equation}
Using $\|y-\X w\|_2\ge R-\langle q,w\rangle$ and
\eqref{eq:lt-clipping} gives
\[
 \bigl(\tau_\kappa-(1-\epsilon_{n,p})\gamma\bigr)\|w\|_1
 \le(1-\epsilon_{n,p})Z_\gamma\|w\|_2+2\epsilon_{n,p}R.
\]
The coefficient on the left is at least $\tau_\kappa-\gamma$.
Thus, as in \eqref{eq:lt-approx-compressibility},
\begin{equation}
 \|w\|_1\le C_\kappa\left(
      \frac{k_*}{\sqrt n}\|w\|_2+\epsilon_{n,p}k_*\right).
 \label{eq:lt-hybrid-compressibility}
\end{equation}
Next, isolate the residual in \eqref{eq:lt-hybrid-zero-comparison}
and square its nonnegative upper bound. Expanding and applying
clipping, with
$(1+\epsilon_{n,p})\tau_\kappa/(1-\epsilon_{n,p})^2>\gamma$,
gives
\begin{align*}
 \|\X w\|_2^2
 -\frac{\tau_\kappa^2}{(1-\epsilon_{n,p})^2}\|w\|_1^2
 &\le\frac{4\epsilon_{n,p}}{(1-\epsilon_{n,p})^2}R^2
       +2RZ_\gamma\|w\|_2\\
 &\le C\epsilon_{n,p}k_*
       +C\frac{k_*}{\sqrt n}\|w\|_2.
\end{align*}
Combining this with \eqref{eq:lt-global-lower},
\eqref{eq:lt-hybrid-compressibility}, and $L/n\asymp1/k_*$ yields
\begin{equation}
 \left(\frac14-C_\kappa\frac{k_*}{n}\right)\|w\|_2^2
 \le C\frac{k_*}{\sqrt n}\|w\|_2+C_\kappa\epsilon_{n,p}k_*.
 \label{eq:lt-hybrid-quadratic}
\end{equation}
If $\|w\|_2=Mk_*/\sqrt n$, division by $k_*^2/n$ gives
\[
 \left(\frac14-C_\kappa\frac{k_*}{n}\right)M^2
 \le CM+C_\kappa\epsilon_{n,p}\frac n{k_*}.
\]
Since $k_*/n\to0$ and
$\epsilon_{n,p}n/k_*\asymp Le^{-L/4}\to0$, this is impossible
for a sufficiently large fixed $M$ and all sufficiently large $n,p$.
This proves the claimed boundary inequality.

Now let $u$ be any minimizer of $H_\kappa$. If
$\|u\|_2\ge Mk_*/\sqrt n$, choose $t\in(0,1]$ so that
$\|tu\|_2=Mk_*/\sqrt n$.
Convexity gives
\[
 H_\kappa(tu)
 \le(1-t)H_\kappa(0_p)+tH_\kappa(u)
 \le H_\kappa(0_p),
\]
contradicting the boundary inequality. Hence
$\|u\|_2<Mk_*/\sqrt n$.
Applying \eqref{eq:lt-hybrid-compressibility} to $u$ gives
\begin{equation}
 \frac{\|u\|_2^2+\|u\|_1}{k_*}
 \le C_\kappa\left(\frac{k_*}{n}+\epsilon_{n,p}\right)
 \le C_\kappa\frac{k_*}{n}.
 \label{eq:lt-sparse-part-energy}
\end{equation}

For every associated minimum-$\ell_\infty$ completion $v$,
\[
 f(y-\X u)\le H_\kappa(u)\le f(y),\qquad
 \|y-\X u\|_2\le\frac{1+\epsilon_{n,p}}{1-\epsilon_{n,p}}R\le CR.
\]
Lemma~\ref{lem:lt-cube} and $R\le2\sqrt{k_*}$ imply
\begin{equation}
 \|v\|_2\le C\sqrt{\frac np}R,
 \qquad \frac{\|v\|_2^2}{k_*}\le C\frac np.
 \label{eq:lt-dense-part-energy}
\end{equation}
Consequently, for $b=u+v$,
\begin{equation}
 \frac{\|b\|_2^2}{k_*}
 \le\frac{2\|u\|_2^2+2\|v\|_2^2}{k_*}
 \le C_\kappa\left(\frac{k_*}{n}+\frac np\right)
 \le C_\kappa\frac{k_*}{n}.
 \label{eq:lt-full-energy}
\end{equation}
The inequalities $|\langle\beta_*,u\rangle|\le\|u\|_1$ and
$|\langle\beta_*,v\rangle|\le\sqrt{k_*}\|v\|_2$ then give
\begin{align}
 \bigl|\mathsf{MSE}_s(b)-1\bigr|
 &\le\frac{\|b\|_2^2}{k_*}
       +\frac{2\|u\|_1}{k_*}+\frac{2\|v\|_2}{\sqrt{k_*}}\notag\\
 &\le C_\kappa\left(\frac{k_*}{n}+\sqrt{\frac np}\right)
 \le C_\kappa\frac{k_*}{n}.
 \label{eq:lt-hybrid-risk}
\end{align}
Here $\sqrt{n/p}=o(k_*/n)$, and therefore also $n/p=o(k_*/n)$.
All these estimates hold on $\Omega_\kappa$ for every minimizer
$u$ of $H_\kappa$ and every associated minimizing completion $v$.
By \eqref{eq:lt-reduction}, they therefore hold for every minimizer
of $J_\kappa$.
Since $k_*/n\approx L^{-1}$, its normalized MSE differs from one by
$O_\kappa(L^{-1})$ on $\Omega_\kappa$, while its training error is
exactly zero by \eqref{eq:lt-exact-training}.
Together with exact recovery for $\kappa<1$, this proves
\eqref{eq:cube-zero-one} and completes the proof of
Theorem~\ref{thm:cube-grokking}.

\subsection{Proofs of the preliminary lemmas}
\label{subsec:lt-classical-proofs}
We prove the preliminary lemmas used above.
\begin{proof}[Proof of Lemma~\ref{lem:lt-comparison}]
First suppose that $T$ and $U$ are finite.
Let $z\sim N(0,1)$, $g_m$, and $g_d$ be independent of each other
and of $A$, and compare
\[
 P_{v,u}=\langle u,Av\rangle+z,\qquad
 Q_{v,u}=\langle g_m,u\rangle-\langle g_d,v\rangle.
\]
Both processes are centered and have variance two. Their covariance
difference is
\[
 \operatorname{Cov}(P_{v,u},P_{v',u'})
 -\operatorname{Cov}(Q_{v,u},Q_{v',u'})
 =(1-\langle v,v'\rangle)(1-\langle u,u'\rangle)\ge0,
\]
with equality when $v=v'$.
Apply Gordon's comparison in
\citet[Theorem~A.1]{thrampoulidis2015gaussian}, with its process
$X$ equal to $Q$, its process $Y$ equal to $P$, and thresholds
$c+f(v)$. For every $c\in\R$, it gives
\[
 \Pp\!\left\{\min_{v\in T}\max_{u\in U}(P_{v,u}-f(v))\ge c\right\}
 \ge
 \Pp\!\left\{\min_{v\in T}\max_{u\in U}(Q_{v,u}-f(v))\ge c\right\}.
\]
Taking expectations, using $\E z=0$, and separating the two
terms in $Q$ yields \eqref{eq:lt-gordon-general}.
For the upper bound, apply the single-row case of the same
Theorem~A.1 (Slepian's comparison), now with $X=P$, $Y=Q$ and
index set $T\times U$. The covariance ordering gives
\[
 \E\max_{v\in T,\,u\in U}P_{v,u}
 \le\E\max_{v\in T,\,u\in U}Q_{v,u}=w(T)+w(U).
\]
Again $\E z=0$, so this proves \eqref{eq:lt-sup-comparison}.

For compact $T,U$, choose finite $\varepsilon$-nets contained in
the respective sets, with $\varepsilon\downarrow0$. Let
\[
 \omega_f(\varepsilon)
 :=\sup\{|f(v)-f(v')|:v,v'\in T,\ \|v-v'\|_2\le\varepsilon\}.
\]
Uniform continuity gives $\omega_f(\varepsilon)\to0$.
Replacing each index by a nearest net point changes $P_{v,u}-f(v)$
by at most $2\varepsilon\|A\|_{\mathrm F}+\omega_f(\varepsilon)$,
and $Q_{v,u}-f(v)$ by at most
$\varepsilon(\|g_m\|_2+\|g_d\|_2)+\omega_f(\varepsilon)$.
Thus the corresponding extrema converge almost surely.
Their absolute values are bounded, respectively, by
\[
 \|A\|_{\mathrm F}+|z|+\|f\|_\infty,
 \qquad
 \|g_m\|_2+\|g_d\|_2+\|f\|_\infty,
\]
where $\|f\|_\infty=\sup_{v\in T}|f(v)|<\infty$.
Both bounds are integrable, so dominated convergence passes the
expectation comparisons to $T,U$. The same argument with $f=0$
applies to the supremum comparison.

Finally, both random extrema in the lemma are $1$-Lipschitz in
$A$, since
\[
 \sup_{u\in U,\,v\in T}|\langle u,(A-B)v\rangle|
 \le\|A-B\|_{\mathrm F}.
\]
Gaussian concentration gives the stated tail bounds.
\end{proof}

\begin{proof}[Proof of Lemma~\ref{lem:h0}]
A change of variables gives
\[
 h_0(t)=2\phi(t)\int_0^\infty x^2e^{-tx-x^2/2}\,\mathrm dx
 \le2\phi(t)\int_0^\infty x^2e^{-tx}\,\mathrm dx
 =\frac{4\phi(t)}{t^3}.
\]
Since $p/n=e^{L/2}$, $\nu\asymp\sqrt L$, and
$e^{-\nu^2/2}\le C e^{-L/2}\sqrt L$,
\[
 \frac pn h_0(\nu)=O(L^{-1})=o(\eta).
\]
For fixed $\mu>1$, substituting $t=\mu\sqrt L$ and using
$k_*\asymp n/L$ gives
\[
 \frac p{k_*}h_0(\mu\sqrt L)
 \le C_\mu\frac{e^{-(\mu^2-1)L/2}}{\sqrt L}\longrightarrow0.
\]
These are the two bounds in \eqref{eq:lt-excess-scales}.
\end{proof}

\begin{proof}[Proof of Lemma~\ref{lem:compressible-width}]
For $v\in T_\ell$, order the coordinates by decreasing magnitude and
partition them into successive blocks $I_1,I_2,\ldots$ of size $\ell$,
except possibly the last.
For $j\ge2$, each entry in $I_j$ is bounded by the average absolute
entry in $I_{j-1}$, so
\[
 \|v_{I_j}\|_2\le\frac{\|v_{I_{j-1}}\|_1}{\sqrt\ell},\qquad
 \sum_j\|v_{I_j}\|_2
 \le\|v\|_2+\frac{\|v\|_1}{\sqrt\ell}\le2.
\]
It follows that
\[
 \sup_{v\in T_\ell}\langle g_p,v\rangle
 \le2\max_{|I|=\ell}\|(g_p)_I\|_2.
\]
Writing $V=\max_{|I|=\ell}\|(g_p)_I\|_2$, we have
\[
 \E e^{V^2/4}
 \le\sum_{|I|=\ell}\E e^{\|(g_p)_I\|_2^2/4}
 =\binom p\ell 2^{\ell/2}.
\]
Jensen's inequality and $\binom p\ell\le(ep/\ell)^\ell$ give
\[
 \E V\le(\E V^2)^{1/2}
 \le C\sqrt{\ell\log(ep/\ell)},
\]
which proves \eqref{eq:lt-width}.
\end{proof}

\begin{proof}[Proof of Lemma~\ref{lem:effective-rip}]
Set $\ell=\lfloor a_K n/L\rfloor$, where $a_K>0$ will be chosen
sufficiently small. Since $L\to\infty$ and $L=o(n)$, eventually
$1\le\ell\le p$ and
\[
 \log(ep/\ell)=\frac L2+\log L+O_K(1).
\]
Lemma~\ref{lem:compressible-width} therefore gives
\[
 w(T_\ell)\le C\sqrt{\ell\log(ep/\ell)}
 \le C\sqrt{a_K n}
\]
for all sufficiently large $n,p$.

By \citet[Theorem~1.4]{liawmehrabianplanvershynin2017}, applied to
$\sqrt n\,\X$ and $T_\ell\subset S^{p-1}$, for every $t\ge1$,
with probability at least $1-e^{-t^2}$,
\[
 \sup_{v\in T_\ell}\bigl|\|\sqrt n\,\X v\|_2-\sqrt n\bigr|
 \le CK^2\bigl(w(T_\ell)+t\bigr).
\]
Here $T_\ell$ has Euclidean radius one.
Choose $a_K>0$, $b_K>0$ sufficiently small and take
$t=b_K\sqrt n$. For all sufficiently large dimensions, the right-hand
side is at most $\sqrt n/2$. After decreasing $c_K>0$, the resulting
event has probability at least $1-e^{-c_K n}$, and on this event
\[
 \|\X v\|_2\ge\frac12\qquad(v\in T_\ell).
\]
By homogeneity,
\[
 \|\X b\|_2^2\ge\frac14\|b\|_2^2
 \quad\text{whenever}\quad
 \|b\|_1\le\sqrt\ell\,\|b\|_2.
\]
For vectors outside this class,
\[
 \frac14\|b\|_2^2-\frac1{4\ell}\|b\|_1^2\le0,
\]
so nonnegativity of $\|\X b\|_2^2$ gives the same lower bound with
the correction term. Combining the two cases and using
$1/\ell\le C_K L/n$ proves, simultaneously for all $b\in\R^p$,
\[
 \|\X b\|_2^2\ge\frac14\|b\|_2^2-C_K\frac Ln\|b\|_1^2,
\]
which is \eqref{eq:lt-global-lower}.
\end{proof}

\section{Proof of Lemma \ref{thm:main-nonint}}\label{app:pfnon-int}

Throughout the remaining appendices, for
$a,b\in\mathbb R^p$ and $1\le s<\infty$, we use the shorthands
\[
\|b\|_{s,p}:=p^{-1/s}\|b\|_s,\qquad
\langle a,b\rangle_p:=p^{-1}\langle a,b\rangle,\qquad
\|b\|_{\infty,p}:=\max_{1\le j\le p}|b_j|.
\]
We note that the proof presented in this appendix gives a result stronger than the convergence in \eqref{eq:resclaim}, namely the convergence in $W_2$ of the joint empirical distribution of $(\widehat\beta_{\kappa,r}, \beta_*)$, as detailed in the remark below. 

\begin{remark}[Joint empirical laws]\label{rem:joint-empirical-laws}
Under the
assumptions of Lemmas~\ref{thm:main-nonint} and
\ref{thm:min-norm-interpolator}, our argument establishes the joint empirical
distribution of true and estimated coefficients:
\begin{align}
 \frac1p\sum_{j=1}^p
 \delta_{((\beta_*)_j,(\widehat\beta_{\kappa,r})_j)}
 &\xrightarrow{W_2}\mathcal L(B,f_\kappa),\qquad
 \frac1p\sum_{j=1}^p
 \delta_{((\beta_*)_j,(\widehat\beta_{r})_j)}
 \xrightarrow{W_2}\mathcal L(B,f_{\mathrm{int}}),
 \label{eq:int-law}
\end{align}
where $f_\kappa:=f_{\tau_\kappa,\alpha_\kappa}$ and
$f_{\mathrm{int}}:=f_{\tau_{\mathrm{int}},\alpha_{\mathrm{int}}}$. This strengthens the results of Lemmas \ref{thm:main-nonint}-\ref{thm:min-norm-interpolator} on the MSE. In fact, \eqref{eq:int-law} implies the convergence of empirical averages of any continuous function of true and estimated coefficients with at most quadratic growth, including MSE, mean absolute error, normalized correlations between vectors as well as their $\ell_r$ norms for $r\in [1, 2]$.
\end{remark}

\subsection{Sketch of the argument}\label{app:sketchnonint}

We first express the optimization problem~\eqref{eq:minprob} in a form
to which the convex Gaussian min--max theorem (CGMT) applies. Write
$G_X:=\sqrt nX$ and $w:=b-\beta_*$. Dividing the objective by $\sqrt n$
and using $\|a\|_2=\max_{\|u\|_2\le1}u^\top a$ gives the primary
objective (see~\eqref{eq:primary-objective}):
\[
\begin{aligned}
P_p(w)
&=\frac1{\sqrt n}\left\|\varepsilon-\frac{G_Xw}{\sqrt n}\right\|_2
  +\frac{\kappa}{\delta_p}\|\beta_*+w\|_{r,p}=\max_{\|u\|_2\le1}
  \left\{\frac{u^\top\varepsilon}{\sqrt n}
  -\frac{u^\top G_Xw}{n}
  +\frac{\kappa}{\delta_p}\|\beta_*+w\|_{r,p}\right\}.
\end{aligned}
\]
Thus, $b$ solves~\eqref{eq:minprob} exactly when $w=b-\beta_*$ minimizes $P_p$.

Conditionally on $\varepsilon$, the CGMT associates to this minimization
an auxiliary problem in which the matrix term $-u^\top G_Xw/n$ is
replaced by
$\bigl(\|w\|_2g^\top u-\|u\|_2h^\top w\bigr)/n$,
where $g\in\mathbb R^n$ and $h\in\mathbb R^p$ are independent standard
Gaussian vectors, independent of $(G_X,\varepsilon)$.
Maximizing first over the direction of $u$ and then over its length
$\|u\|_2\in[0,1]$ gives the following auxiliary objective (see~\eqref{eq:rawAO}):
\[
A_p^{\mathrm{raw}}(b)
=\left[
  \left\|\frac{\varepsilon}{\sqrt n}
  +\frac{\|b-\beta_*\|_2}{n}g\right\|_2
  -\frac1n h^\top(b-\beta_*)
  \right]_+
  +\frac\kappa{\delta_p}\|b\|_{r,p}.
\]
Here $[x]_+=\max\{x,0\}$; see Appendix~\ref{app:completion-nonint}
for the derivation.

The norm in this expression converges to
$\sqrt{\sigma^2+\delta_p^{-1}\|b-\beta_*\|_{2,p}^2}$,
uniformly when $\|b-\beta_*\|_{2,p}$ is bounded. Thus, $A_p^{\mathrm{raw}}(b)$ is approximated by %yields the
%approximation~\eqref{eq:normreplace} by the objective
%in~\eqref{eq:detAO}: %, with $C_p$ defined in~\eqref{eq:int-det-constraint}:
\[
\begin{aligned}
\overline A_p(b)&=[C_p(b)]_++\frac\kappa{\delta_p}\|b\|_{r,p},\\
C_p(b)&=\sqrt{\sigma^2+\frac1{\delta_p}\|b-\beta_*\|_{2,p}^2}
       -\frac1{\delta_p}\langle h,b-\beta_*\rangle_p, 
\end{aligned}
\]
see \eqref{eq:normreplace}.
Appendix~\ref{app:auxiliary-optimization} bounds
$\|b-\beta_*\|_{2,p}$ and $\|b\|_{r,p}$ uniformly over fixed sublevel sets
of both auxiliary objectives (see~\eqref{eq:AO-sublevel-localization}), so the approximation applies at their minimizers.
Dropping the positive part gives the empirical
functional in~\eqref{eq:empQ}:
\[
Q_p(b)
=\sqrt{\sigma^2+\frac1{\delta_p}\|b-\beta_*\|_{2,p}^2}
 -\frac1{\delta_p}\langle h,b-\beta_*\rangle_p
 +\frac\kappa{\delta_p}\|b\|_{r,p}.
\]
Thus, $Q_p(b)\le\overline A_p(b)$, with equality whenever $C_p(b)\ge0$.
We will show that this lower bound is asymptotically attained.

The population counterpart of $Q_p$ replaces coordinate averages by
expectations and $\delta_p$ by $\delta$. The true coefficient, fitted
coefficient, and auxiliary Gaussian coordinate are represented by
$B$, $f$, and $G$, respectively, where $f$ ranges over $L_2\cap L_r$
on the probability space supporting $(B,G)$. This gives the population
objective in~\eqref{eq:popQ}:
\[
\cQ(f)
=\sqrt{\sigma^2+\frac1\delta\mathbb E(f-B)^2}
 -\frac1\delta\mathbb E[G(f-B)]
 +\frac\kappa\delta\|f\|_r,
\qquad \|f\|_r=(\mathbb E|f|^r)^{1/r}.
\]
The fixed point equations in \eqref{eq:se1} ensure that the
random variable in~\eqref{eq:fstar},
\[
f_\kappa=\tau_\kappa\eta_{\alpha_\kappa,r}
\left(\frac B{\tau_\kappa}+G\right),
\]
satisfies the optimality conditions of $\cQ$.
Appendix~\ref{app:popobj} proves that it is the unique minimizer,
including when $r=1$.

The minimizing vectors depend on the Gaussian data, so passing to the
population problem requires norm bounds and stability.
Appendix~\ref{app:empirical-margin} establishes convergence of the joint
empirical law of $((\beta_*)_j,h_j)$ to $\mathcal L(B,G)$ and bounds the
Gaussian inner products. These bounds imply
$Q_p(b)\ge c\|b\|_{2,p}-C$ as in~\eqref{eq:Qcoercive}, so
near-minimizers have bounded normalized $\ell_2$ norms. A compactness
argument then gives $\min_bQ_p(b)\to\cQ(f_\kappa)$
in~\eqref{eq:minQ-convergence} and shows that every sequence
of near-minimizers has limiting squared error $\mathbb E(f_\kappa-B)^2$,
coefficient moment $\mathbb E|f_\kappa|^r$, and joint empirical law
$\mathcal L(B,f_\kappa)$ in $W_2$; see
Appendix~\ref{app:empirical-stability}.

To see why dropping the positive part gives the correct auxiliary
minimum, evaluate the objectives at
$b_{p,j}^\circ=\tau_\kappa\eta_{\alpha_\kappa,r}
((\beta_*)_j/\tau_\kappa+h_j)$.
The empirical convergence results give
$Q_p(b_p^\circ)\to\cQ(f_\kappa)$ and
$C_p(b_p^\circ)\to R_\kappa>0$.
Hence $\overline A_p(b_p^\circ)=Q_p(b_p^\circ)$ eventually.
Combining this upper bound with $\overline A_p\ge Q_p$ and the uniform
norm approximation shows that all three minimum values converge
almost surely to the value in~\eqref{eq:ellstar}
(see~\eqref{eq:minQ-convergence} and~\eqref{eq:AO-minima-formal}):
\[
\ell_\kappa=\cQ(f_\kappa)
=R_\kappa+\frac\kappa\delta\|f_\kappa\|_r.
\]
The norm bounds and approximation also carry the stability of
$Q_p$ to $A_p^{\mathrm{raw}}$; see
Appendix~\ref{app:auxiliary-optimization}.

Finally, vectors whose error, coefficient moment, or empirical law
stays away from these limits have an auxiliary objective value
separated from its minimum. Appendix~\ref{app:cgmt-transfer} first
confines the primary and auxiliary minimizers to deterministic sets
with bounded normalized norms. An application of CGMT (as in \cite{thrampoulidis2015gaussian}) 
exclude such deviations for every primary minimizer almost surely,
giving that 
$\min_wP_p(w)\to\ell_\kappa$ almost surely (see
as in~\eqref{eq:primary-value-limit}). Subtracting the regularization gives
\[
\frac1{\sqrt n}\|y-X\widehat\beta_{\kappa,r}\|_2
\to
\ell_\kappa-\frac\kappa\delta\|f_\kappa\|_r=R_\kappa,
\]
as detailed in Appendix~\ref{app:completion-nonint}.
 
 We note that several intermediate results are
established under the weaker condition
$D_r(\tau_\kappa,\alpha_\kappa)\leq\delta$, allowing their reuse in the
interpolation regime considered in Appendix~\ref{app:pfint}.

\subsection{Uniqueness of the minimizer of the population objective}\label{app:popobj}

For
$1<r<\infty$ and $\alpha\ge 0$, recall
\[
 \eta_{\alpha,r}(z):=\argmin_{u\in\R}
 \left\{\frac12(u-z)^2+\frac\alpha r|u|^r\right\}.
\]
and let
\[
 \psi_r(u):=\operatorname{sign}(u)|u|^{r-1}.
\]

For $r=1$, set $\eta_{\alpha,1}(z):=\ST(z;\alpha)
:=\operatorname{sign}(z)(|z|-\alpha)_+$. 
Setting the derivative to $0$ in the optimization problem above gives the following equation for the proximal operator $\eta_{\alpha, r}$:
\begin{equation}\label{eq:prox-eq}
 z=\eta_{\alpha,r}(z)+\alpha\psi_r(\eta_{\alpha,r}(z)).
\end{equation}

We start with two auxiliary lemmas that will be used throughout the argument.

\begin{lemma}[Results on proximal operator]\label{lem:prox}
For every $1<r<\infty$ and $\alpha>0$, $\eta_{\alpha,r}$ is odd,
nondecreasing, and $1$-Lipschitz.  Moreover,
\begin{equation}\label{eq:prox-bounds}
 |\eta_{\alpha,r}(z)|\le |z|,
 \qquad
 \eta'_{\alpha,r}(z)
 =\frac1{1+\alpha(r-1)|\eta_{\alpha,r}(z)|^{r-2}}\in[0,1],
\end{equation}
with the continuous value $0$ at $\eta_{\alpha,r}(z)=0$ when $1<r<2$.
The same oddness, monotonicity, Lipschitz, and growth conclusions hold for
$\ST(\,\cdot\,;\alpha)$.
\end{lemma}

\begin{proof}
The objective defining $\eta_{\alpha,r}(z)$ has a unique minimizer. In fact, the objective is continuous and coercive, because its quadratic term tends to $+\infty$ as $|u|\to+\infty$. It is also strictly convex, since the quadratic term is strictly convex and $|u|^r$ is convex.

The minimizer satisfies \eqref{eq:prox-eq}, due to the first-order optimality condition. Furthermore, the map
 $T_\alpha(u):=u+\alpha\psi_r(u)$ is an odd, continuous, strictly increasing
 bijection of $\R$, as 
 $T_\alpha'(u)=1+\alpha(r-1)|u|^{r-2}>0$ for $u\ne0$. %, and since $T_\alpha(u)\to\pm\infty$ as $u\to\pm\infty$, the remaining assertions follow.
This readily implies that the inverse $\eta_{\alpha,r}=T_\alpha^{-1}$ is odd and nondecreasing. 

To see that $|\eta_{\alpha,r}(z)|\le |z|$, note that, if $z\ge0$ and
 $u=\eta_{\alpha,r}(z)$, then $u\ge0$ and $z=u+\alpha u^{r-1}\ge u$; furthermore, for $z<0$, the claim follows from the fact that $\eta_{\alpha, r}$ is odd.

Formula \eqref{eq:prox-bounds} for $\eta'_{\alpha, r}$ holds wherever the inverse-function theorem
 applies, and its RHS extends continuously at the origin. To see this, 
  differentiate $z=u+\alpha\psi_r(u)$.  For $1<r<2$, the denominator tends
 to $+\infty$ as $u\to0$; for $r\ge2$, \eqref{eq:prox-bounds} is immediate.  Its
 value lies in $[0,1]$, so $\eta_{\alpha,r}$ is globally $1$-Lipschitz.

The assertions for $\ST(\,\cdot\,;\alpha)$ follow directly from the fact that the function has three affine
 pieces.
\end{proof}

%Next, we show a useful consequence of having a zero prior.

\begin{lemma}[Zero-prior gap]\label{lem:null-prior-gap}
If \(B=0\) almost surely, then, for every \(\tau,\alpha>0\),
\begin{equation}
 D_r(\tau,\alpha)
 =
 H_r(\tau,\alpha)+\alpha U_r(\tau,\alpha)
 >
 H_r(\tau,\alpha).
 \label{eq:null-prior-gap}
\end{equation}
\end{lemma}

\begin{proof}
Set \(u=\eta_{\alpha,r}(G)\). Gaussian integration by parts and \eqref{eq:prox-eq} give
\[
 D_r(\tau,\alpha)
 =\E[Gu]
 =\E u^2+\alpha\E|u|^r
 =H_r(\tau,\alpha)+\alpha U_r(\tau,\alpha).
\]
For \(r=1\), write \(G-u=\alpha\xi\), where
\(\xi\in\partial|u|\), and use \(\xi u=|u|\).
The inequality is strict because \(\Pp(u\ne0)>0\).
\end{proof}

For a random variable $f$ on the probability space supporting $(B,G)$, write
$\|f\|_s=(\E|f|^s)^{1/s}$, and define on $L_2\cap L_r$ the following population objective:
\begin{equation}\label{eq:popQ}
 \cQ(f):=
 \sqrt{\sigma^2+\frac1\delta\E(f-B)^2}
 -\frac1\delta\E[G(f-B)]+\frac\kappa\delta\|f\|_r.
\end{equation}
Set
\begin{equation}\label{eq:fstar}
 A_\kappa:=\frac B{\tau_\kappa},
 \qquad
 u_\kappa:=\eta_{\alpha_\kappa,r}(A_\kappa+G),
 \qquad
 f_\kappa:=\tau_\kappa u_\kappa.
\end{equation}
The following lemma characterizes the unique minimizer of the population objective.

\begin{lemma}[Population optimizer]\label{prop:population}
Assume that \eqref{eq:se1} holds and that
$
 D_r(\tau_\kappa,\alpha_\kappa)\le\delta$.
Then \(f_\kappa\) is the unique minimizer of \(\cQ\).
\end{lemma}

\begin{proof}
Note that
 \begin{equation}\label{eq:tauscale}
 \sqrt{\sigma^2+\delta^{-1}\E(f_\kappa-B)^2}=\tau_\kappa.
 \end{equation}
  In fact, by \eqref{eq:fstar},
 $f_\kappa-B=\tau_\kappa(u_\kappa-A_\kappa)$, so the square of the LHS is
 $\sigma^2+\delta^{-1}\tau_\kappa^2H_r(\tau_\kappa,\alpha_\kappa)$,
 which equals $\tau_\kappa^2$ by \eqref{eq:se1}.

We now show that, if $1<r<\infty$, then $0\in\partial\cQ(f_\kappa)$. First, note that the $L_r$ norm is differentiable at $f_\kappa\ne0$ and
  \[
  \nabla\|f_\kappa\|_r
  =\|f_\kappa\|_r^{1-r}\psi_r(f_\kappa).
  \]
In fact, for $1<r<\infty$, this is the standard derivative of the $L_r$ norm, and $u_\kappa$ cannot vanish almost surely, because
  $A_\kappa+G$ has a nondegenerate Gaussian conditional law and
  $\eta_{\alpha,r}(z)=0$ only when $z=0$. Next, as $\|f_\kappa\|_r=\tau_\kappa
  U_r(\tau_\kappa,\alpha_\kappa)^{1/r}$ and
  $\psi_r(f_\kappa)=\tau_\kappa^{r-1}\psi_r(u_\kappa)$, the second equation in \eqref{eq:se1} implies that
  \begin{equation}\label{eq:penalty-calibration}
  \kappa\|f_\kappa\|_r^{1-r}\psi_r(f_\kappa)
  =\alpha_\kappa\psi_r(u_\kappa).
  \end{equation}
Combining \eqref{eq:penalty-calibration} with the proximal equation \eqref{eq:prox-eq} and using that $f_\kappa-B=\tau_\kappa(u_\kappa-A_\kappa)$ gives
  \begin{equation}\label{eq:QKKT}
  \frac{f_\kappa-B}{\tau_\kappa}-G
  +\kappa\|f_\kappa\|_r^{1-r}\psi_r(f_\kappa)=0,
  \end{equation}
which is equivalent to $0=\nabla\cQ(f_\kappa)$.

If $r=1$, then $0\in\partial\cQ(f_\kappa)$. To see this, note that the soft-thresholding relation is equivalent to the existence of
  $\xi_\kappa\in\partial|u_\kappa|$ such that
  \[
  A_\kappa+G-u_\kappa=\kappa \xi_\kappa.
  \]
   In fact, $\alpha_\kappa=\kappa$ by the second equation in \eqref{eq:se1}, and this is the proximal
  optimality condition for $\kappa|u|$. Since $\partial|\tau u|=\partial|u|$ for $\tau>0$, the preceding identity
  becomes
  \[
  \frac{f_\kappa-B}{\tau_\kappa}-G+\kappa \xi_\kappa=0,
  \qquad \xi_\kappa\in\partial|f_\kappa|.
  \]
  By \eqref{eq:tauscale}, this is exactly $0\in\partial\cQ(f_\kappa)$.

The subgradient calculations for $r>1$ and $r=1$, together with convexity of $\cQ$,
 imply that $f_\kappa$ is a global minimizer. In fact, the square root in \eqref{eq:popQ} is the Hilbert norm of
 $(\sigma,(f-B)/\sqrt\delta)$; the second term is affine and the last is a norm.
 Hence $\cQ$ is convex, and a zero subgradient is sufficient for global optimality.

If \(\sigma=0\), we first show that \(\Pp(B\ne0)>0\).
Indeed, if \(B=0\) almost surely, then
Lemma~\ref{lem:null-prior-gap}, evaluated at
\((\tau_\kappa,\alpha_\kappa)\), gives
$
 D_r(\tau_\kappa,\alpha_\kappa)
 >
 H_r(\tau_\kappa,\alpha_\kappa)
 =\delta$,
where the equality follows from \eqref{eq:se1}. This contradicts
\(D_r(\tau_\kappa,\alpha_\kappa)\le\delta\).

 It remains to show the uniqueness of the minimizer, and we split this argument into three cases: \emph{(i)} $1<r<\infty$ and $\sigma>0$, \emph{(ii)} $1<r<\infty$ and $\sigma=0$, and \emph{(iii)} $r=1$.

As for first case ($1<r<\infty$ and $\sigma>0$), note that the map
 $f\mapsto\sqrt{\sigma^2+\delta^{-1}\|f-B\|_2^2}$ is strictly convex when
 $\sigma>0$. % Indeed, equality in the Hilbert-space triangle inequality for $(\sigma,(f_i-B)/\sqrt\delta)$ would force the two vectors to be nonnegative multiples of one another; equality of their first coordinates forces that multiple to be one.  Adding a convex norm and an affine functional preserves strict convexity.

As for the second case ($1<r<\infty$ and $\sigma=0$), let $f_0=f_\kappa$ and $f_1$ be any other minimizer. Then, for
  $0<t<1$, equality must hold separately in the convexity inequalities for
  $\|f-B\|_2$ and $\|f\|_r$ along $(1-t)f_0+tf_1$.
 %  The affine term has zero convexity gap.  The other two gaps are nonnegative,
 % and their sum is zero because both endpoints and the interior point attain the
 % minimum.
This implies that there exist $a,b\ge0$ such that
  \begin{equation}\label{eq:equality-rays}
  f_1-B=a(f_0-B),\qquad f_1=bf_0.
  \end{equation}
 In fact, equality in the $L_2$ triangle inequality gives the first relation; since
  $L_r$ is strictly convex for $1<r<\infty$, equality in its triangle inequality
  gives the second. If $a=b$, subtracting the two identities in \eqref{eq:equality-rays} gives
  $(1-a)B=0$. As $B$ is not almost surely zero, we have $a=1$, which implies $f_1=f_0$.
If $a
\neq b$, $f_0=cB$ for a deterministic constant $c$. However, \eqref{eq:QKKT} would express $G$ as a measurable function of $B$, which is impossible. %  This is impossible because, conditionally on $B$, $G$ has the nondegenerate $N(0,1)$ law. This concludes the uniqueness argument in the second case.
  
As for the third case ($r=1$), let us
define the Moreau envelope and scalar profile
  \begin{align*}
  M_1(z;\kappa)&:=\min_u\{\tfrac12(u-z)^2+\kappa|u|\},\\
  \Phi_1(\tau)&:=\frac\tau2+\frac{\sigma^2}{2\tau}
  +\frac\tau\delta\left(\E M_1(B/\tau+G;\kappa)-\frac12\right).
  \end{align*}
Denote $H_\kappa(\tau):=H_1(\tau,\kappa)$. We show that  \begin{equation}\label{eq:Hprime}
  H_\kappa'(\tau)
  =-\frac2{\tau^3}\E\left[B^2
  \mathbf{1}_{\{|B/\tau+G|\le\kappa\}}\right]\le0.
  \end{equation}
Set $A=B/\tau$, $z=A+G$, $u=\ST(z;\kappa)$, and $q=u-A$.
   Away from $|z|=\kappa$, $q$ is independent of $\tau$ on the active set
   $|z|>\kappa$, whereas $q=-B/\tau$ on the inactive set.
Therefore, $\partial_\tau q^2=-2B^2\tau^{-3}
   \mathbf{1}_{\{|z|\le\kappa\}}$ almost surely.
The threshold event has conditional Gaussian probability zero, and the
   derivative is dominated on compact $\tau$ intervals by $C B^2$. Hence,  dominated
   convergence proves \eqref{eq:Hprime}.

Next, we show that the derivative of the scalar profile is
  \begin{equation}\label{eq:Phiprime}
  \Phi_1'(\tau)=\frac1{2\tau^2}
  \left(\tau^2\left(1-\frac{H_\kappa(\tau)}\delta\right)-\sigma^2\right).
  \end{equation}
Note that the envelope theorem gives $\partial_zM_1(z;\kappa)=z-u$.
With $e=z-u=G-q$, the proximal condition gives $eu=\kappa|u|$.
   Hence,
   \[
   M_1(z;\kappa)-Ae
   =\tfrac12e^2+eu-Ae
   =\tfrac12e^2+eq
   =\tfrac12G^2-\tfrac12q^2.
   \]
Differentiating $\tau\{\E M_1(B/\tau+G;\kappa)-1/2\}$, we obtain the following derivative:
   \[
   \E[M_1-1/2-A\partial_zM_1]
   =-\tfrac12H_\kappa(\tau),
   \]
   because $\E G^2=1$. Adding the derivatives of
   $\tau/2+\sigma^2/(2\tau)$ proves \eqref{eq:Phiprime}.

At this point, we show that the function $\Phi_1$ has the unique minimizer $\tau_\kappa$.
Let $K(\tau):=\tau^2(1-H_\kappa(\tau)/\delta)$.  By
   \eqref{eq:se1}, $K(\tau_\kappa)=\sigma^2$.
If $\sigma>0$, then $H_\kappa(\tau_\kappa)<\delta$.  If
   $0<s<t$ and $H_\kappa(s)<\delta$, then
   \[
   K(t)-K(s)=(t^2-s^2)\left(1-\frac{H_\kappa(s)}\delta\right)
   +\frac{t^2}{\delta}\{H_\kappa(s)-H_\kappa(t)\}>0.
   \]
We now compare \(K(\tau)\) with
\(K(\tau_\kappa)=\sigma^2\).  If \(\tau<\tau_\kappa\) and
\(H_\kappa(\tau)\geq\delta\), then
$
 K(\tau)\leq0<\sigma^2$.
If instead \(H_\kappa(\tau)<\delta\), applying the preceding identity
with \(s=\tau\) and \(t=\tau_\kappa\) gives
$
 K(\tau)<K(\tau_\kappa)=\sigma^2$. 
Finally, if \(\tau>\tau_\kappa\), then
\(H_\kappa(\tau_\kappa)<\delta\), and applying the same identity with
\(s=\tau_\kappa\) and \(t=\tau\) gives
$
 K(\tau)>K(\tau_\kappa)=\sigma^2$.
Therefore, \eqref{eq:Phiprime} shows that \(\Phi_1\) is strictly
decreasing on \((0,\tau_\kappa)\) and strictly increasing on
\((\tau_\kappa,\infty)\).

If \(\sigma=0\), then \(H_\kappa(\tau_\kappa)=\delta\).
Since \(\Pp(B\ne0)>0\), the expectation in \eqref{eq:Hprime}
is strictly positive for every \(\tau>0\), so \(H_\kappa\) is
strictly decreasing. Hence $K(\tau)<0$ before $\tau_\kappa$ and
   $K(\tau)>0$ after it. As a result, \eqref{eq:Phiprime} gives that $\Phi_1$ decreases before
   $\tau_\kappa$ and increases after it, which %.  It also gives strict separation at
  % $0$ and $+\infty$: fix $0<\tau_0<\tau_\kappa<\tau_1$ and compare all
%   $\tau\le\tau_0$ with $\tau_0$, and all $\tau\ge\tau_1$ with $\tau_1$. This 
implies the desired claim on the uniqueness of the minimizer of $\Phi_1$.

As  $\sqrt{x}=\inf_{\tau>0}\{\tau/2+x/(2\tau)\}$, we have 
%To conclude the uniqueness of $f_\kappa$, it then suffices to show the variational identity
\begin{equation}\label{eq:Qprofile}
  \cQ(f)=\inf_{\tau>0}\left[
  \frac\tau2+\frac{\sigma^2}{2\tau}
  +\frac1\delta\E\left[
  \frac{(f-B)^2}{2\tau}-G(f-B)+\kappa|f|\right]\right]
  \end{equation}
%  implies uniqueness of $f_\kappa$.
    For fixed $\tau$, the term inside the expectation is strictly convex in $f$ and its unique pointwise minimizer is
  $f=\tau\ST(B/\tau+G;\kappa)$.  Minimizing first in $f$ produces
  $\Phi_1(\tau)$. 
  If \(\sigma>0\), or if \(\|f-B\|_2>0\), the infimum over
\(\tau\) in \eqref{eq:Qprofile} is attained at a positive scale.
The strict minimality of \(\tau_\kappa\) for \(\Phi_1\), together
with pointwise strict convexity, then shows that
\(\cQ(f)=\cQ(f_\kappa)\) forces \(f=f_\kappa\).

It remains, when \(\sigma=0\), to exclude \(f=B\), whose optimal
scale is the boundary \(\tau\downarrow0\). Fix
\(0<\tau_0<\tau_\kappa\). At \(f=B\), the expression in brackets
in \eqref{eq:Qprofile} is $\frac{\tau}{2}+\frac{\kappa}{\delta}\E|B|$, which is bounded below by \(\Phi_1(\tau)\). Hence, for
\(0<\tau\le\tau_0\), it is at least
\(\Phi_1(\tau_0)>\Phi_1(\tau_\kappa)\). Letting
\(\tau\downarrow0\) gives
\[
 \cQ(B)
 \ge \Phi_1(\tau_0)
 >\Phi_1(\tau_\kappa)
 =\cQ(f_\kappa).
\]
Thus, \(f_\kappa\) is the unique minimizer of \(\cQ\) and the proof is complete.
%\par\medskip\noindent
%The preceding arguments prove uniqueness for $r>1$ and for $r=1$, completing the proposition.
\end{proof}

\subsection{Joint empirical laws and bounds on Gaussian inner products}
\label{app:empirical-margin}

\begin{lemma}[Convergence of joint empirical distribution]
\label{lem:empirical}
Let $h=(h_1,\ldots,h_p)$ have i.i.d.\ $N(0,1)$ entries and be independent of
$\beta_*$.  Almost surely,
\begin{equation}\label{eq:jointWq}
 \frac1p\sum_{j=1}^p\delta_{((\beta_*)_j,h_j)}
 \xrightarrow[W_q]{}\cL(B,G).
\end{equation}
Consequently, for fixed $\tau,\alpha>0$, if
$
 b_j(\tau,\alpha):=\tau\eta_{\alpha,r}((\beta_*)_j/\tau+h_j),
$
then
\begin{align}
 \|b(\tau,\alpha)-\beta_*\|_{2,p}^2
 &\to\tau^2H_r(\tau,\alpha),\label{eq:empH}\\
 \|b(\tau,\alpha)\|_{r,p}^r
 &\to\tau^rU_r(\tau,\alpha),\label{eq:empU}\\
 \langle h,b(\tau,\alpha)-\beta_*\rangle_p
 &\to\tau D_r(\tau,\alpha).
 \label{eq:empD}
\end{align}
For $r=1$, \eqref{eq:empU} corresponds to convergence of the first absolute moment, and 
the empirical law of $((\beta_*)_j,b_j(\tau,\alpha))$ converges in $W_2$ to the
corresponding population law.
\end{lemma}

\begin{proof}
We start by showing that, for every bounded Lipschitz $f:\R^2\to\R$,
 \begin{equation}\label{eq:BLconv}
 \frac1p\sum_{j=1}^p f((\beta_*)_j,h_j)\to \E f(B,G)
 \qquad\text{almost surely}.
 \end{equation}
% which implies that  the empirical measures in \eqref{eq:jointWq} converge weakly to $\cL(B,G)$.
%To prove \eqref{eq:BLconv}, l
Let $\bar f(x):=\E f(x,G)$. Then,
  $p^{-1}\sum_j\bar f((\beta_*)_j)\to\E\bar f(B)$, as
$\bar f$ is bounded and continuous, and $W_q$ convergence of
  the empirical law of $\beta_*$ implies weak convergence.
Consider the centered average
  \[
  Z_p:=\frac1p\sum_{j=1}^p\{f((\beta_*)_j,h_j)-\bar f((\beta_*)_j)\}.
  \]
   Conditional on $\beta_*$, the summands are independent,
  centered, and uniformly bounded.
  Let $M>0$ satisfy $\|f\|_\infty\le M$. Hoeffding's inequality
\cite[Theorem~2.2.6]{vershynin2018highdimensional},
applied to both tails, gives, for every $t>0$,
\[
\Pp(|Z_p|>t)
\le 2\exp\left(-\frac{pt^2}{2M^2}\right).
\]
Thus, an application of Borel--Cantelli gives that $Z_p$  converges almost surely to zero, which gives  \eqref{eq:BLconv}.

%  Apply \eqref{eq:BLconv} to a countable convergence-determining class
% of bounded Lipschitz functions and intersect the resulting probability-one events.

%To obtain the $W_q$ convergence claimed in \eqref{eq:jointWq}, it remains to prove that the $q$-th moments of the empirical measures  are uniformly integrable and converge to the $q$-th
% moment of $(B,G)$.
The deterministic measures $p^{-1}\sum_j\delta_{(\beta_*)_j}$ are uniformly
  $q$-integrable, due to their assumed $W_q$ convergence. Applying first the strong law of large numbers to the truncated functions
  $|h|^q\wedge K$ and then letting $K\to\infty$ gives that, almost surely,
  \[
  \lim_{L\to\infty}\limsup_{p\to\infty}
  \frac1p\sum_{j=1}^p|h_j|^q\mathbf{1}_{\{|h_j|>L\}}=0.
  \]
Since $\|(x,g)\|^q\le C_q(|x|^q+|g|^q)$, we obtain that the $q$-th moments of the empirical measures  are uniformly integrable and converge to the $q$-th
 moment of $(B,G)$. Thus, a union bound over a countable class of bounded Lipschitz function estalishes the weak convergence in  \eqref{eq:jointWq}.

%Choose a countable class \(\mathcal F_0\) of bounded Lipschitz functions on \(\mathbb R^2\) sufficient to establish weak convergence. Applying the preceding argument to every \(f\in\mathcal F_0\) and intersecting the resulting probability-one events, we obtain, with probability one,
%\[
% \frac1p\sum_{j=1}^p f(\beta_j,h_j)
% \to \E f(B,G)
% \qquad\text{for every }f\in\mathcal F_0.
%\]
%It follows that the empirical measures in \eqref{eq:jointWq}
%converge weakly to \(\mathcal L(B,G)\).

Consider the continuous map
 \[
 (x,g)\mapsto\bigl(x,\tau\eta_{\alpha,r}(x/\tau+g)\bigr).
 \]
Lemma~\ref{lem:prox} gives linear growth of this map. Furthermore, the squared-error
 function has quadratic growth, and the $r$th-power function has growth at most
 $C(1+|x|^q+|g|^q)$ because $q=2\vee r$. Hence, \eqref{eq:jointWq} implies \eqref{eq:empH} and \eqref{eq:empU}, and the asserted $W_2$
 convergence for $r=1$. 

It remains to show \eqref{eq:empD}. To that aim, consider the function
  \[
  (x,g)\mapsto g(\eta_{\alpha,r}(x/\tau+g)-x/\tau).
  \]
This function  is continuous except on a null set when $r=1$, and has at most quadratic growth due to Lemma~\ref{lem:prox}. % bounds its absolute value by
%  $C(|g|^2+|x|^2)$; the only discontinuities of the derivative representation for
%  $r=1$ occur on the two threshold curves, which have zero conditional Gaussian
%  probability.
%\par\medskip\noindent
Hence, using \eqref{eq:jointWq}, we have that its empirical average converges to
  $\E[G\{\eta_{\alpha,r}(B/\tau+G)-B/\tau\}]$.
%   Use truncation together with  and $q\ge2$.
   Conditional on $B$, the map of $G$ is Lipschitz and absolutely continuous. Thus,
  Gaussian integration by parts gives
  \[
  \E_G[G\eta_{\alpha,r}(B/\tau+G)]
  =\E_G[\eta'_{\alpha,r}(B/\tau+G)],
  \]
  while $\E_G[G B/\tau]=0$. Thus, integrating over $B$ and using dominated convergence gives \eqref{eq:empD} and concludes the proof. 
\end{proof}

Define
\begin{equation}\label{eq:gamma-pop}
 \gamma_r(\kappa):=
 \sup_{v\in L_2\cap L_r,\ \|v\|_2\le 1}
 \{\E[Gv]-\kappa\|v\|_r\}.
\end{equation}
Here, $L_s=L_s(P_B \otimes N(0, 1))$, we identify with $B$ and $G$ the coordinate random variables on this product space and for $v\in L_2\cap L_r$, $\E[Gv]$ denotes the $L_2$ inner product between $G$ and $v=v(B, G)$.

\begin{lemma}[Upper bound on $\gamma_r(\kappa)$]\label{lem:margin}
Assume that \eqref{eq:se1} holds and that
\(D_r(\tau_\kappa,\alpha_\kappa)\le\delta\). Then
\begin{equation}\label{eq:gamma-strict}
 \gamma_r(\kappa)<\sqrt\delta.
\end{equation}
\end{lemma}

\begin{proof}
We start by showing that
the quantity in \eqref{eq:gamma-pop} has the dual representation
 \begin{equation}\label{eq:gamma-distance}
 \gamma_r(\kappa)
 ={\rm dist}_{L_2}\bigl(G,K_\kappa\bigr),
 \qquad
 K_\kappa:=\{s\in L_2:\|s\|_{r^*}\le\kappa\},
 \end{equation}
 where $r^*$ denotes the conjugate exponent of $r$, i.e., $1/r+1/r^*=1$, with $r^*=\infty$ when $r=1$.
Note that the set $K_\kappa$ is convex in $L_2$. We now prove that $K_\kappa$ is closed with respect to the $L_2$ topology. This requires showing that, if $s_n\in K_\kappa$  and $\|s_n-s\|_2\to 0$, then $\|s\|_{r^*}\le \kappa$. If $r^*\le 2$, then 
$$
|\|s_n\|_{r^*}-\|s\|_{r^*}|\le \|s_n-s\|_{r^*}\le\|s_n-s\|_{2}\to 0.
$$
If $2<r^*<\infty$, then $L_2$ convergence does not imply $L_{r^*}$ convergence. However, $L_2$ convergence implies convergence in probability and, hence, there is a subsequence $(s_{n_k})$ such that $s_{n_k}\to s$ almost surely. Fatou's lemma then gives
$$
\mathbb E|s|^{r^*}\le \liminf_{k\to\infty} \mathbb E|s_{n_k}|^{r^*}\le \kappa^{r^*}.
$$
Finally, if $r^*=\infty$, then $\|s_n\|_\infty\le \kappa$ implies that $|s_n|\le \kappa$ almost surely. Thus, taking the almost-surely convergent subsequence above, we obtain pointwise $|s|=\lim_{k\to\infty} |s_{n_k}|\le \kappa$ almost surely, which gives that $\|s\|_\infty\le \kappa$. Since $K_\kappa$ is closed and convex, $G$ has a unique projection onto it, which we denote by $s_0$. Moreover, \(s_0\) is measurable with respect to \(G\).  Indeed, if
\(s\in K_\kappa\), then \(\bar s:=\E[s\mid G]\) is feasible by
conditional Jensen, and
\[
 \E(G-\bar s)^2
 =
 \E(G-s)^2-\E(s-\bar s)^2.
\]
Uniqueness of the projection therefore gives
\(s_0=\E[s_0\mid G]\) almost surely.  By symmetry and uniqueness,
\(s_0\) is an odd function of \(G\).

For every $s\in K_\kappa$ and every $v\in L_2\cap L_r$,
  \[
  \E[Gv]-\kappa\|v\|_r
  \le\E[(G-s)v]\le\|G-s\|_2\|v\|_2,
  \]
where the first inequality follows from H\"older's inequality and the fact that $s\in K_\kappa$, and the second one follows from Cauchy--Schwarz. Taking the
  supremum in $v$ and infimum in $s$ proves one side of the equivalence in
  \eqref{eq:gamma-distance}.

  Set \(q_0:=G-s_0\).  As $s_0$ is the projection of $G$, $\E[(G-s_0)(s-s_0)]\le0$ for all $s\in K_\kappa$.
Consequently,
\begin{equation}
 \sup_{s\in K_\kappa}\E[q_0s]
 =\E[q_0s_0]<\infty.
 \label{eq:projection-support}
\end{equation}
The support function of \(K_\kappa\) is
\begin{equation}
 \sup_{s\in K_\kappa}\E[qs]
 =
 \begin{cases}
  \kappa\|q\|_r,&q\in L_r,\\
  +\infty,&q\notin L_r.
 \end{cases}
 \label{eq:K-support-function}
\end{equation}
Indeed, the upper bound in the first case follows from H\"older's
inequality.  The reverse bound follows by truncating the formal
optimizer
$
 s^*
 =
 \kappa\,
 \frac{\sign(q)|q|^{r-1}}{\|q\|_r^{r-1}}
$
when \(1<r<\infty\).  For \(r=1\), the optimizer is
\(s^*=\kappa\sign(q)\).  The second case in
\eqref{eq:K-support-function} follows by applying the same truncation
argument.

Since the LHS of \eqref{eq:projection-support} is finite,
\eqref{eq:K-support-function} implies that \(q_0\in L_r\).  Moreover,
$
 \E[q_0s_0]=\kappa\|q_0\|_r$.
Therefore,
\[
 \E[Gq_0]-\kappa\|q_0\|_r
 =
 \E[(G-s_0)q_0]
 =
 \|q_0\|_2^2.
\]
If \(q_0\neq0\), taking \(v=q_0/\|q_0\|_2\) in
\eqref{eq:gamma-pop} gives
\[
 \gamma_r(\kappa)\geq\|q_0\|_2
 =\operatorname{dist}_{L_2}(G,K_\kappa).
\]
Together with the opposite inequality proved above, this establishes
\eqref{eq:gamma-distance}.  If \(q_0=0\), the conclusion is immediate.

Recall the definitions in \eqref{eq:fstar} and consider the random variable
 \begin{equation}\label{eq:dual-candidate}
 s_\kappa:=A_\kappa+G-u_\kappa=A_\kappa+G-\eta_{\alpha_\kappa, r}(A_\kappa+G).
 \end{equation}
If $r=1$, then $s_\kappa=A_\kappa+G-\ST(A_\kappa+G;\kappa)$, which implies that $\|s_\kappa\|_\infty\le\kappa$. If $1<r<\infty$, then
  $s_\kappa=\alpha_\kappa\psi_r(u_\kappa)$ by
  \eqref{eq:prox-eq}, and
  $
  \|s_\kappa\|_{r^*}
  =\alpha_\kappa(\E|u_\kappa|^r)^{(r-1)/r}
  =\kappa
  $
  by the second equation in \eqref{eq:se1}. Therefore, for all $r\ge 1$, $s_\kappa\in K_\kappa$ and we have
 \begin{equation}\label{eq:gamma-H}
 \gamma_r(\kappa)^2
 \le \|G-s_\kappa\|_2^2
 =\|u_\kappa-A_\kappa\|_2^2
 =H_r(\tau_\kappa,\alpha_\kappa)
 \le\delta,
 \end{equation}
  where the first inequality follows from \eqref{eq:gamma-distance}, the next two equalities respectively use
 \eqref{eq:dual-candidate} and \eqref{eq:Hdef}, and  the last inequality follows from
 \eqref{eq:se1}. %
% $H_r=\delta(1-\sigma^2/\tau_\kappa^2)$.
If $\sigma>0$, then the last inequality in \eqref{eq:gamma-H} is strict and \eqref{eq:gamma-strict} follows immediately.

It remains to show that the inequality in \eqref{eq:gamma-H} is strict for $\sigma=0$. We will do that by contradiction and suppose that equality in \eqref{eq:gamma-H} holds. Equality in \eqref{eq:gamma-H} would imply that \(s_\kappa\) is a
metric projection of \(G\) onto \(K_\kappa\).  By uniqueness of the
projection,
$
 s_\kappa=s_0$.
By the observation following the definition of \(s_0\), it follows
that \(s_\kappa\) is an odd measurable function of \(G\).

  Consider $1<r<\infty$. From $s_\kappa=\alpha_\kappa\psi_r(u_\kappa)$, the invertibility of
  $\psi_r$ makes $u_\kappa$ an odd measurable function of $G$. Furthermore,
  $u_\kappa-A_\kappa=G-s_\kappa$ is odd and measurable in $G$.  Hence
  $A_\kappa$ is both measurable in $G$ and independent of $G$, so it is constant; being odd, we have that $A_\kappa=0$ almost surely.
   For $r=1$, we also have that  $A_\kappa=0$ almost surely. In fact, 
   the metric projection onto the $L_\infty$ ball is uniquely
  ${\rm clip}(G;[-\kappa,\kappa])$, where ${\rm clip}(z; [a, b]):=\min\{b, \max\{a, z\}\}$.  Hence, 
  $
  \ST(A_\kappa+G;\kappa)-A_\kappa=\ST(G;\kappa).
  $
  Conditional on $A_\kappa=a$, this equality holds for almost every $G$.  On the
  interval $|G|<\kappa$ it can hold only for $a=0$.  Independence therefore gives
  $A_\kappa=0$ almost surely.

As \(A_\kappa=B/\tau_\kappa\), equality in \eqref{eq:gamma-H}
would therefore imply \(B=0\) almost surely. Since equality also
forces \(\sigma=0\), \eqref{eq:se1} gives
\(H_r(\tau_\kappa,\alpha_\kappa)=\delta\). On the other hand,
Lemma~\ref{lem:null-prior-gap} would give
$
 D_r(\tau_\kappa,\alpha_\kappa)
 >
 H_r(\tau_\kappa,\alpha_\kappa)
 =\delta,
$
contradicting the hypothesis. Hence \eqref{eq:gamma-H} is strict, and the proof is complete.
\end{proof}

%Let \(r^*\) denote the conjugate exponent of \(r\), with
%\(r^*=\infty\) when \(r=1\); we use
%\(\|s\|_{\infty,p}:=\max_{1\le j\le p}|s_j|\). Define

\begin{lemma}[Convergence to $\gamma_r(\kappa)$]\label{lem:empirical-dual-margin}
Let
\begin{equation}\label{eq:gamma-emp}
 \gamma_{r,p}(\kappa):=
 \sup_{\|v\|_{2,p}\le 1}
 \{\langle h,v\rangle_p-\kappa\|v\|_{r,p}\}.
\end{equation}
Then, for every fixed \(\kappa>0\),
\[
 \gamma_{r,p}(\kappa)\to\gamma_r(\kappa)
\]
almost surely and in \(L_1\).
\end{lemma}

\begin{proof}
Repeating the same argument of Lemma~\ref{lem:margin} and replacing expectations with
empirical averages, we have
\[
 \gamma_{r,p}(\kappa)
 =
 \min_{\|s\|_{r^*,p}\le\kappa}\|h-s\|_{2,p},
\]
with $r^*$ the conjugate exponent of $r$.
Let \(s_0(G)\) be the population metric projection in
\eqref{eq:gamma-distance}. Evaluating it at \(h_1,\ldots,h_p\) and,
if necessary, rescaling by $
 \frac{\kappa}{\|s_0(h)\|_{r^*,p}\vee\kappa}$
gives
\[
 \limsup_{p\to\infty}\gamma_{r,p}(\kappa)
 \le\gamma_r(\kappa),
\]
  from 
the strong law of large numbers applied to \(|h-s_0(h)|^2\) and, when
\(r^*<\infty\), to \(|s_0(h)|^{r^*}\). 

Conversely, let \(s_p\) be an empirical projection. Then
\(\|h-s_p\|_{2,p}\le\|h\|_{2,p}\), and hence
\(\|s_p\|_{2,p}\le2\|h\|_{2,p}\). Along every subsequence, the
empirical laws of \((h_j,s_{p,j})\) have a further weakly convergent
subsequence, say to the law of \((G,S)\), with
\[
 \|S\|_{r^*}\le\kappa,
 \qquad
 \|G-S\|_2
 \le
 \liminf_{p\to\infty}\|h-s_p\|_{2,p}.
\]
Here tightness follows from the \(L_2\) bound; lower semicontinuity
gives the constraint and the distance inequality, with truncation
when \(r^*>2\).
The limit variable \(S\) need not be measurable with respect to \(G\),
so put \(\overline S:=\E[S\mid G]\). Conditional Jensen gives
\[
 \|\overline S\|_{r^*}\le\|S\|_{r^*}\le\kappa,
 \qquad
 \|G-\overline S\|_2\le\|G-S\|_2.
\]
Thus \(\overline S\), viewed as a function of the \(G\)-coordinate on
the product space supporting \((B,G)\), is feasible in
\eqref{eq:gamma-distance}. Consequently,
\[
 \gamma_r(\kappa)
 \le\|G-\overline S\|_2
 \le\|G-S\|_2
 \le\liminf_{p\to\infty}\|h-s_p\|_{2,p}.
\]
This proves almost-sure convergence. Finally,
\[
 0\le\gamma_{r,p}(\kappa)\le\|h\|_{2,p},
\]
and the variables on the right have uniformly bounded second
moments. Thus the convergence also holds in \(L_1\).
\end{proof}

\subsection{Stability of the empirical objective}
\label{app:empirical-stability}

Define the empirical active functional
\begin{equation}\label{eq:empQ}
 Q_p(b):=
 \sqrt{\sigma^2+\frac1{\delta_p}\|b-\beta_*\|_{2,p}^2}
 -\frac1{\delta_p}\langle h,b-\beta_*\rangle_p
 +\frac\kappa{\delta_p}\|b\|_{r,p}.
\end{equation}
Set
\[
M_\kappa^{\mathrm{err}}
:=
\E(f_\kappa-B)^2
=
\tau_\kappa^2H_r(\tau_\kappa,\alpha_\kappa),
\qquad M_\kappa^{\mathrm{coef}}
:=
\E|f_\kappa|^r
=
\tau_\kappa^rU_r(\tau_\kappa,\alpha_\kappa),
\qquad
\mu_\kappa:=\mathcal L(B,f_\kappa).
\]
%where, for $r=1$, the definition of $M_\kappa^{\mathrm{coef}}$ is replaced by
%$$
%M_\kappa^{\mathrm{coef}}
%:=\tau_\kappa\mathbb %E|\ST(B/\tau_\kappa+G;\kappa)|.
%$$
For $b\in\mathbb R^p$, define
\[
\mu_{p,b}
:=
\frac1p\sum_{j=1}^p\delta_{((\beta_*)_j,b_j)}
\]
and
\begin{equation}\label{eq:dedb}
d_p(b)
:=
\max\left\{
 \left|\|b-\beta_*\|_{2,p}^2-M_\kappa^{\mathrm{err}}\right|,
 \left|\|b\|_{r,p}^r-M_\kappa^{\mathrm{coef}}\right|,
 W_2(\mu_{p,b},\mu_\kappa)
\right\}.
\end{equation}
For $\epsilon>0$, define
\begin{equation}\label{eq:defBp}
\mathcal B_{p,\epsilon}
:=
\{b\in\mathbb R^p:d_p(b)\ge\epsilon\}.
\end{equation}

\begin{lemma}[Stability of empirical objective]\label{lem:stability}
Assume that \eqref{eq:se1} holds and that
\(D_r(\tau_\kappa,\alpha_\kappa)\le\delta\). Then, 
almost surely, the following statements hold.
\begin{enumerate}
 \item[(i)] There are deterministic constants $c,C>0$ such that, for every sufficiently
 large $p$ and every $b\in\R^p$,
 \begin{equation}\label{eq:Qcoercive}
 Q_p(b)\ge c\|b\|_{2,p}-C.
 \end{equation}
\item[(ii)] We have
\begin{equation}
 \inf_{b\in\mathbb R^p}Q_p(b)
 \to
 \cQ(f_\kappa).
 \label{eq:minQ-convergence}
\end{equation}
 \item[(iii)] If $Q_p(b_p)\le\min_bQ_p(b)+o(1)$, then
 \begin{align}
 \|b_p-\beta_*\|_{2,p}^2
 &\to\tau_\kappa^2H_r(\tau_\kappa,\alpha_\kappa),
 \label{eq:stabmse}\\
 \|b_p\|_{r,p}^r
 &\to\tau_\kappa^rU_r(\tau_\kappa,\alpha_\kappa),
 \label{eq:stabmoment}\\
 \frac1p\sum_{j=1}^p\delta_{((\beta_*)_j,b_{p,j})}
 &\xrightarrow[W_2]{}\mu_\kappa.
 \label{eq:stablaw}
 \end{align}
% where for $r=1$, \eqref{eq:stabmoment} is replaced by
% \begin{equation}
%     \|b_p\|_{1, p}\to \tau_\kappa \mathbb E|\ST(B/\tau_\kappa+G; \kappa)|.
% \end{equation}
 \item[(iv)] For every $\epsilon>0$,
\begin{equation}\label{eq:Q-gap-formal}
\liminf_{p\to\infty}
\left(
 \inf_{b\in\mathcal B_{p,\epsilon}}Q_p(b)
 -
 \inf_{b\in\mathbb R^p}Q_p(b)
\right)>0.
\end{equation}
\end{enumerate}
\end{lemma}

\begin{proof}
By Lemmas~\ref{lem:margin} and \ref{lem:empirical-dual-margin}, and
since \(\delta_p\to\delta\), there is a deterministic \(c_1>0\)
such that, eventually,
 \begin{equation}\label{eq:emp-margin}
 \sqrt{\delta_p}\|v\|_{2,p}-\langle h,v\rangle_p
 +\kappa\|v\|_{r,p}\ge c_1\|v\|_{2,p},
 \end{equation}
 for $v\in\R^p$.
An application of the triangle inequality gives
  \[
  \sqrt{\sigma^2+\delta_p^{-1}\|b-\beta_*\|_{2,p}^2}
  \ge\delta_p^{-1/2}\|b\|_{2,p}
  -\delta_p^{-1/2}\|\beta_*\|_{2,p}.
  \]
As $\|\beta_*\|_{2,p}=O(1)$ and $|\langle h,\beta_*\rangle_p|=O(1)$ almost surely, applying \eqref{eq:emp-margin} with $v=b$ gives that \eqref{eq:Qcoercive} holds.

The bound in \eqref{eq:Qcoercive} readily implies that every sequence with $Q_p(b_p)=O(1)$ has bounded
 $\|b_p\|_{2,p}$. Furthermore, rearranging \eqref{eq:empQ} gives
  \[
  \frac\kappa{\delta_p}\|b_p\|_{r,p}
  \le Q_p(b_p)+\frac1{\delta_p}
  |\langle h,b_p-\beta_*\rangle_p|,
  \]
 which implies that $\|b_p\|_{r,p}$ is also bounded whenever $Q_p(b_p)$ is bounded. 

Let $
 b^\circ_{p,j}:=\tau_\kappa
 \eta_{\alpha_\kappa,r}((\beta_*)_j/\tau_\kappa+h_j)$.
Then, by Lemma~\ref{lem:empirical}, we have
 $Q_p(b_p^\circ)\to\cQ(f_\kappa)$. This implies that 
 \begin{equation}\label{eq:minQlimit}
 \limsup_p\min_bQ_p(b)\le\cQ(f_\kappa).
 \end{equation}
Let \(b_p\) be any sequence satisfying
\[
 Q_p(b_p)\leq\inf_b Q_p(b)+o(1).
\]
Then, the sequences
\(\|b_p-\beta_*\|_{2,p}\) and \(\|b_p\|_{r,p}\) are bounded.
Fix an arbitrary subsequence.  Passing to a further subsequence, we may
assume that
\[
 a_p:=\|b_p-\beta_*\|_{2,p}^2\to a,
 \qquad
 c_p:=\|b_p\|_{r,p}^r\to c,
\]
and that the empirical laws
\[
 \nu_p:=\frac1p\sum_{j=1}^p
 \delta_{((\beta_*)_j,h_j,b_{p,j})}
\]
converge weakly to the law of a triple \((B,G,F)\).  Such a
subsequence exists by the previous moment bounds and Prokhorov's
theorem.

We first verify convergence of the Gaussian cross term.  Set
\(e_{p,j}:=b_{p,j}-(\beta_*)_j\).  For \(K,T>0\),
\[
\begin{aligned}
 &\frac1p\sum_{j=1}^p|h_je_{p,j}|
   \mathbf 1_{\{|h_je_{p,j}|>T\}}\leq
 \left(
  \frac1p\sum_{j=1}^p
  h_j^2\mathbf 1_{\{|h_j|>K\}}
 \right)^{1/2}
 \|e_p\|_{2,p}
 +\frac{K^2}{T}\|e_p\|_{2,p}^2.
\end{aligned}
\]
Indeed, the first term controls the contribution from
\(\{|h_j|>K\}\); on its complement,
\(|h_je_{p,j}|>T\) implies \(|e_{p,j}|>T/K\).
%The norms \(\|e_p\|_{2,p}\) are uniformly bounded, and the empirical Gaussian squares are uniformly integrable.  
Letting first
\(T\to\infty\) and then \(K\to\infty\) proves uniform
integrability of the products \(h_je_{p,j}\).  Consequently,
\begin{equation}
 \langle h,b_p-\beta_*\rangle_p
 \to
 \E[G(F-B)].
 \label{eq:stability-cross-term}
\end{equation}
It follows that
\begin{equation}
 \lim_{p\to\infty}Q_p(b_p)
 =
 \sqrt{\sigma^2+\frac a\delta}
 -\frac1\delta\E[G(F-B)]
 +\frac{\kappa}{\delta}c^{1/r}.
 \label{eq:stability-objective-limit}
\end{equation}
Let
\[
 \overline F:=\E[F\mid B,G].
\]
Lower semicontinuity and conditional Jensen's inequality give
\[
 a\geq\E(F-B)^2\geq\E(\overline F-B)^2,
 \qquad
 c^{1/r}\geq\|F\|_r\geq\|\overline F\|_r,
\]
whereas
\[
 \E[G(F-B)]=\E[G(\overline F-B)].
\]
Hence the RHS of
\eqref{eq:stability-objective-limit} is at least
\(\cQ(\overline F)\), which in turn is at least
\(\cQ(f_\kappa)\) by Lemma~\ref{prop:population}.
On the other hand, \eqref{eq:minQlimit} and the near-minimality of
\(b_p\) give
\[
 \limsup_{p\to\infty}Q_p(b_p)\leq\cQ(f_\kappa).
\]
Therefore equality holds throughout:
\[
 \lim_{p\to\infty}Q_p(b_p)
 =\cQ(\overline F)
 =\cQ(f_\kappa).
\]
Taking \(b_p\) to be an exact minimizer proves
\eqref{eq:minQ-convergence}.

The uniqueness of the population minimizer gives that $\overline F=f_\kappa$ almost surely. Moreover,
$
 \E(F-B)^2
 =
 \E(f_\kappa-B)^2+\E(F-f_\kappa)^2$.
If \(\E(F-f_\kappa)^2>0\), the square-root term in
\eqref{eq:stability-objective-limit} would be strictly larger than
the corresponding term in \(\cQ(f_\kappa)\), while the norm term
cannot be smaller.  This contradicts equality above.  Thus,
$
 F=f_\kappa$ almost surely.

It remains to exclude escape of moment mass.  Lower semicontinuity
now gives
 $a\geq M_\kappa^{\mathrm{err}}$,
 $c\geq M_\kappa^{\mathrm{coef}}$.
Using \(F=f_\kappa\) in
\eqref{eq:stability-objective-limit} and comparing with
\(\cQ(f_\kappa)\), we obtain
\[
\begin{aligned}
0=
 \sqrt{\sigma^2+\frac a\delta}
 -
 \sqrt{\sigma^2+
  \frac{M_\kappa^{\mathrm{err}}}{\delta}}+
 \frac{\kappa}{\delta}
 \left(
  c^{1/r}
  -(M_\kappa^{\mathrm{coef}})^{1/r}
 \right).
\end{aligned}
\]
Both terms on the right are nonnegative.  Since \(\kappa>0\), both
must vanish, and hence
$
 a=M_\kappa^{\mathrm{err}}, c=M_\kappa^{\mathrm{coef}}$.
Because the original subsequence was arbitrary, this proves
\eqref{eq:stabmse} and \eqref{eq:stabmoment} for the full sequence.

Finally, the empirical laws of
\(((\beta_*)_j,b_{p,j}-(\beta_*)_j)\) converge weakly to
\(\mathcal L(B,f_\kappa-B)\).  Their second moments also converge:
the first-coordinate second moments converge by the
\(W_q\) assumption, and the second-coordinate second moments converge
by \eqref{eq:stabmse}.  Thus these empirical laws converge in \(W_2\).
Applying the continuous linear map $
 (x,e)\longmapsto(x,x+e)$ 
gives \eqref{eq:stablaw}.

 It remains to show statement \emph{(iv)}. By contradiction,
suppose it fails for some $\epsilon>0$. Then, there is a subsequence $p_k$ and $b_{p_k}\in\mathcal B_{p_k, \epsilon}$ such that $Q_{p_k}(b_{p_k})-\inf_b Q_{p_k}(b)\to 0$. For indices outside the subsequence, choose an exact minimizer of $Q_p$. Then, $Q_p(b_p)\le \min_b Q_p(b)+o(1)$ and by statement \emph{(iii)} we have that $d_p(b_p)\to 0$, which is a contradiction. 
  \end{proof}

  \subsection{Analysis of the auxiliary optimization}
\label{app:auxiliary-optimization}

We next consider the auxiliary optimization problem and relate it to
the original one via CGMT. Let \(X=G_X/\sqrt n\) and \(w=b-\beta_*\).
Let \(g\in\R^n\) and \(h\in\R^p\) be independent standard Gaussian
vectors, independent of \((G_X,\varepsilon)\). Define
\begin{align}
 C_p(b)
 &:=
 \sqrt{\sigma^2+\delta_p^{-1}\|b-\beta_*\|_{2,p}^2}
 -\delta_p^{-1}\langle h,b-\beta_*\rangle_p,
 \label{eq:int-det-constraint}\\
 C_p^{\mathrm{raw}}(b)
 &:=
 \left\|
  \frac{\varepsilon}{\sqrt n}
  +\frac{\|b-\beta_*\|_2}{n}g
 \right\|_2
 -\frac1n h^\top(b-\beta_*).
 \label{eq:int-raw-constraint}
\end{align}
The CGMT auxiliary objective is
\begin{equation}\label{eq:rawAO}
 A_p^{\mathrm{raw}}(b):=
 [C_p^{\mathrm{raw}}(b)]_+
 +\frac\kappa{\delta_p}\|b\|_{r,p}.
\end{equation}
Its deterministic-norm counterpart is
\begin{equation}\label{eq:detAO}
 \overline A_p(b):=
 [C_p(b)]_+
 +\frac\kappa{\delta_p}\|b\|_{r,p}.
\end{equation}
Define the minima
\begin{equation}
     \overline a_p^*:=\inf_b \overline A_p(b),
 \qquad
 a_{p,\mathrm{raw}}^*:=\inf_b A_p^{\mathrm{raw}}(b),
\end{equation}
and the set
$$
C_{p, M}:=\{b\in\mathbb R^p : \|b-\beta_*\|_{2, p}\le M, \|b\|_{r, p}\le M\}.
$$
\begin{lemma}[Characterization of the auxiliary optimization]
\label{lem:AO}
Assume that \eqref{eq:se1} holds and that
\(D_r(\tau_\kappa,\alpha_\kappa)\le\delta\). Then,
almost surely, the following statements hold.
\begin{enumerate}
 \item[(i)] For any $M<\infty$,
 \begin{equation}\label{eq:normreplace}
 \sup_{\|b-\beta_*\|_{2,p}\le M}
 |A_p^{\mathrm{raw}}(b)-\overline A_p(b)|\longrightarrow0.
 \end{equation}
 \item[(ii)] For any $L<\infty$, there exists
$M_L<\infty$ such that, for all sufficiently large $p$,
\begin{equation}\label{eq:AO-sublevel-localization}
 \{b:\overline A_p(b)\le L\}
 \cup
 \{b:A_p^{\mathrm{raw}}(b)\le L\}
 \subseteq C_{p,M_L}.
\end{equation}
\item[(iii)] The two auxiliary minimum values satisfy
\begin{equation}\label{eq:AO-minima-formal}
 \overline a_p^*\to\ell_\kappa,
 \qquad
 a_{p,\mathrm{raw}}^*\to\ell_\kappa,
\end{equation}
where
\begin{equation}\label{eq:ellstar}
 \ell_\kappa:=R_\kappa+\frac\kappa\delta\|f_\kappa\|_r.
 \end{equation}
\item[(iv)] For every $\epsilon>0$,
\begin{align}
\liminf_{p\to\infty}
\left(
 \inf_{b\in\mathcal B_{p,\epsilon}}\overline A_p(b)
 -
 \overline a_p^*
\right)
&>0,
\label{eq:detAO-gap-formal}\\
\liminf_{p\to\infty}
\left(
 \inf_{b\in\mathcal B_{p,\epsilon}}A_p^{\mathrm{raw}}(b)
 -
 a_{p,\mathrm{raw}}^*
\right)
&>0.
\label{eq:rawAO-gap-formal}
\end{align}
\end{enumerate}
\end{lemma}

\begin{proof}
Let $\varepsilon=\sigma z$ with $z\sim N(0,I_n)$ independent of $g$,
  and set $\rho:=\|b-\beta_*\|_2/\sqrt n$.  Then,
  \[
  \left\|\frac\varepsilon{\sqrt n}
  +\frac{\|b-\beta_*\|_2}{n}g\right\|_2
  =T_n(\rho):=\frac1{\sqrt n}\|\sigma z+\rho g\|_2.
  \]
For every fixed $\rho\ge0$,
  $T_n(\rho)^2\to\sigma^2+\rho^2$ almost surely. As $\sqrt{\sigma^2+\rho^2}$ is $1$-Lipschitz, for all $\rho,\rho'\ge0$,
  \[
  |T_n(\rho)-T_n(\rho')|
  \le\frac{\|g\|_2}{\sqrt n}|\rho-\rho'|.
  \]
Since $\|g\|_2/\sqrt n\to1$, a finite-net argument gives uniform
  convergence on every compact $\rho$ interval and
  \eqref{eq:normreplace} readily follows.

Consider the sublevel sets of $\overline A_p(b)$, and set $w=b-\beta_*$.  Using $[x]_+\ge x$ and
  $\|b\|_{r,p}\ge\|w\|_{r,p}-\|\beta_*\|_{r,p}$, we have
  \begin{align*}
  \overline A_p(b)
  &\ge \frac1{\delta_p}
  \left(\sqrt{\delta_p}\|w\|_{2,p}
  -\langle h,w\rangle_p+\kappa\|w\|_{r,p}\right)-C.
  \end{align*}
By \eqref{eq:emp-margin},
  $\overline A_p(b)\ge c\|w\|_{2,p}-C$ eventually.
  Hence, a bounded objective sublevel has bounded $\|w\|_{2,p}$.
Since $\overline A_p(b)\ge(\kappa/\delta_p)\|b\|_{r,p}$, the same
  sublevel has bounded $\|b\|_{r,p}$.

Moving to the sublevel sets of $A_p^{\mathrm{raw}}(b)$, the reverse triangle inequality yields
  \[
  \frac1{\sqrt n}\|\sigma z+\rho g\|_2
  \ge \frac{\|g\|_2}{\sqrt n}\rho
  -\sigma\frac{\|z\|_2}{\sqrt n}.
  \]
Let $a_n:=\|g\|_2/\sqrt n\to1$. Then, the same argument as above gives
  \[
  A_p^{\mathrm{raw}}(b)
  \ge\frac1{\delta_p}
  \left(a_n\sqrt{\delta_p}\|w\|_{2,p}
  -\langle h,w\rangle_p+\kappa\|w\|_{r,p}\right)-C.
  \]
Since
  $a_n\sqrt{\delta_p}-\gamma_{r,p}(\kappa)\to
  \sqrt\delta-\gamma_r(\kappa)>0$, the RHS is at least
  $c\|w\|_{2,p}-C$ eventually.
Again, $A_p^{\mathrm{raw}}(b)\ge(\kappa/\delta_p)\|b\|_{r,p}$, which concludes the proof of \eqref{eq:AO-sublevel-localization}.

We next prove statement \emph{(iii)}. For every \(b\),
 $\overline A_p(b)\ge Q_p(b)$.
Lemma~\ref{lem:stability} and \eqref{eq:minQ-convergence} therefore give
\[
 \liminf_{p\to\infty}\overline a_p^*
 \ge \cQ(f_\kappa).
\]
For the sequence
$
 b^\circ_{p,j}
 :=
 \tau_\kappa
 \eta_{\alpha_\kappa,r}
 \left(\frac{(\beta_*)_j}{\tau_\kappa}+h_j\right)
$, Lemma~\ref{lem:empirical}, \eqref{eq:se1}, and Gaussian integration by
parts give
\[
\begin{aligned}
 C_p(b_p^\circ)&\to R_\kappa,\qquad
 \|b_p^\circ\|_{r,p}\to\|f_\kappa\|_r.
\end{aligned}
\]
The hypothesis \(D_r(\tau_\kappa,\alpha_\kappa)\le\delta\) gives
\(R_\kappa\ge0\). Consequently,
\[
 \limsup_{p\to\infty}\overline a_p^*
 \le
 R_\kappa+\frac{\kappa}{\delta}\|f_\kappa\|_r
 =\ell_\kappa.
\]
A similar calculation gives
\[
 \cQ(f_\kappa)
 =
 R_\kappa+\frac{\kappa}{\delta}\|f_\kappa\|_r
 =\ell_\kappa.
\]
Thus, \(\overline a_p^*\to\ell_\kappa\).

Notice also that every near-minimizer of \(\overline A_p\) is a
near-minimizer of \(Q_p\). Indeed, if
\(\overline A_p(b_p)\le\overline a_p^*+o(1)\), then
\[
0\le
 Q_p(b_p)-\inf_bQ_p(b)
\le
 \overline A_p(b_p)-\inf_bQ_p(b)
\longrightarrow0.
\]
We now compare the minima of the two objectives \eqref{eq:rawAO} and \eqref{eq:detAO}.  Evaluating
both objectives at $b=\beta_*$ gives
\[
\overline A_p(\beta_*)
=
\sigma+\frac{\kappa}{\delta_p}\|\beta_*\|_{r,p}
=O(1)
\]
and
\[
A_p^{\mathrm{raw}}(\beta_*)
=
\frac{\|\varepsilon\|_2}{\sqrt n}
+\frac{\kappa}{\delta_p}\|\beta_*\|_{r,p}
=O(1)
\]
almost surely.  Hence there is a deterministic $L<\infty$ such that both
minimum values are at most $L$ for all sufficiently large $p$.
By statement~\emph{(ii)}, the minimizers of both objectives are then contained in
$C_{p,M_L}$ eventually.  Set
\begin{equation}\label{eq:Deltap}
\Delta_p
:=
\sup_{\|b-\beta_*\|_{2,p}\le M_L}
\left|A_p^{\mathrm{raw}}(b)-\overline A_p(b)\right|.
\end{equation}
Statement~\emph{(i)} gives $\Delta_p\to0$ almost surely.  Evaluating each objective
at a minimizer of the other one yields
\[
a_{p,\mathrm{raw}}^*
\le \overline a_p^*+\Delta_p,
\qquad
\overline a_p^*
\le a_{p,\mathrm{raw}}^*+\Delta_p.
\]
Therefore,
\[
\left|a_{p,\mathrm{raw}}^*-\overline a_p^*\right|
\le\Delta_p\longrightarrow0.
\]
Since $\overline a_p^*\to\ell_\kappa$, it follows that
$a_{p,\mathrm{raw}}^*\to\ell_\kappa$.

We finally prove statement~\emph{(iv)}. Fix \(\epsilon>0\) and work
on the probability-one event on which
Lemma~\ref{lem:stability}\emph{(iv)} and
statements~\emph{(i)}--\emph{(iii)} above hold.
Since \(\overline A_p(b)\ge Q_p(b)\),
\[
\begin{aligned}
\inf_{b\in\mathcal B_{p,\epsilon}}\overline A_p(b)
-\overline a_p^*
&\ge
\inf_{b\in\mathcal B_{p,\epsilon}}Q_p(b)
-\inf_bQ_p(b)\\
&\quad+
\inf_bQ_p(b)-\overline a_p^*.
\end{aligned}
\]
The last line tends to zero by \eqref{eq:minQ-convergence} and
statement~\emph{(iii)}. Lemma~\ref{lem:stability}\emph{(iv)}
therefore implies \eqref{eq:detAO-gap-formal}.

We prove~\eqref{eq:rawAO-gap-formal} by contradiction.  Suppose that
\[
\liminf_{p\to\infty}
\left(
\inf_{b\in\mathcal B_{p,\epsilon}}A_p^{\mathrm{raw}}(b)
-a_{p,\mathrm{raw}}^*
\right)=0.
\]
The expression inside braces is nonnegative.  Hence, there exists a subsequence,
still denoted by $p$, and points
$b_p\in\mathcal B_{p,\epsilon}$ such that
$
A_p^{\mathrm{raw}}(b_p)-a_{p,\mathrm{raw}}^*
\longrightarrow0$.
Here, if the infimum is not attained, choose $b_p$ within $1/p$
of it.
By statement~\emph{(iii)},
$
a_{p,\mathrm{raw}}^*\to\ell_\kappa$.
It follows that, by taking $L>\ell_\kappa+1$, 
$
A_p^{\mathrm{raw}}(b_p)\le L
$
for all sufficiently large $p$ in the subsequence. By statement~\emph{(ii)}, there exists $M_L<\infty$ such that
$
b_p\in C_{p,M_L}
$
eventually.
Setting $\Delta_p$ as in \eqref{eq:Deltap} and recalling that $\Delta_p\to0$, we have
\[
\begin{aligned}
0
&\le
\overline A_p(b_p)-\overline a_p^*\\
&\le
\left|\overline A_p(b_p)-A_p^{\mathrm{raw}}(b_p)\right|
+A_p^{\mathrm{raw}}(b_p)-a_{p,\mathrm{raw}}^*
+\left|a_{p,\mathrm{raw}}^*-\overline a_p^*\right|\\
&\le
\Delta_p
+A_p^{\mathrm{raw}}(b_p)-a_{p,\mathrm{raw}}^*
+\left|a_{p,\mathrm{raw}}^*-\overline a_p^*\right|
\longrightarrow0.
\end{aligned}
\]
On the other hand, $b_p\in\mathcal B_{p,\epsilon}$, so~\eqref{eq:detAO-gap-formal} implies that
$\overline A_p(b_p)-\overline a_p^*$ is bounded away from zero along every
sufficiently large subsequence, which gives a contradiction and concludes the argument.
\end{proof}

\subsection{Application of CGMT}
\label{app:cgmt-transfer}

We use the following nonasymptotic form of CGMT; see Theorem~II.1\emph{(i)}--\emph{(ii)} (and in particular equations~(16) and~(21)) of \cite{thrampoulidis2015gaussian}.
Let $\mathcal X\subseteq\mathbb R^p$ and
$\mathcal U\subseteq\mathbb R^n$ be nonempty compact sets, let
$G\in\mathbb R^{n\times p}$, $g\in\mathbb R^n$, and
$h\in\mathbb R^p$ have independent standard Gaussian entries, and let
$\psi:\mathcal X\times\mathcal U\to\mathbb R$ be continuous.  Define
\begin{align*}
\Phi(G)
&:=
\min_{x\in\mathcal X}\max_{u\in\mathcal U}
\left\{
 u^\top Gx+\psi(x,u)
\right\},\\
\phi(g,h)
&:=
\min_{x\in\mathcal X}\max_{u\in\mathcal U}
\left\{
 \|x\|_2g^\top u+\|u\|_2h^\top x+\psi(x,u)
\right\}.
\end{align*}
Then, for every $t\in\mathbb R$,
\begin{equation}\label{eq:CGMT-lower-general}
\Pp(\Phi(G)\le t)
\le
2\Pp(\phi(g,h)\le t).
\end{equation}
No convexity of $\mathcal X$ or $\mathcal U$ is required for \eqref{eq:CGMT-lower-general} to hold.  If, in addition, $\mathcal X$ and $\mathcal U$
are convex and $\psi$ is convex in $x$ and concave in $u$, then
\begin{equation}\label{eq:CGMT-upper-general}
\Pp(\Phi(G)\ge t)
\le
2\Pp(\phi(g,h)\ge t).
\end{equation}
We note that \cite{thrampoulidis2015gaussian} use strict inequalities in the events on the LHS. To include equality at the threshold, it suffices to replace $t$ by $t\pm \eta$ and then let $\eta\downarrow0$, using continuity of
probability for decreasing events.

In our case, $\psi$ also depends on the Gaussian noise
$\varepsilon$, which is independent of $G_X$, $g$, and $h$. Thus, we apply
these results conditionally on $\varepsilon$ and then
average over $\varepsilon$.
%Below, we use the convention $\inf\varnothing=+\infty$. % is used in
%\eqref{eq:raw-bad-gap-assumption}.

\begin{lemma}[Localization of minimizers and application of CGMT]\label{lem:cgmt-transfer}
For $w\in\mathbb R^p$, define
\begin{equation}\label{eq:primary-objective}
P_p(w)
:=
\max_{\|u\|_2\le1}
\left\{
 \frac1{\sqrt n}u^\top\varepsilon
 -\frac1n u^\top G_Xw
 +\frac{\kappa}{\delta_p}\|\beta_*+w\|_{r,p}
\right\}.
\end{equation}
For all $p$, let $S_p\subseteq\mathbb R^p$ be a deterministic closed set, and suppose
that
\begin{equation}\label{eq:raw-bad-gap-assumption}
\liminf_{p\to\infty}
\left(
 \inf_{b\in S_p}A_p^{\mathrm{raw}}(b)
 -a_{p,\mathrm{raw}}^*
\right)>0
\qquad\text{almost surely}.
\end{equation}
Then, almost surely, there exists $p_0<\infty$ such that, for every
$p\ge p_0$,
\begin{equation}\label{eq:bad-minimizer-exclusion}
\operatorname*{arg\,min}_{w\in\mathbb R^p}P_p(w)
\cap
\{w\in\mathbb R^p:\beta_*+w\in S_p\}
=\varnothing.
\end{equation}
Moreover,
\begin{equation}\label{eq:primary-value-limit}
\inf_{w\in\mathbb R^p}P_p(w)
\to\ell_\kappa
\qquad\text{almost surely}.
\end{equation}
\end{lemma}
%The convention $\inf\varnothing=+\infty$ is used in \eqref{eq:raw-bad-gap-assumption}.

\begin{proof}
We first prove that all minimizers of $P_p$ eventually belong to a deterministic
bounded set. A direct calculation gives that %the recession functional $P_p^\infty$ of $P_p$ is
\begin{equation}\label{eq:primary-recession}
P_p^\infty(v):=\lim_{t\to\infty}\frac{P_p(w+tv)-P_p(w)}{t}
=
\frac1n\|G_Xv\|_2
+\frac{\kappa}{\delta_p}\|v\|_{r,p}.
\end{equation}
We claim that there exists a deterministic $c_0>0$ such that, almost surely,
for all sufficiently large $p$,
\begin{equation}\label{eq:primary-recession-bound}
P_p^\infty(v)\ge c_0\|v\|_{2,p}
\qquad\text{for every }v\in\mathbb R^p.
\end{equation}
To show this, first note that, by positive homogeneity, it is enough to minimize
$P_p^\infty(v)$ over the sphere $\|v\|_{2,p}=1$. Define %the minimum recession slope of the primary optimization by
\[
R_p^{\mathrm{PO}}
:=
\inf_{\|v\|_{2,p}=1}P_p^\infty(v)
=
\inf_{\|v\|_{2,p}=1}
\left\{
 \frac1n\|G_Xv\|_2
 +\frac{\kappa}{\delta_p}\|v\|_{r,p}
\right\}.
\]
This quantity measures the smallest linear growth rate of \(P_p\) at
infinity. Using the dual representation of the Euclidean norm,
\[
\frac1n\|G_Xv\|_2
=
\sup_{\|u\|_2\le 1}
\left\{-\frac1n u^\top G_Xv\right\},
\]
the Gaussian min--max theorem associates with
\(R_p^{\mathrm{PO}}\) the auxiliary optimization
\[
 R_p^{\mathrm{AO}}
:=
\inf_{\|v\|_{2,p}=1}
\sup_{\|u\|_2\le1}
\left\{
 \frac{\|v\|_2}{n}g^\top u
 -\frac{\|u\|_2}{n}h^\top v
 +\frac{\kappa}{\delta_p}\|v\|_{r,p}
\right\},
\]
where \(g\in\mathbb R^n\) and \(h\in\mathbb R^p\) are independent
standard Gaussian vectors. Since \(\|v\|_{2,p}=1\), we have \(\|v\|_2=\sqrt p\).  For fixed \(v\),
write \(\rho=\|u\|_2\in[0,1]\).  Maximizing first over the direction of
\(u\) and then over its norm gives
\[
\begin{aligned}
&\sup_{\|u\|_2\le1}
\left\{
 \frac{\|v\|_2}{n}g^\top u
 -\frac{\|u\|_2}{n}h^\top v
\right\} 
=
\sup_{0\le\rho\le1}
\rho\left\{
 \frac{\|g\|_2}{\sqrt n\sqrt{\delta_p}}
 -\frac1{\delta_p}\langle h,v\rangle_p
\right\} 
=
\left[
 \frac{\|g\|_2}{\sqrt n\sqrt{\delta_p}}
 -\frac1{\delta_p}\langle h,v\rangle_p
\right]_+.
\end{aligned}
\]
%where \([x]_+:=\max\{x,0\}\).  
Indeed, the maximizing radius is
\(\rho=1\) when the expression in braces is positive and \(\rho=0\)
when it is negative.  Consequently,
\[
R_p^{\mathrm{AO}}
=
\inf_{\|v\|_{2,p}=1}
\left\{
 \left[
  \frac{\|g\|_2}{\sqrt n\sqrt{\delta_p}}
  -\frac1{\delta_p}\langle h,v\rangle_p
 \right]_+
 +\frac{\kappa}{\delta_p}\|v\|_{r,p}
\right\}.
\]
Using $[x]_+\ge x$ and the definition of $\gamma_{r,p}(\kappa)$, we have
\begin{align}
R_p^{\mathrm{AO}}
&\ge
\frac1{\delta_p}
\inf_{\|v\|_{2,p}=1}
\left\{
 \frac{\|g\|_2}{\sqrt n}\sqrt{\delta_p}
 -\langle h,v\rangle_p
 +\kappa\|v\|_{r,p}
\right\}\ge
\frac1{\delta_p}
\left(
 \frac{\|g\|_2}{\sqrt n}\sqrt{\delta_p}
 -\gamma_{r,p}(\kappa)
\right).
\label{eq:AO-recession-lower-bound}
\end{align}
Since
\[
\frac{\|g\|_2}{\sqrt n}\longrightarrow1,
\qquad
\delta_p\to\delta,
\qquad
\gamma_{r,p}(\kappa)\to\gamma_r(\kappa)
\]
almost surely, the RHS of
\eqref{eq:AO-recession-lower-bound} converges to
\[
c_*:=\frac{\sqrt\delta-\gamma_r(\kappa)}{\delta}>0,
\]
where the inequality follows from Lemma \ref{lem:margin}.
The RHS of \eqref{eq:AO-recession-lower-bound} also converges
in $L_1$ as $
0\le\gamma_{r,p}(\kappa)\le\|h\|_{2,p}$, hence
\[
\liminf_{p\to\infty}\E R_p^{\mathrm{AO}}\ge c_*.
\]
Furthermore, $R_p^{\mathrm{AO}}$ is a
$C/\sqrt p$-Lipschitz function of $(g,h)$, for a deterministic $C$ independent
of $p$.  Gaussian concentration therefore gives deterministic constants
$a_0,c>0$ such that
\begin{equation}\label{eq:AO-recession-tail}
\Pp(R_p^{\mathrm{AO}}\le a_0)\le e^{-cp}
\end{equation}
for all sufficiently large $p$.
Applying 
\eqref{eq:CGMT-lower-general} (cf.\ \cite[Theorem~II.1\emph{(i)}]{thrampoulidis2015gaussian}), to
$
\inf_{\|v\|_{2,p}=1}P_p^\infty(v)
$, we conclude that
\[
\Pp\left(
 \inf_{\|v\|_{2,p}=1}P_p^\infty(v)\le a_0
\right)
\le
2\Pp(R_p^{\mathrm{AO}}\le a_0)
\le2e^{-cp}.
\]
The probabilities are summable, so Borel--Cantelli proves
\eqref{eq:primary-recession-bound}, with $c_0=a_0$.

The triangle inequality and \eqref{eq:primary-recession-bound} imply
\begin{align}
P_p(w)
&\ge
P_p^\infty(w)
-\frac{\|\varepsilon\|_2}{\sqrt n}
-\frac{\kappa}{\delta_p}\|\beta_*\|_{r,p}\ge
c_0\|w\|_{2,p}-C
\label{eq:primary-coercivity}
\end{align}
eventually almost surely, where $C<\infty$ is deterministic.  Since
$P_p(0)=O(1)$ almost surely, every minimizer of $P_p$ has bounded
$\|w\|_{2,p}$.  Moreover,
$
P_p(w)\ge\frac{\kappa}{\delta_p}\|\beta_*+w\|_{r,p}$,
so $\|\beta_*+w\|_{r,p}$ is bounded as well.
Consequently, there exists a deterministic $M<\infty$ such that, almost
surely, all minimizers of $P_p$ eventually belong to
\[
\mathcal C^w_{p,M}
:=
\left\{
w\in\mathbb R^p:
\|w\|_{2,p}\le M,\ 
\|\beta_*+w\|_{r,p}\le M
\right\}.
\]
Increasing $M$ if necessary, Lemma~\ref{lem:AO}\emph{(ii)}-\emph{(iii)} ensures that all
minimizers of $A_p^{\mathrm{raw}}$ eventually satisfy
$
b-\beta_*\in\mathcal C^w_{p,M}$.
Define %the localized auxiliary
%values
\begin{align*}
V_p
&:=
\inf_{\substack{b\in\mathbb R^p\\b-\beta_*\in\mathcal C^w_{p,M}}}
A_p^{\mathrm{raw}}(b),\qquad
W_p
:=
\inf_{\substack{b\in S_p\\b-\beta_*\in\mathcal C^w_{p,M}}}
A_p^{\mathrm{raw}}(b).
\end{align*}
If the feasible set defining $W_p$ is empty, then \eqref{eq:bad-minimizer-exclusion} trivially holds. We will therefore assume that the feasible set is not empty.
By the localization of the minimizers of $A_p^{\rm raw}$ proved in Lemma \ref{lem:AO}, we have that
$
V_p=a_{p,\mathrm{raw}}^*
$
eventually almost surely.  Hence, \eqref{eq:raw-bad-gap-assumption} implies
\begin{equation}\label{eq:localized-random-gap}
\Gamma
:=
\liminf_{p\to\infty}(W_p-V_p)>0
\qquad\text{almost surely},
\end{equation}
which also implies that
\[
c_S
:=
\frac14\E[\min\{\Gamma,1\}]>0.
\]
Fatou's lemma gives
$
\liminf_{p\to\infty}
\E\big[\min\{W_p-V_p,1\}\big]
\ge
\E[\min\{\Gamma,1\}]
=4c_S$.
Since $W_p\ge V_p$, it follows that, for all sufficiently large $p$,
\begin{equation}\label{eq:mean-value-gap}
\E W_p-\E V_p\ge2c_S.
\end{equation}

We now apply the CGMT results.  Define
\begin{align*}
\Phi_p
&:=
\inf_{w\in\mathcal C^w_{p,M}}P_p(w),\qquad
\Phi_p^S
:=
\inf_{\substack{w\in\mathcal C^w_{p,M}\\\beta_*+w\in S_p}}P_p(w).
\end{align*}
For any $t\in\mathbb R$, 
\eqref{eq:CGMT-lower-general} (cf.\
\cite[Theorem~II.1\emph{(i)}]{thrampoulidis2015gaussian}) gives
\begin{equation}\label{eq:CGMT-bad-lower-tail}
\Pp(\Phi_p^S\le t)
\le
2\Pp(W_p\le t).
\end{equation}
Furthermore,  $\mathcal C^w_{p,M}$ is convex,
$\{u:\|u\|_2\le1\}$ is convex, and
$
\psi_p(w,u)
=
\frac1{\sqrt n}u^\top\varepsilon
+\frac{\kappa}{\delta_p}\|\beta_*+w\|_{r,p}
$
is convex in $w$ and affine, hence concave, in $u$.  Thus, \eqref{eq:CGMT-upper-general} (cf.\
\cite[Theorem~II.1\emph{(ii)}]
{thrampoulidis2015gaussian}) gives
\begin{equation}\label{eq:CGMT-good-upper-tail}
\Pp(\Phi_p\ge t)
\le
2\Pp(V_p\ge t).
\end{equation}
Let
$
t_p:=(\E V_p+\E W_p)/2$. By \eqref{eq:mean-value-gap}, $t_p-\E V_p\ge c_S$ and
$\E W_p-t_p\ge c_S$.
Note that both $V_p$ and $W_p$ are $C_M/\sqrt n$-Lipschitz functions of their
standard Gaussian data.  %For example, changing $G_X$ by $\Delta$ changes the primary bilinear term by at most
%\[
%\frac1n\|u\|_2\|\Delta\|_{\mathrm F}\|w\|_2
%\le
%\frac{M\sqrt p}{n}\|\Delta\|_{\mathrm F}
%\le
%\frac{C_M}{\sqrt n}\|\Delta\|_{\mathrm F},
%\]
%and the analogous bounds hold for the $g$, $h$, and standard Gaussian noise
%terms in the auxiliary problem.  Taking an infimum or supremum over a fixed
%set preserves the Lipschitz constant.
Gaussian concentration therefore yields
\begin{equation}\label{eq:AO-separation-tails}
\Pp(W_p\le t_p)+\Pp(V_p\ge t_p)
\le2e^{-c_S'n}
\end{equation}
for a deterministic $c_S'>0$ and all sufficiently large $p$.
Combining \eqref{eq:CGMT-bad-lower-tail},
\eqref{eq:CGMT-good-upper-tail}, and
\eqref{eq:AO-separation-tails} gives
\[
\Pp(\Phi_p^S\le\Phi_p)
\le4e^{-c_S'n}.
\]
As these probabilities are summable,  Borel-Cantelli
 allows us to conclude that $
\Phi_p^S>\Phi_p$ eventually almost surely.  Thus, no minimizer of $P_p$ in
$\mathcal C^w_{p,M}$ satisfies $\beta_*+w\in S_p$ and since all primary
minimizers are eventually contained in $\mathcal C^w_{p,M}$,
\eqref{eq:bad-minimizer-exclusion} follows.

It remains to prove \eqref{eq:primary-value-limit}. By Lemma~\ref{lem:AO}\emph{(iii)},
\[
V_p=a_{p,\mathrm{raw}}^*\to\ell_\kappa
\qquad\text{almost surely}
\]
Moreover, $0\le V_p\le A_p^{\mathrm{raw}}(\beta_*)$ and
the latter variables have uniformly bounded second moments.  Hence,
$(V_p)_p$ is uniformly integrable and
$\E V_p\to\ell_\kappa$. Gaussian concentration then implies that, for any $t>0$,
$$
\Pp(|V_p-\ell_\kappa|>t)\le2e^{-c_tn},
$$
for all sufficiently large $p$. Hence, using  \eqref{eq:CGMT-lower-general}-\eqref{eq:CGMT-upper-general} (cf.\ \cite[Theorem~II.1\emph{(i)}-\emph{(ii)}]{thrampoulidis2015gaussian}), we conclude that
\[
\Pp(|\Phi_p-\ell_\kappa|>t)
\le4e^{-c_tn}.
\]
Another application of Borel-Cantelli then implies that 
\[
\Phi_p\to\ell_\kappa
\qquad\text{almost surely}.
\]
Finally, all primary minimizers are eventually contained in
$\mathcal C^w_{p,M}$, so
\[
\Phi_p=\inf_{w\in\mathbb R^p}P_p(w)
\]
eventually almost surely, which concludes the argument.
\end{proof}

\subsection{Completing the proof}
\label{app:completion-nonint}

\begin{proof}
Set \(w=b-\beta_*\), so that \(y-Xb=\varepsilon-Xw\). Since
\(X=G_X/\sqrt n\), dividing the objective in \eqref{eq:minprob} by \(\sqrt n\)
gives
\[
 \frac1{\sqrt n}
 \left\|
   \varepsilon-\frac1{\sqrt n}G_Xw
 \right\|_2
 +\frac{\kappa}{\delta_p}\|\beta_*+w\|_{r,p}.
\]
As $\|a\|_2=\max_{\|u\|_2\le1}u^\top a$,
the preceding objective equals
\[
 P_p(w)
 :=
 \max_{\|u\|_2\le1}
 \left\{
   \frac1{\sqrt n}u^\top\varepsilon
   -\frac1n u^\top G_Xw
   +\frac{\kappa}{\delta_p}\|\beta_*+w\|_{r,p}
 \right\}.
\]
Thus, \(b\) minimizes the objective in \eqref{eq:minprob} if and only if
\(w=b-\beta_*\) minimizes \(P_p\).

We next identify the auxiliary optimization appearing in Lemmas \ref{lem:AO}-\ref{lem:cgmt-transfer}. Conditionally on \(\varepsilon\), the CGMT auxiliary problem
associated with the bilinear term in \(P_p\) is
\[
 \min_w\max_{\|u\|_2\le1}
 \left\{
   \frac{\|w\|_2}{n}g^\top u
   -\frac{\|u\|_2}{n}h^\top w
   +\frac1{\sqrt n}u^\top\varepsilon
   +\frac{\kappa}{\delta_p}\|\beta_*+w\|_{r,p}
 \right\}.
\]
For fixed \(w\), the same radial maximization used in the proof of
Lemma~\ref{lem:cgmt-transfer} gives
\[
\max_{\|u\|_2\le1}
\left\{
u^\top\left(
\frac{\varepsilon}{\sqrt n}
+\frac{\|w\|_2}{n}g
\right)
-\frac{h^\top w}{n}\|u\|_2
\right\}
=
\left[
\left\|
\frac{\varepsilon}{\sqrt n}
+\frac{\|w\|_2}{n}g
\right\|_2
-\frac{h^\top w}{n}
\right]_+.
\]
Consequently, the scalar auxiliary objective is
\[
 A_p^{\mathrm{raw}}(b)
 =
 \left[
   \left\|
     \frac{\varepsilon}{\sqrt n}
     +\frac{\|b-\beta_*\|_2}{n}g
   \right\|_2
   -\frac1n h^\top(b-\beta_*)
 \right]_+
 +\frac{\kappa}{\delta_p}\|b\|_{r,p},
\]
which is studied in Lemma \ref{lem:AO}.

We now transfer the stability conclusions from the auxiliary optimization to the original
primary problem. 
Note that $\mathcal B_{p,\varepsilon}$ is a  closed set. Indeed, the two moment maps
appearing in \(d_p\) are continuous, and the coupling that pairs the
\(j\)-th atom of \(\mu_{p,b}\) with the \(j\)-th atom of
\(\mu_{p,b'}\) gives
$ W_2(\mu_{p,b},\mu_{p,b'})
 \le \|b-b'\|_{2,p}$.
Thus, \(b\mapsto W_2(\mu_{p,b},\mu_\kappa)\) is continuous.

For any \(\varepsilon>0\), Lemma \ref{lem:AO}\emph{(iv)} gives
\[
 \liminf_{p\to\infty}
 \left(
   \inf_{b\in\mathcal B_{p,\varepsilon}}
        A_p^{\mathrm{raw}}(b)
   -
   \inf_{b\in\mathbb R^p}
        A_p^{\mathrm{raw}}(b)
 \right)>0
 \qquad\text{almost surely}.
\]
Thus, by Lemma \ref{lem:cgmt-transfer}, almost surely, for all
sufficiently large \(p\), no primary minimizer \(w\) satisfies
\(\beta_*+w\in\mathcal B_{p,\varepsilon}\), which gives
\begin{equation}\label{eq:convclaim}
 \sup_{\widehat\beta\in\mathcal M_p}
 d_p(\widehat\beta)\longrightarrow0
 \qquad\text{almost surely},
\end{equation}
where
\(\mathcal M_p\) denotes the set of minimizers of \eqref{eq:minprob}.
In particular, uniformly over all choices of minimizer,
\[
 \|\widehat\beta-\beta_*\|_{2,p}^2
 \to
 M_\kappa^{\mathrm{err}}
 =
 \tau_\kappa^2H_r(\tau_\kappa,\alpha_\kappa).
\]
Hence, the convergence of the MSE in \eqref{eq:resclaim} follows
from \eqref{eq:se1}. Furthermore, \eqref{eq:convclaim} implies that
\[
 W_2(\mu_{p,\widehat\beta},\mu_\kappa)\longrightarrow0,
\]
which proves the first convergence in \eqref{eq:int-law}
of Remark~\ref{rem:joint-empirical-laws}. Note that \eqref{eq:convclaim} also implies that
\[
 \|\widehat\beta\|_{r,p}\to\|f_\kappa\|_r.
\]
Lemma \ref{lem:cgmt-transfer} gives
\[
 \min_w P_p(w)\to\ell_\kappa
 \qquad\text{almost surely},
\]
where, by \eqref{eq:ellstar},
$
 \ell_\kappa
 =
 R_\kappa+\frac{\kappa}{\delta}\|f_\kappa\|_r$.
For any \(\widehat\beta\in\mathcal M_p\), setting
\(\widehat w=\widehat\beta-\beta_*\) gives
\[
 \min_wP_p(w)
 =
 P_p(\widehat w)
 =
 \frac1{\sqrt n}\|y-X\widehat\beta\|_2
 +\frac{\kappa}{\delta_p}\|\widehat\beta\|_{r,p}.
\]
Since \(\delta_p\to\delta\), we conclude that
\[
 \frac1{\sqrt n}\|y-X\widehat\beta\|_2
 \to
 \ell_\kappa-\frac{\kappa}{\delta}\|f_\kappa\|_r
 =
 R_\kappa,
\]
which gives the convergence of the training error in \eqref{eq:resclaim} and concludes the proof.
\end{proof}

\section{Proof of Lemma~\ref{thm:min-norm-interpolator}}
\label{app:pfint}

%We first record the finite-dimensional estimate used at the end of the proof.  Recall that \(X=G_X/\sqrt n\), where \(G_X\) has independent standard Gaussian entries.

We start by isolating the ingredient used to pass from
penalized optimization to exact interpolation. First, we show that every
residual can be corrected with a small
\(\ell_r\) cost. This implies that, if the penalty on the residual is large enough, the solution of the regularized problem is an exact interpolator.

\begin{lemma}%[Gaussian quotient estimate]
\label{lem:int-gaussian-quotient}
Fix \(1\le r<\infty\) and \(\delta\in(0,1)\).
There is a deterministic constant \(C_{r,\delta}<\infty\) such that,
almost surely for all sufficiently large \(p\), for every
\(z\in\mathbb R^n\) there exists \(u\in\mathbb R^p\) satisfying
\begin{equation}
 Xu=z,
 \qquad
 \|u\|_{r,p}
 \le C_{r,\delta}\frac{\|z\|_2}{\sqrt n}.
 \label{eq:int-gaussian-quotient}
\end{equation}
\end{lemma}

\begin{proof}
Let \(r^*\) be the conjugate exponent of $r$ and \(N_p(u)=\|u\|_{r,p}\). Its dual norm is $
 N_p^*(a)=p^{1/r}\|a\|_{r^*}$. We claim that there is a deterministic \(c_{r,\delta}>0\) such that,
with probability at least \(1-e^{-c p}\),
\begin{equation}
 \inf_{\|u\|_2=1}N_p^*(X^\top u)
 \ge c_{r,\delta}\sqrt n.
 \label{eq:int-dual-lower-bound}
\end{equation}
Suppose first that \(r^*\ge2\). Note that
$p^{1/r}\|a\|_{r^*}\ge \sqrt p\,\|a\|_2$. Thus, 
Corollary 5.35 of \cite{vershynin2012nonasymptotic} gives
\[
 \mathbb P\bigl(
  s_{\min}(\sqrt{n} X^\top)\le\sqrt p-\sqrt n-t
 \bigr)\le 2e^{-t^2/2},
\]
where $s_{\min}(\cdot)$ denotes the smallest singular value. This proves \eqref{eq:int-dual-lower-bound}, as
\(n/p\to\delta<1\).

It remains to consider \(1<r^*<2\). For every fixed
\(u\in\mathbb S^{n-1}\) and \(0<a\le1\), taking the exponential and applying  Markov's
inequality gives
\begin{align*}
 \mathbb P\left(
  p^{-1/r^*}\|\sqrt{n}X^\top u\|_{r^*}\le a
 \right)
 &\le
 e^p\left(\mathbb E
  e^{-|G|^{r^*}/a^{r^*}}\right)^p
 \le (C a)^p,
\end{align*}
where $G\sim N(0,1)$ and $C>0$ is a deterministic constant
depending only on $r$.
%Indeed, the last expectation is at most a constant times \(a\), by the change of variables \(x=at\) and boundedness of the Gaussian density. 
Moreover,
\[
 \sup_{\|u\|_2=1}p^{-1/r^*}\|\sqrt{n}X^\top u\|_{r^*}
 \le p^{-1/2}\|\sqrt{n}X^\top\|_{\mathrm{op}},
\]
and hence this supremum is at most a deterministic constant \(K\)
except on an event of probability \(e^{-c p}\).  On the complementary event, take
an \(a/(2K)\)-net of \(\mathbb S^{n-1}\). Its cardinality is at most
\((1+4K/a)^n\) by Corollary 4.2.13 of \cite{vershynin2018highdimensional}, so the union bound yields
\[
 \mathbb P\left(
  \inf_{\|u\|_2=1}p^{-1/r^*}\|\sqrt{n}X^\top u\|_{r^*}<\frac a2
 \right)
 \le e^{-c p}+(1+4K/a)^n(C a)^p.
\]
Since \(n/p\to\delta<1\), the RHS is at most \(2e^{-c'p}\)
when \(a>0\) is chosen sufficiently small. Thus,
\[
 N_p^*(X^\top u)
 =\frac{p^{1/r}}{\sqrt n}\|\sqrt{n}X^\top u\|_{r^*}
 \ge c\frac p{\sqrt n}\|u\|_2
 \ge c_{r,\delta}\sqrt n\,\|u\|_2.
\]
%which proves \eqref{eq:int-dual-lower-bound} in the remaining case.
The failure probabilities are summable, so
\eqref{eq:int-dual-lower-bound} holds eventually almost surely.
Eventually \(X\) has full row rank.  Convex duality then gives, for
every \(z\in\mathbb R^n\),
\[
 \inf_{Xu=z}N_p(u)
 =\sup\left\{
  u^\top z:N_p^*(X^\top u)\le1
 \right\}
 \le \frac{\|z\|_2}{c_{r,\delta}\sqrt n},
\]
which is the desired result.
\end{proof}

\begin{lemma}[Exact penalty for interpolation]
\label{lem:int-exact-penalty}
Fix \(L>C_{r,\delta}\), where $C_{r, \delta}$ is the deterministic constant of Lemma \ref{lem:int-gaussian-quotient}. Almost surely, for all sufficiently large
\(p\) and every \(y\in\mathbb R^n\),
\begin{equation}
 \operatorname*{arg\,min}_{b\in\mathbb R^p}
 \left\{
  \|b\|_{r,p}+\frac L{\sqrt n}\|y-Xb\|_2
 \right\}
 =
 \operatorname*{arg\,min}_{Xb=y}\|b\|_{r,p}.
 \label{eq:int-exact-penalty}
\end{equation}
\end{lemma}

\begin{proof}
Let $V_p(y):=\inf_{X\widetilde b=y}\|\widetilde b\|_{r,p}$. Lemma~\ref{lem:int-gaussian-quotient} shows that the feasible set is
nonempty; the infimum is attained because it is a finite-dimensional
norm minimization over a closed affine set.
For arbitrary \(b\), applying Lemma~\ref{lem:int-gaussian-quotient} to
\(z=y-Xb\) gives that there exists \(u\) such that \(X(b+u)=y\) and
\[
 \|u\|_{r,p}
 \le
 C_{r,\delta}\frac{\|y-Xb\|_2}{\sqrt n}.
\]
Therefore
\[
 \|b\|_{r,p}+\frac L{\sqrt n}\|y-Xb\|_2
 \ge
 V_p(y)
 +\frac{L-C_{r,\delta}}{\sqrt n}\|y-Xb\|_2.
\]
Since \(L>C_{r,\delta}\), equality is possible precisely when
\(Xb=y\) and \(\|b\|_{r,p}=V_p(y)\).
\end{proof}

We next extend the results of Lemma \ref{lem:cgmt-transfer} by allowing the loss and regularizer to have independent coefficients.

\begin{lemma}[Extension of Lemma \ref{lem:cgmt-transfer}]
\label{lem:weighted-cgmt-transfer}
Let \(\rho_p,\lambda_p>0\) be deterministic sequences such that
$ \rho_p\to\rho\in(0,\infty)$ and
$\lambda_p\to\lambda\in(0,\infty)$.
Fix \(\kappa_0>0\) such that
\(\gamma_r(\kappa_0)<\sqrt\delta\), and suppose that, eventually,
\begin{equation}
 \lambda_p\ge\frac{\delta_p\rho_p}{\kappa_0}.
 \label{eq:weighted-coefficient-condition}
\end{equation}
Define
\begin{align}
 P_p^{\rho,\lambda}(w)
 &:=
 \rho_p\|\beta_*+w\|_{r,p}
 +\frac{\lambda_p}{\sqrt n}
 \left\|
  \varepsilon-\frac1{\sqrt n}G_Xw
 \right\|_2,
 \label{eq:weighted-primary-objective}\\
 A_p^{\rho,\lambda,\mathrm{raw}}(b)
 &:=
 \rho_p\|b\|_{r,p}
 +\lambda_p[C_p^{\mathrm{raw}}(b)]_+.
 \label{eq:weighted-raw-objective}
\end{align}
For every \(p\), let \(S_p\subseteq\mathbb R^p\) be deterministic
and closed. Suppose that, for some deterministic \(\ell<\infty\),
\begin{equation}
 \inf_bA_p^{\rho,\lambda,\mathrm{raw}}(b)
 \to\ell
 \qquad\text{almost surely},
 \label{eq:weighted-raw-minimum}
\end{equation}
and
\begin{equation}
 \liminf_{p\to\infty}
 \left(
  \inf_{b\in S_p}A_p^{\rho,\lambda,\mathrm{raw}}(b)
  -
  \inf_bA_p^{\rho,\lambda,\mathrm{raw}}(b)
 \right)>0
 \qquad\text{almost surely}.
 \label{eq:weighted-raw-gap}
\end{equation}
Then, almost surely, for all sufficiently large \(p\),
\begin{equation}\label{eq:exclusion-statement}
 \operatorname*{arg\,min}_{w\in\mathbb R^p}
 P_p^{\rho,\lambda}(w)
 \cap
 \{w:\beta_*+w\in S_p\}
 =\varnothing.
\end{equation}
Moreover,
\begin{equation}\label{eq:value-con}
 \inf_wP_p^{\rho,\lambda}(w)
 \to\ell
 \qquad\text{almost surely}.
\end{equation}
\end{lemma}

\begin{proof}
By the dual representation of the Euclidean norm,
\begin{align*}
 P_p^{\rho,\lambda}(w)
 =
 \max_{\|u\|_2\le1}
 \left\{
  \frac{\lambda_p}{\sqrt n}u^\top\varepsilon
  -\frac{\lambda_p}{n}u^\top G_Xw
  +\rho_p\|\beta_*+w\|_{r,p}
 \right\}.
\end{align*}
The associated CGMT auxiliary problem is obtained exactly as in the
proof of Lemma~\ref{lem:cgmt-transfer}. Maximizing first over the
direction and then over the norm of \(u\) shows that its scalar
objective is
\[
 A_p^{\rho,\lambda,\mathrm{raw}}(\beta_*+w)
 =
 \rho_p\|\beta_*+w\|_{r,p}
 +\lambda_p[C_p^{\mathrm{raw}}(\beta_*+w)]_+.
\]
We first verify the localization needed to apply the CGMT on
deterministic compact sets. Define
\[
 R_p^{\mathrm{AO},\rho,\lambda}
 :=
 \inf_{\|v\|_{2,p}=1}
 \left\{
  \rho_p\|v\|_{r,p}
  +\lambda_p
  \left[
   \frac{\|g\|_2}{\sqrt n\sqrt{\delta_p}}
   -\frac1{\delta_p}\langle h,v\rangle_p
  \right]_+
 \right\}.
\]
Set $
 \nu_p:=\frac{\delta_p\rho_p}{\kappa_0}$.
Since \(\lambda_p\ge\nu_p\) eventually and
\(\lambda_p[x]_+\ge\nu_px\) for every \(x\in\mathbb R\), we obtain
\begin{align*}
 R_p^{\mathrm{AO},\rho,\lambda}
 &\ge
 \frac{\rho_p}{\kappa_0}
 \inf_{\|v\|_{2,p}=1}
 \left\{
  \frac{\|g\|_2}{\sqrt n}\sqrt{\delta_p}
  -\langle h,v\rangle_p
  +\kappa_0\|v\|_{r,p}
 \right\}\ge
 \frac{\rho_p}{\kappa_0}
 \left(
  \frac{\|g\|_2}{\sqrt n}\sqrt{\delta_p}
  -\gamma_{r,p}(\kappa_0)
 \right).
\end{align*}
By Lemma~\ref{lem:empirical-dual-margin}, the last expression
converges almost surely and in \(L_1\) to
$
 c_*:=
 \frac{\rho}{\kappa_0}
 \{\sqrt\delta-\gamma_r(\kappa_0)\}>0$.
Moreover, since \(\lambda_p\) is bounded,
\(R_p^{\mathrm{AO},\rho,\lambda}\) is a
\(C/\sqrt p\)-Lipschitz function of \((g,h)\). Gaussian
concentration and an application of \cite[Theorem~II.1\emph{(i)}]{thrampoulidis2015gaussian} as in \eqref{eq:primary-recession-bound}--\eqref{eq:AO-recession-tail}
therefore gives that, almost surely
eventually,
\begin{equation}
 (P_p^{\rho,\lambda})^\infty(v)
 \ge c\|v\|_{2,p},
 \label{eq:weighted-primary-recession-bound}
\end{equation}
for $v\in\mathbb R^p$ and a deterministic $c>0$.
The triangle inequality then gives
\[
 P_p^{\rho,\lambda}(w)
 \ge
 c\|w\|_{2,p}
 -\rho_p\|\beta_*\|_{r,p}
 -\frac{\lambda_p}{\sqrt n}\|\varepsilon\|_2.
\]
Since the last two terms are eventually bounded and
\(P_p^{\rho,\lambda}(0)=O(1)\), all primary minimizers have bounded
\(\|w\|_{2,p}\). Furthermore,
\[
 P_p^{\rho,\lambda}(w)
 \ge\rho_p\|\beta_*+w\|_{r,p},
\]
and \(\rho_p\) is bounded away from zero, so their
\(\|\beta_*+w\|_{r,p}\) norms are bounded as well.

The minimizers of $A_p^{\rho,\lambda,\mathrm{raw}}$ satisfy the same localization. Indeed,
with \(a_n:=\|g\|_2/\sqrt n\), the reverse triangle inequality gives,
almost surely eventually,
\begin{align*}
 A_p^{\rho,\lambda,\mathrm{raw}}(\beta_*+w)
 \ge
 \frac{\rho_p}{\kappa_0}
 \bigl\{
  a_n\sqrt{\delta_p}\|w\|_{2,p}
  -\langle h,w\rangle_p
  +\kappa_0\|w\|_{r,p}
 \bigr\}
 -C,
\end{align*}
for a deterministic \(C<\infty\). By homogeneity and the definition
of \(\gamma_{r,p}(\kappa_0)\), the expression in braces is at least
$
 \{a_n\sqrt{\delta_p}-\gamma_{r,p}(\kappa_0)\}
 \|w\|_{2,p}$. As $a_n\sqrt{\delta_p}-\gamma_{r,p}(\kappa_0)$ converges to
\(\sqrt\delta-\gamma_r(\kappa_0)>0\), we conclude that
\[
 A_p^{\rho,\lambda,\mathrm{raw}}(\beta_*+w)
 \ge c\|w\|_{2,p}-C
\]
eventually. Since also
$
 A_p^{\rho,\lambda,\mathrm{raw}}(b)
 \ge\rho_p\|b\|_{r,p}$,
the minimizers of $A_p^{\rho,\lambda,\mathrm{raw}}$ are localized as claimed.

Consequently, there is a deterministic \(M<\infty\) such that all minimizers of $P_p^{\rho,\lambda}$ and  $A_p^{\rho,\lambda,\mathrm{raw}}$ eventually belong to
\[
 \mathcal C_{p,M}^w
 :=
 \{w:\|w\|_{2,p}\le M,\
       \|\beta_*+w\|_{r,p}\le M\}.
\]
%Enlarge \(M\), if necessary, so that \(0\in\mathcal C_{p,M}^w\), and
%The sets over which these infima are taken are deterministic and
%compact.
By localization and \eqref{eq:weighted-raw-gap},
\[
 \liminf_{p\to\infty}(\inf_{\substack{w\in\mathcal C_{p,M}^w\\
                  \beta_*+w\in S_p}}
 A_p^{\rho,\lambda,\mathrm{raw}}(\beta_*+w)- \inf_{w\in\mathcal C_{p,M}^w}
 A_p^{\rho,\lambda,\mathrm{raw}}(\beta_*+w))>0
 \qquad\text{almost surely}.
\]
The argument in the proof of Lemma~\ref{lem:cgmt-transfer} starting from
\eqref{eq:localized-random-gap} %through
%\eqref{eq:AO-separation-tails} 
and involving the application of \cite[Theorem~II.1\emph{(ii)}]{thrampoulidis2015gaussian}) now applies without change, giving \eqref{eq:exclusion-statement}. %Indeed,
%the weights \(\rho_p,\lambda_p\) are bounded, so \(V_p\) and \(W_p\)
%remain \(C_M/\sqrt n\)-Lipschitz functions of the Gaussian data; the
%convexity assumptions needed for the upper CGMT comparison are
%unchanged; and the lower comparison does not require \(S_p\) to be
%convex. We therefore obtain
%\[
% \Pp(\Phi_p^S\le\Phi_p)\le Ce^{-cn}.
%\]
%Borel--Cantelli and primary localization prove the asserted
%exclusion of minimizers.

It remains to transfer the minimum value. By localization and
\eqref{eq:weighted-raw-minimum},
\[
 V_p\to\ell
 \qquad\text{almost surely}.
\]
Moreover,
\[
 0\le V_p
 \le
 \rho_p\|\beta_*\|_{r,p}
 +\frac{\lambda_p}{\sqrt n}\|\varepsilon\|_2,
\]
so \((V_p)_p\) is uniformly integrable. Thus, the argument at the end
of the proof of Lemma~\ref{lem:cgmt-transfer} therefore
gives
\[
 \Phi_p\to\ell
 \qquad\text{almost surely}.
\]
Since all minimizers of $P_p^{\rho,\lambda}$ are eventually contained in
\(\mathcal C_{p,M}^w\), \(\Phi_p=\inf_wP_p^{\rho,\lambda}(w)\)
eventually, which proves \eqref{eq:value-con}.
\end{proof}

At this point, we are ready to give the proof of Lemma \ref{thm:min-norm-interpolator}.

\begin{proof}[Proof of Lemma~\ref{thm:min-norm-interpolator}]
Let us define
\[
 \kappa_{\mathrm{int}}
 :=
 \begin{cases}
  \alpha_{\mathrm{int}},
     & r=1,\\[1mm]
  \alpha_{\mathrm{int}}
  U_r(\tau_{\mathrm{int}},\alpha_{\mathrm{int}})^{(r-1)/r},
     & 1<r<\infty,
 \end{cases}
\]
where $U_r$ is defined in \eqref{eq:Hdef}. We first verify that Lemma~\ref{lem:AO} applies at the
interpolation parameters. Equation~\eqref{eq:int-state-mse}
and the definition of $\kappa_{\mathrm{int}}>0$ above imply
\eqref{eq:se1} with
$(\tau_\kappa,\alpha_\kappa,\kappa,f_\kappa)
=(\tau_{\mathrm{int}},\alpha_{\mathrm{int}},
\kappa_{\mathrm{int}},f_{\mathrm{int}})$.
Moreover, the second equation in \eqref{eq:int-state-mse} gives
$D_r(\tau_{\mathrm{int}},\alpha_{\mathrm{int}})=\delta$,
which satisfies the remaining hypothesis of the lemma
and implies $R_\kappa=0$.
Moreover, \(\kappa_{\mathrm{int}}>0\). %: this is immediate when \(r=1\);
%when \(r>1\), the proximal output is nonzero with positive probability, and hence
%\(U_r(\tau_{\mathrm{int}},\alpha_{\mathrm{int}})>0\).

%We first establish the boundary transfer that will be used below.
Let \(L_p>0\) be any deterministic sequence such that
\begin{equation}
 L_p\to L\in(0,\infty),
 \qquad
 L_p\ge\frac{\delta_p}{\kappa_{\mathrm{int}}}
 \quad\text{eventually}.
 \label{eq:int-boundary-weight-condition}
\end{equation}
Let $c_p:=\frac{\delta_p}{\kappa_{\mathrm{int}}}$,
and define
$$I_{p,L_p}^{\mathrm{raw}}(b)
 :=
 \|b\|_{r,p}+L_p[C_p^{\mathrm{raw}}(b)]_+.$$
At interpolation, \(R_\kappa=0\). Therefore,
Lemma~\ref{lem:AO}\emph{(iii)}--\emph{(iv)} gives
\begin{equation}
 a_p^*
 :=
 \inf_bA_p^{\mathrm{raw}}(b)
 \to
 \frac{\kappa_{\mathrm{int}}}{\delta}
 \|f_{\mathrm{int}}\|_r
 \qquad\text{almost surely},
 \label{eq:int-boundary-base-AO-value}
\end{equation}
and, for every \(\eta>0\),
\begin{equation}
 \liminf_{p\to\infty}
 \left(
  \inf_{b\in\mathcal B_{p,\eta}}A_p^{\mathrm{raw}}(b)
  -a_p^*
 \right)>0
 \qquad\text{almost surely}.
 \label{eq:int-boundary-base-AO-gap}
\end{equation}
Here and below, \(d_p\) and \(\mathcal B_{p,\eta}\) (defined respectively in \eqref{eq:dedb} and \eqref{eq:defBp}) are evaluated at interpolation, i.e., under the substitution
$ (\tau_\kappa,\alpha_\kappa,\kappa,f_\kappa)
 =
 (\tau_{\mathrm{int}},\alpha_{\mathrm{int}},
  \kappa_{\mathrm{int}},f_{\mathrm{int}})$.

Since \(L_p\ge c_p\), we have
\begin{equation}
 I_{p,L_p}^{\mathrm{raw}}(b)
 \ge c_pA_p^{\mathrm{raw}}(b),
 \label{eq:int-boundary-AO-domination}
\end{equation}
for $b\in\mathbb R^p$. Let $
 b_{p,j}^{\circ}
 :=
 \tau_{\mathrm{int}}
 \eta_{\alpha_{\mathrm{int}},r}
 \left(
  \frac{(\beta_*)_j}{\tau_{\mathrm{int}}}+h_j
 \right)$.
Lemma~\ref{lem:empirical}, \eqref{eq:int-state-mse} and
Gaussian integration by parts give
\[
 C_p(b_p^\circ)\longrightarrow0,
 \qquad
 \|b_p^\circ-\beta_*\|_{2,p}=O(1),
 \qquad
 \|b_p^\circ\|_{r,p}
 \to\|f_{\mathrm{int}}\|_r.
\]
Equation~\eqref{eq:normreplace} therefore gives
\[
 [C_p^{\mathrm{raw}}(b_p^\circ)]_+
 \longrightarrow0.
\]
Since \(L_p\) is bounded,
\[
 I_{p,L_p}^{\mathrm{raw}}(b_p^\circ)
 \to\|f_{\mathrm{int}}\|_r.
\]
On the other hand,
\eqref{eq:int-boundary-AO-domination} and
\eqref{eq:int-boundary-base-AO-value} give
\[
 \liminf_{p\to\infty}
 \inf_b I_{p,L_p}^{\mathrm{raw}}(b)
 \ge
 \lim_{p\to\infty}c_pa_p^*
 =
 \|f_{\mathrm{int}}\|_r.
\]
Consequently,
\begin{equation}
 i_p^*
 :=
 \inf_b I_{p,L_p}^{\mathrm{raw}}(b)
 \to\|f_{\mathrm{int}}\|_r
 \qquad\text{almost surely},
 \label{eq:int-boundary-weighted-AO-value}
\end{equation}
and
\[
\begin{aligned}
 &\inf_{b\in\mathcal B_{p,\eta}}
 I_{p,L_p}^{\mathrm{raw}}(b)-i_p^*\ge
 c_p\left(
  \inf_{b\in\mathcal B_{p,\eta}}A_p^{\mathrm{raw}}(b)-a_p^*
 \right)
 +c_pa_p^*-i_p^*.
\end{aligned}
\]
The last two terms differ by \(o(1)\) by
\eqref{eq:int-boundary-base-AO-value} and
\eqref{eq:int-boundary-weighted-AO-value}. Since
\(c_p\to\delta/\kappa_{\mathrm{int}}>0\),
\eqref{eq:int-boundary-base-AO-gap} gives
\begin{equation}
 \liminf_{p\to\infty}
 \left(
  \inf_{b\in\mathcal B_{p,\eta}}
  I_{p,L_p}^{\mathrm{raw}}(b)
  -i_p^*
 \right)>0
 \qquad\text{almost surely}.
 \label{eq:int-boundary-weighted-AO-gap}
\end{equation}
%An application of Lemma
%The sets \(\mathcal B_{p,\eta}\) are deterministic and closed. Indeed, the two moment maps in \(d_p\) are continuous, and
%\[
% W_2(\mu_{p,b},\mu_{p,b'})
% \le\|b-b'\|_{2,p}.
%\]
%We may therefore apply
%Lemma~\ref{lem:weighted-cgmt-transfer} %with
%\[
% \rho_p=1,\qquad
% \lambda_p=L_p,\qquad
% %\kappa_0=\kappa_{\mathrm{int}},\qquad
% S_p=\mathcal B_{p,\eta}.
%\]
%Its coefficient condition is precisely
%\eqref{eq:int-boundary-weight-condition}, while its strict-margin
%condition follows from Lemma~\ref{lem:margin}. Equations
%\eqref{eq:int-boundary-weighted-AO-value} and
%\eqref{eq:int-boundary-weighted-AO-gap} give the remaining
%hypotheses.
%
%Thus, for
Letting
\begin{equation}\label{eq:int-w-prim}
 P_{p,L_p}(w)
 :=
 \|\beta_*+w\|_{r,p}
 +\frac{L_p}{\sqrt n}
 \left\|
  \varepsilon-\frac1{\sqrt n}G_Xw
 \right\|_2,
\end{equation}
an application of Lemma~\ref{lem:weighted-cgmt-transfer}
%the value conclusion of Lemma~\ref{lem:weighted-cgmt-transfer} 
gives
\begin{equation}
 \inf_wP_{p,L_p}(w)
 \to\|f_{\mathrm{int}}\|_r
 \qquad\text{almost surely}.
 \label{eq:int-boundary-transfer-value}
\end{equation}
and %Its exclusion conclusion, applied to \(\eta=1/k\), \(k\in\mathbb N\), gives
\begin{equation}
 \sup_{\widehat w\in\argmin_wP_{p,L_p}(w)}
 d_p(\beta_*+\widehat w)
 \longrightarrow0
 \qquad\text{almost surely}.
 \label{eq:int-boundary-transfer-stability}
\end{equation}
To conclude, choose
$
 L>
 \max\left\{
  \frac{\delta}{\kappa_{\mathrm{int}}},
  C_{r,\delta}
 \right\}$.
Since \(\delta_p/\kappa_{\mathrm{int}}
\to\delta/\kappa_{\mathrm{int}}\), the constant sequence \(L_p=L\)
satisfies \eqref{eq:int-boundary-weight-condition}. Hence
\eqref{eq:int-boundary-transfer-stability} applies to the minimizers
of
\[
 b\longmapsto
 \|b\|_{r,p}+\frac L{\sqrt n}\|y-Xb\|_2.
\]
By Lemma~\ref{lem:int-exact-penalty}, these minimizers are exactly the
minimum-\(\ell_r\) interpolators. Therefore, uniformly over all such
interpolators,
\[
 d_p(\widehat\beta_r)
 \longrightarrow0
 \qquad\text{almost surely}.
\]
In particular, the second convergence in \eqref{eq:int-law}
of Remark~\ref{rem:joint-empirical-laws} holds, and
\[
\begin{aligned}
 \|\widehat\beta_r-\beta_*\|_{2,p}^2
 &\to
 \tau_{\mathrm{int}}^2
 H_r(\tau_{\mathrm{int}},\alpha_{\mathrm{int}}), %,\qquad \|\widehat\beta_r\|_{r,p}^r &\to \tau_{\mathrm{int}}^r U_r(\tau_{\mathrm{int}},\alpha_{\mathrm{int}}).
\end{aligned}
\]
Finally, \eqref{eq:int-state-mse} gives
\[
 \tau_{\mathrm{int}}^2
 H_r(\tau_{\mathrm{int}},\alpha_{\mathrm{int}})
 =
 \delta(\tau_{\mathrm{int}}^2-\sigma^2),
\]
which proves \eqref{eq:int-mse}.
\end{proof}
\section{Proofs for Section \ref{sec:mse-slope}}\label{app:pfslope}

\subsection{Auxiliary results}\label{app:aux}

\begin{lemma}[Monotonicity of the residual]
\label{lem:residual-monotonicity}
Let $1\le r<\infty$. For each finite $(n,p)$ and $\kappa\ge0$,
let $\widehat\beta_{\kappa,r}$ be any minimizer
in \eqref{eq:minprob}. If $0\leq\kappa_1<\kappa_2$, then, for any choices of minimizers,
\[
 \|y-X\widehat\beta_{\kappa_1,r}\|_2
 \leq
 \|y-X\widehat\beta_{\kappa_2,r}\|_2.
\]
%and
%\[
% J_{r,p}(\widehat\beta_{\kappa_1,r})
% \geq
% J_{r,p}(\widehat\beta_{\kappa_2,r}).
%\]
%Consequently, whenever the asymptotic residual
%\[
% \mathcal R_r(\kappa)
% :=
% \lim_{n,p\to\infty}
% \frac1{\sqrt n}
% \|y-X\widehat\beta_{\kappa,r}\|_2
%\]
%exists, the map $\kappa\mapsto\mathcal R_r(\kappa)$ is nondecreasing.
\end{lemma}

\begin{proof}
For $i=1,2$, write
\[
 L_i:=\|y-X\widehat\beta_{\kappa_i,r}\|_2,
 \qquad
 P_i:=\frac{p}{\sqrt n}
 \|\widehat\beta_{\kappa_i,r}\|_{r, p}.
\]
Optimality at $\kappa_1$ and $\kappa_2$ gives
\[
 L_1+\kappa_1P_1\leq L_2+\kappa_1P_2,
 \qquad
 L_2+\kappa_2P_2\leq L_1+\kappa_2P_1.
\]
Equivalently,
\[
 \kappa_2(P_2-P_1)
 \le L_1-L_2
 \le \kappa_1(P_2-P_1).
\]
Since \(\kappa_2>\kappa_1\), these inequalities imply
\(P_2\le P_1\), and hence
\[
 L_1-L_2\le\kappa_1(P_2-P_1)\le0.
\]
Thus \(L_1\le L_2\). Dividing by \(\sqrt n\) and passing to the limit
proves the final assertion.
\end{proof}

\begin{lemma}[Limit towards interpolation]
\label{prop:branch-attachment}
Suppose that $0<\delta<1$, $1<r<2$, $\sigma\geq0$, and
$\E B^2<\infty$.  If $\sigma=0$, assume that the fixed point equations in 
\eqref{eq:int-state-mse} admit a
solution in $(0,\infty)^2$.  Then, these equations admit a unique
solution $(\tau_{\mathrm{int}},\alpha_{\mathrm{int}})$, and
\[
 \kappa_r^\star=\kappa_{\mathrm{int}},
 \qquad
 (\tau_\kappa,\alpha_\kappa)
 \to
 (\tau_{\mathrm{int}},\alpha_{\mathrm{int}})
 \quad\text{as }\kappa\downarrow\kappa_r^\star
\]
along any non-interpolating set of solutions.  Moreover,
$R_\kappa\to0$ as $\kappa\downarrow\kappa_r^\star$.
\end{lemma}

\begin{proof}
We first prove existence when $\sigma>0$. For every
$\tau\geq\sigma$, the map
$\alpha\mapsto D_r(\tau,\alpha)$ is continuous and strictly
decreasing from $1$ to $0$. Indeed, if $z\ne0$,
$u=\eta_{\alpha,r}(z)$, and $t=\alpha|u|^{r-2}$, then
\[
 \eta'_{\alpha,r}(z)=\frac1{1+(r-1)t},
 \qquad
 \frac{d\log t}{d\log\alpha}
 =\frac{1+t}{1+(r-1)t}>0.
\]
Thus, for every $z\ne0$, the map
$\alpha\mapsto\eta'_{\alpha,r}(z)$ is continuous and strictly
decreasing from $1$ to $0$. Since, conditionally on $B$, the
variable $B/\tau+G$ has a continuous law, dominated convergence
gives that $D_r(\tau,\alpha)$ is jointly continuous in
$(\tau,\alpha)$ and satisfies
\[
 \lim_{\alpha\downarrow0}D_r(\tau,\alpha)=1,
 \qquad
 \lim_{\alpha\to\infty}D_r(\tau,\alpha)=0.
\]
Consequently, there is a unique continuous function
$\alpha(\tau)>0$ such that
\[
 D_r(\tau,\alpha(\tau))=\delta.
\]
Set
\[
 g(\tau):=
 H_r(\tau,\alpha(\tau))
 -\delta\left(1-\frac{\sigma^2}{\tau^2}\right),
 \qquad \tau\geq\sigma.
\]
We have
\[
 g(\sigma)=H_r(\sigma,\alpha(\sigma))>0.
\]
Indeed, equality would imply
$\eta_{\alpha(\sigma),r}(B/\sigma+G)=B/\sigma$
almost surely. By \eqref{eq:prox-eq}, this would give
$G=\alpha(\sigma)\psi_r(B/\sigma)$ almost surely,
which gives a contradiction.
We next determine the limit of $g$ at infinity. If
$\tau_j\to\infty$, the identity
$D_r(\tau_j,\alpha(\tau_j))=\delta$ prevents
$\alpha(\tau_j)$ from having a subsequence tending to either $0$ or
$\infty$. Every subsequential limit $\bar\alpha\in(0,\infty)$
satisfies
\[
 \E\eta'_{\bar\alpha,r}(G)=\delta.
\]
This equation has a unique solution, denoted by $\alpha_0$, and
therefore
$\alpha(\tau)\to\alpha_0$.
Lemma~\ref{lem:prox} and dominated convergence then give
\[
 H_r(\tau,\alpha(\tau))
 \to
 \E\eta_{\alpha_0,r}(G)^2.
\]
Writing $u_0=\eta_{\alpha_0,r}(G)$, Gaussian integration by parts
and \eqref{eq:prox-eq} yield
\[
 \delta=\E[Gu_0]
 =\E u_0^2+\alpha_0\E|u_0|^r.
\]
Consequently,
\[
 \lim_{\tau\to\infty}g(\tau)
 =-\alpha_0\E|u_0|^r<0.
\]
The intermediate-value theorem then gives the existence of a solution with
$\tau_{\mathrm{int}}>\sigma$. %In the noiseless case, existence of a positive solution is assumed in the statement.

We now move to proving uniqueness.
On the probability space supporting $(B,G)$, define
\begin{align}
 \mathcal C(f)
 &:=
 \left(\sigma^2+\frac1\delta\E(f-B)^2\right)^{1/2}
 -\frac1\delta\E[G(f-B)],
 \label{eq:attachment-C}\\
 \mathcal A_\kappa(f)
 &:=
 [\mathcal C(f)]_+
 +\frac{\kappa}{\delta}\|f\|_r.
 \label{eq:attachment-A}
\end{align}
Since $1<r<2$, these functionals are well defined on $L_2$.
Let $(\tau,\alpha)\in(0,\infty)^2$ satisfy
\eqref{eq:int-state-mse} and set
\[
 f:=\tau\eta_{\alpha,r}(B/\tau+G),
 \qquad
 \kappa:=\alpha U_r(\tau,\alpha)^{(r-1)/r}.
\]
Then, %\eqref{eq:int-state-mse} and Gaussian integration by parts give
\[
 \mathcal C(f)=0,
 \qquad
 \|f-B\|_2^2=\delta(\tau^2-\sigma^2).
\]
Moreover, \eqref{eq:int-state-mse} and the definition of $\kappa$,  are precisely the fixed point equations in 
\eqref{eq:se1} with
$D_r(\tau,\alpha)=\delta$. Hence
Lemma~\ref{prop:population} shows that $f$ is the unique minimizer
of
\begin{equation}\label{eq:ast}
 h\longmapsto
 \mathcal C(h)+\frac{\kappa}{\delta}\|h\|_r.
\end{equation}
The optimality calculation in the proof of that lemma also gives
\begin{equation}
 0=\nabla\mathcal C(f)
 +\frac{\kappa}{\delta}\nabla\|f\|_r.
 \label{eq:attachment-stationarity}
\end{equation}
If $\mathcal C(h)\leq0$, optimality in \eqref{eq:ast} yields
\[
 \frac{\kappa}{\delta}\|h\|_r
 \geq
 \mathcal C(h)+\frac{\kappa}{\delta}\|h\|_r
 \geq
 \frac{\kappa}{\delta}\|f\|_r,
\]
with equality only for $h=f$, thus proving  %Thus every interpolation state produces the unique minimum-$L_r$ element of $\{h\in L_2:\mathcal C(h)\leq0\}$. This feasible set is independent of the interpolation state, so every interpolation state produces the same $f$, which proves 
uniqueness of $(\tau, \alpha)$, as well as of $\kappa_{\mathrm{int}}$. 

Let $(\tau_{\mathrm{int}},\alpha_{\mathrm{int}})$ denote
this unique solution, and write $f_{\mathrm{int}}
 :=\tau_{\mathrm{int}}\eta_{\alpha_{\mathrm{int}},r}
 \left(B/\tau_{\mathrm{int}}+G\right)$.
We now study the minimizers of $\mathcal A_\kappa$, beginning
with their existence for every $\kappa>0$.
Lemma~\ref{lem:margin}, applied at
$(\tau_{\mathrm{int}},\alpha_{\mathrm{int}},
\kappa_{\mathrm{int}})$, gives
\begin{equation}
 \gamma_r(\kappa_{\mathrm{int}})<\sqrt\delta.
 \label{eq:attachment-strict-margin}
\end{equation}
Since $[\mathcal C]_+\geq\mathcal C$, the definition of
$\gamma_r$, its monotonicity, and $\E[GB]=0$ imply, for every
$\kappa\geq\kappa_{\mathrm{int}}$,
\begin{equation}
 \mathcal A_\kappa(f)
 \geq
 \frac{\sqrt\delta-\gamma_r(\kappa_{\mathrm{int}})}{\delta}
 \|f\|_2-\frac{\|B\|_2}{\sqrt\delta}.
 \label{eq:attachment-coercivity}
\end{equation}
Moreover, since both terms defining $\mathcal A_\kappa$
are nonnegative, for every $\kappa>0$,
\[
 \mathcal A_\kappa(f)
 \geq
 \min\left(1,\frac{\kappa}{\kappa_{\mathrm{int}}}\right)
 \mathcal A_{\kappa_{\mathrm{int}}}(f).
\]
Thus, for each fixed $\kappa>0$,
$\mathcal A_\kappa(f)\to\infty$ as $\|f\|_2\to\infty$.
In particular, every sequence with objective values
bounded above is bounded in $L_2$. Since $\mathcal A_\kappa$ is also proper, convex, and weakly
lower semicontinuous, it attains its minimum.

We next identify these minimizers. The previous argument
shows that $f_{\mathrm{int}}$ is the unique minimizer of
$\|f\|_r$ over
$\{f\in L_2:\mathcal C(f)\leq0\}$.
For $0<\kappa\leq\kappa_{\mathrm{int}}$, we claim that
$f_{\mathrm{int}}$ also uniquely minimizes
$\mathcal A_\kappa$. Indeed, if $\mathcal C(f)\geq0$, then
\[
 \frac{\delta}{\kappa}\mathcal A_\kappa(f)
 =
 \|f\|_r+\frac{\delta}{\kappa}\mathcal C(f)
 \geq
 \|f\|_r+\frac{\delta}{\kappa_{\mathrm{int}}}\mathcal C(f)
 \geq
 \|f_{\mathrm{int}}\|_r,
\]
where the last inequality follows from the optimality of
$f_{\mathrm{int}}$ in \eqref{eq:ast} with
$\kappa=\kappa_{\mathrm{int}}$.
If $\mathcal C(f)\leq0$, optimality of $f_{\mathrm{int}}$
in \eqref{eq:ast} at $\kappa=\kappa_{\mathrm{int}}$ gives
\[
 \frac{\kappa_{\mathrm{int}}}{\delta}\|f\|_r
 \geq
 \mathcal C(f)+\frac{\kappa_{\mathrm{int}}}{\delta}\|f\|_r
 \geq
 \frac{\kappa_{\mathrm{int}}}{\delta}\|f_{\mathrm{int}}\|_r.
\]
Since $\mathcal A_\kappa(f)=(\kappa/\delta)\|f\|_r$
in this case, the same bound follows.
Moreover, $\mathcal C(f_{\mathrm{int}})=0$, so the bound
is attained at $f_{\mathrm{int}}$.
In either case, equality forces $f=f_{\mathrm{int}}$
by uniqueness of the minimizer in \eqref{eq:ast}
at $\kappa=\kappa_{\mathrm{int}}$.

Now, let $\kappa>\kappa_{\mathrm{int}}$.
We first show that $f_{\mathrm{int}}$ cannot minimize
$\mathcal A_\kappa$. Since
$\mathcal C(f_{\mathrm{int}})=0$,
\[
 \partial[\mathcal C]_+(f_{\mathrm{int}})
 =
 \left\{
  \theta\nabla\mathcal C(f_{\mathrm{int}}):
  0\leq\theta\leq1
 \right\}.
\]
Together with \eqref{eq:attachment-stationarity},
the first-order optimality condition would therefore
require $\kappa=\theta\kappa_{\mathrm{int}}$
for some $\theta\in[0,1]$, which is impossible. Furthermore, if a minimizer $f_\kappa$ of $\mathcal A_\kappa$
satisfied $\mathcal C(f_\kappa)\leq0$, comparison with
$f_{\mathrm{int}}$ would give
$\|f_\kappa\|_r\leq\|f_{\mathrm{int}}\|_r$ and the minimality of $f_{\mathrm{int}}$ would then imply
$f_\kappa=f_{\mathrm{int}}$, a contradiction.
Consequently, every minimizer satisfies
\begin{equation}
 \mathcal C(f_\kappa)>0,
 \qquad \kappa>\kappa_{\mathrm{int}}.
 \label{eq:attachment-positive-C}
\end{equation}

We now prove that these minimizers converge to
$f_{\mathrm{int}}$ as
$\kappa\downarrow\kappa_{\mathrm{int}}$.
For any minimizer $f_\kappa$,
\[
 \mathcal A_\kappa(f_\kappa)
 \leq
 \mathcal A_\kappa(f_{\mathrm{int}})
 =
 \frac{\kappa}{\delta}\|f_{\mathrm{int}}\|_r.
\]
Together with \eqref{eq:attachment-coercivity}, this
shows that all minimizers remain uniformly bounded
in $L_2$ as $\kappa\downarrow\kappa_{\mathrm{int}}$.
Let $\kappa_j\downarrow\kappa_{\mathrm{int}}$ and choose
$f_j\in\argmin\mathcal A_{\kappa_j}$. On bounded subsets of $L_2$,
\[
 \left|
 \mathcal A_{\kappa_j}(h)
 -\mathcal A_{\kappa_{\mathrm{int}}}(h)
 \right|
 \leq
 \frac{|\kappa_j-\kappa_{\mathrm{int}}|}{\delta}\|h\|_2
 \longrightarrow0.
\]
Weak compactness, lower semicontinuity, and uniqueness of the minimizer of $\mathcal A_{\kappa_{\mathrm{int}}}$ therefore show that $f_j$ converges weakly to $f_{\mathrm{int}}$ in $L_2$.
Since $f_j$ minimizes $\mathcal A_{\kappa_j}$,
comparison with $f_{\mathrm{int}}$ and
\eqref{eq:attachment-positive-C} give
\[
 0<\mathcal C(f_j)
 \leq
 \frac{\kappa_j}{\delta}
 \left(\|f_{\mathrm{int}}\|_r-\|f_j\|_r\right).
\]
In particular, $\|f_j\|_r\leq\|f_{\mathrm{int}}\|_r$.
On the other hand, weak lower semicontinuity gives
$\|f_{\mathrm{int}}\|_r\leq\liminf_j\|f_j\|_r$.
Therefore,
\[
 \|f_j\|_r\to\|f_{\mathrm{int}}\|_r,
 \qquad
 \mathcal C(f_j)\to0.
\]
Furthermore, weak convergence implies
\[
 \E[G(f_j-B)]\to\E[G(f_{\mathrm{int}}-B)].
\]
Using \eqref{eq:attachment-C},
$\mathcal C(f_{\mathrm{int}})=0$ and
\eqref{eq:int-state-mse}, we obtain
\[
 \left(\sigma^2+\frac1\delta\|f_j-B\|_2^2\right)^{1/2}
 =\mathcal C(f_j)+\frac1\delta\E[G(f_j-B)]
 \to\tau_{\mathrm{int}}.
\]
It follows that
$$\|f_j-B\|_2\to\|f_{\mathrm{int}}-B\|_2,$$
and hence
\begin{equation}
 \|f_j-f_{\mathrm{int}}\|_2\to0.
 \label{eq:attachment-strong-convergence}
\end{equation}

We next show that these minimizers determine positive solutions
of \eqref{eq:se1} with $R_\kappa>0$.
For all sufficiently large $j$, define
\begin{equation}
 \tau_j:=
 \left(\sigma^2+\frac1\delta\E(f_j-B)^2\right)^{1/2},
 \qquad
 \alpha_j:=\kappa_j
 \left(\frac{\tau_j}{\|f_j\|_r}\right)^{r-1}.
 \label{eq:attachment-parameter-definitions}
\end{equation}
These quantities are positive because
$\tau_j\to\tau_{\mathrm{int}}>0$ and
$\|f_j\|_r\to\|f_{\mathrm{int}}\|_r>0$.
Since $\mathcal C(f_j)>0$, the first-order optimality
condition for $\mathcal A_{\kappa_j}$ is
\[
 \frac{f_j-B}{\tau_j}-G
 +\kappa_j\|f_j\|_r^{1-r}\psi_r(f_j)=0.
\]
Equivalently,
\[
 \frac B{\tau_j}+G
 =\frac{f_j}{\tau_j}
 +\alpha_j\psi_r\!\left(\frac{f_j}{\tau_j}\right).
\]
By \eqref{eq:prox-eq},
\[
 f_j=\tau_j\eta_{\alpha_j,r}(B/\tau_j+G).
\]
Substituting this expression into
\eqref{eq:attachment-parameter-definitions} gives
\[
 \tau_j^2
 =\sigma^2+\frac{\tau_j^2}{\delta}H_r(\tau_j,\alpha_j),
 \qquad
 \kappa_j=\alpha_jU_r(\tau_j,\alpha_j)^{(r-1)/r}.
\]
Thus $(\tau_j,\alpha_j)$ satisfies
\eqref{eq:se1}. Gaussian integration by
parts also gives
$\E[G(f_j-B)]=\tau_jD_r(\tau_j,\alpha_j)$, and hence
\[
 R_{\kappa_j}
 =\tau_j\left(1-\frac{D_r(\tau_j,\alpha_j)}{\delta}\right)
 =\mathcal C(f_j)>0.
\]
Finally, \eqref{eq:attachment-parameter-definitions} and
the definition of $\kappa_{\mathrm{int}}$ imply
\[
 \tau_j\to\tau_{\mathrm{int}},
 \qquad
 \alpha_j\to
 \kappa_{\mathrm{int}}
 \left(\frac{\tau_{\mathrm{int}}}
 {\|f_{\mathrm{int}}\|_r}\right)^{r-1}
 =\alpha_{\mathrm{int}},
 \qquad
 R_{\kappa_j}\to0.
\]
This argument applies to every sequence
$\kappa_j\downarrow\kappa_{\mathrm{int}}$ and every choice
of minimizers $f_j$. In particular, the fixed point equations in
\eqref{eq:se1} admit a positive solution
with $R_\kappa>0$ for every $\kappa>\kappa_{\mathrm{int}}$
sufficiently close to $\kappa_{\mathrm{int}}$.

We now show that the same convergence holds for every
positive solution $(\tau_\kappa,\alpha_\kappa)$ of
\eqref{eq:se1} with $R_\kappa>0$.
Set
\[
 f_\kappa:=\tau_\kappa\eta_{\alpha_\kappa,r}
 \left(B/\tau_\kappa+G\right).
\]
Gaussian integration by parts and \eqref{eq:se1}
give $\mathcal C(f_\kappa)=R_\kappa>0$.
In particular, $D_r(\tau_\kappa,\alpha_\kappa)<\delta$,
so Lemma~\ref{prop:population} applies and yields,
for every $f\in L_2$,
\[
 \mathcal A_\kappa(f)
 \geq\mathcal C(f)+\frac\kappa\delta\|f\|_r
 \geq\mathcal C(f_\kappa)+\frac\kappa\delta\|f_\kappa\|_r
 =\mathcal A_\kappa(f_\kappa).
\]
Thus $f_\kappa$ minimizes $\mathcal A_\kappa$.
The convergence proved for arbitrary sequences of minimizers
therefore applies to these random variables as well.
Moreover, \eqref{eq:se1} expresses
$\tau_\kappa$ and $\alpha_\kappa$ in terms of $f_\kappa$
by the same formulas as
\eqref{eq:attachment-parameter-definitions}.
Consequently, for every such choice of solutions,
\[
 (\tau_\kappa,\alpha_\kappa)
 \to(\tau_{\mathrm{int}},\alpha_{\mathrm{int}}),
 \qquad R_\kappa\to0,
 \qquad \kappa\downarrow\kappa_{\mathrm{int}}.
\]

It remains to prove that $\kappa_{\mathrm{int}}$ equals
$\kappa_r^\star$ as defined in \eqref{eq:kappastar}.
Set $L_p:=\delta_p/\kappa_{\mathrm{int}}$.
Since $L_p\to\delta/\kappa_{\mathrm{int}}>0$, this choice
satisfies \eqref{eq:int-boundary-weight-condition}.
Recall the definition of
 $P_{p,L_p}(w)$ in \eqref{eq:int-w-prim}. Then, \eqref{eq:int-boundary-transfer-value}-\eqref{eq:int-boundary-transfer-stability} give,
almost surely,
\[
 \inf_wP_{p,L_p}(w)\to\|f_{\mathrm{int}}\|_r,
 \qquad
 \sup_{\widehat w\in\argmin_wP_{p,L_p}(w)}
 \left|\|\beta_*+\widehat w\|_{r,p}
 -\|f_{\mathrm{int}}\|_r\right|\to0.
\]
Note that
\[
 P_{p,L_p}(b-\beta_*)
 =\frac{\delta_p}{\kappa_{\mathrm{int}}\sqrt n}
 \left(\|y-Xb\|_2
 +\frac{p\kappa_{\mathrm{int}}}{\sqrt n}\|b\|_{r, p}\right).
\]
Thus $\widehat w$ minimizes $P_{p,L_p}$ if and only if
$\widehat\beta=\beta_*+\widehat w$ minimizes
\eqref{eq:minprob} at $\kappa=\kappa_{\mathrm{int}}$.
For every such $\widehat\beta$,
\[
 \frac1{\sqrt n}\|y-X\widehat\beta\|_2
 =\frac1{L_p}
 \left(\inf_wP_{p,L_p}(w)
 -\|\widehat\beta\|_{r,p}\right) \to0
 \qquad\text{almost surely}.
\]
Lemma~\ref{lem:residual-monotonicity} then gives,
for every fixed $0\leq\kappa\leq\kappa_{\mathrm{int}}$
and every choice of minimizer in \eqref{eq:minprob},
\[
 \frac1{\sqrt n}\|y-X\widehat\beta_{\kappa,r}\|_2
 \to0
 \qquad\text{almost surely}.
\]

Conversely, fix any $\kappa>\kappa_{\mathrm{int}}$.
Choose $\widetilde\kappa\in(\kappa_{\mathrm{int}},\kappa)$
close enough to $\kappa_{\mathrm{int}}$ that
\eqref{eq:se1} admits a positive solution
with $R_{\widetilde\kappa}>0$, as established above.
Lemma~\ref{thm:main-nonint}, specifically
\eqref{eq:resclaim}, then gives
\[
 \frac1{\sqrt n}
 \|y-X\widehat\beta_{\widetilde\kappa,r}\|_2
 \to R_{\widetilde\kappa}>0
 \qquad\text{almost surely}.
\]
By Lemma~\ref{lem:residual-monotonicity},
\[
 \liminf_{p\to\infty}
 \frac1{\sqrt n}\|y-X\widehat\beta_{\kappa,r}\|_2
 \geq R_{\widetilde\kappa}>0
 \qquad\text{almost surely}.
\]
Therefore $\kappa_r^\star=\kappa_{\mathrm{int}}$.
Together with the convergence of
$(\tau_\kappa,\alpha_\kappa)$ and $R_\kappa$ proved above,
this completes the proof.
\end{proof}

\begin{lemma}[Derivative at the interpolation threshold]
\label{prop:mse-slope-formula}
Suppose that $\E B^2<\infty$, $\delta\in(0,1)$, $\sigma\geq0$,
and $1<r<2$. If $\sigma=0$, assume in addition that
$\Pp(B\neq0)>0$ and that \eqref{eq:int-state-mse} admits a solution
$(\tau_{\mathrm{int}},\alpha_{\mathrm{int}})\in(0,\infty)^2$.
Then the limit in \eqref{eq:slope-definition} exists, and it is given by
\begin{equation}
 S_r
 =\frac{2\delta^2\partial_\alpha H_r}
 {\left(\partial_\tau H_r
          -\dfrac{2\delta\sigma^2}{\tau_{\mathrm{int}}^3}\right)
        \partial_\alpha D_r
       -(\partial_\alpha H_r)(\partial_\tau D_r)}.
 \label{eq:mse-slope-formula}
\end{equation}
Here $H_r,D_r$, and all their partial derivatives are evaluated at
$(\tau_{\mathrm{int}},\alpha_{\mathrm{int}})$.
\end{lemma}

\begin{proof}
By Lemma~\ref{prop:branch-attachment},
\[
 \kappa_r^\star=\kappa_{\mathrm{int}},
 \qquad
 (\tau_\kappa,\alpha_\kappa)
 \to
 (\tau_{\mathrm{int}},\alpha_{\mathrm{int}})
 \quad\text{as }\kappa\downarrow\kappa_r^\star.
\]
We can therefore compute \(S_r\) by differentiating the
system of fixed point equations \eqref{eq:se1} at
\((\tau_{\mathrm{int}},\alpha_{\mathrm{int}})\). For $(\tau,\alpha)$
in the neighborhood of $(\tau_{\mathrm{int}},\alpha_{\mathrm{int}})$, define
\begin{equation}
 v:=\eta_{\alpha,r}\left(\frac B\tau+G\right),
 \qquad
 \omega
 :=|v|^{r-2}
 \eta'_{\alpha,r}\left(\frac B\tau+G\right).
 \label{eq:slope-proof-local-notation}
\end{equation}
At $v=0$, the second expression is understood by continuity as
$\{\alpha(r-1)\}^{-1}$.  Differentiating the proximal equation
\eqref{eq:prox-eq} gives
\begin{equation}
 \partial_\alpha v
 =-\sign(v)|v|^{r-1}
   \eta'_{\alpha,r}\left(\frac B\tau+G\right),
 \qquad
 \partial_\tau v
 =-\frac B{\tau^2}
   \eta'_{\alpha,r}\left(\frac B\tau+G\right),
 \label{eq:slope-prox-derivatives}
\end{equation}
and
\begin{equation}
 1-\eta'_{\alpha,r}\left(\frac B\tau+G\right)
 =\alpha(r-1)\omega.
 \label{eq:slope-prox-complement}
\end{equation}
Since $0<\omega\leq\{\alpha(r-1)\}^{-1}$ and
$|v|\leq|B/\tau+G|$, differentiation under the expectations defining
$H_r$ and $U_r$ is justified by dominated convergence and we have
\begin{align}
 \partial_\alpha H_r(\tau,\alpha)
 &=-2\left(
    \E[v^2\omega]
    -\E\left[\frac B\tau v\omega\right]
   \right),
 \label{eq:slope-H-alpha-direct}\\
 \partial_\tau H_r(\tau,\alpha)
 &=\frac{2\alpha(r-1)}\tau
   \left(
    \E\left[\frac B\tau v\omega\right]
    -\E\left[\left(\frac B\tau\right)^2\omega\right]
   \right),
 \label{eq:slope-H-tau-direct}\\
 \partial_\tau U_r(\tau,\alpha)
 &=-\frac r\tau
   \E\left[\frac B\tau v\omega\right],
 \label{eq:slope-U-tau-direct}\\
 \partial_\alpha U_r(\tau,\alpha)
 &=-r\E[|v|^r\omega].
 \label{eq:slope-U-alpha-direct}
\end{align}
Equation \eqref{eq:slope-prox-complement} also gives
\begin{equation}
 \alpha(r-1)\E[|v|^r\omega]
 =U_r(\tau,\alpha)-\E[v^2\omega].
 \label{eq:slope-U-identity-direct}
\end{equation}
If $(\tau, \alpha)$ solves \eqref{eq:se1}, then
\begin{equation}
 1-\frac1\delta H_r(\tau,\alpha)
 =\frac{\sigma^2}{\tau^2}.
 \label{eq:slope-state-identity}
\end{equation}
Substituting
\eqref{eq:slope-H-alpha-direct}-\eqref{eq:slope-U-identity-direct}
into the Jacobian of \eqref{eq:se1} gives
\begin{align}
 &\det
 \frac{\partial\left(
  \tau^2-\sigma^2-\dfrac{\tau^2}{\delta}H_r(\tau,\alpha),
  \ \alpha U_r(\tau,\alpha)^{(r-1)/r}
 \right)}
 {\partial(\tau,\alpha)}
 \notag\\
 &\quad=
 \frac{2\tau U_r(\tau,\alpha)^{-1/r}}\delta
 \Bigg[
  \frac{\delta\sigma^2}{\tau^2}\E[v^2\omega]
  +\alpha(r-1)
  \Bigg(
   \E[v^2\omega]
   \E\left[\left(\frac B\tau\right)^2\omega\right]
   -\E\left[\frac B\tau v\omega\right]^2
  \Bigg)
 \Bigg].
 \label{eq:slope-jacobian}
\end{align}
Note that the RHS is strictly positive by Cauchy-Schwartz when $\sigma>0$.  If $\sigma=0$, equality would imply
that $v$ is a deterministic multiple of $B/\tau$ almost surely.  This is
impossible when $\Pp(B\neq0)>0$, because, conditionally on $B=b$, the
random variable $\eta_{\alpha,r}(b/\tau+G)$ is nondegenerate for
$P_B$-almost every $b\neq0$.  This proves that the Jacobian in
\eqref{eq:slope-jacobian} is strictly positive.

The implicit-function theorem therefore gives a unique \(C^1\) local
solution curve through
\((\tau_{\mathrm{int}},\alpha_{\mathrm{int}})\).
By Lemma~\ref{prop:branch-attachment}, for
$\kappa>\kappa_r^\star$ sufficiently close to $\kappa_r^\star$,
this local solution curve agrees with the parameters
$(\tau_\kappa,\alpha_\kappa)$ corresponding to positive
training residual $R_\kappa$. Thus, we have
\begin{equation}
 \frac{d\tau}{d\kappa}
 =
 \frac{
  \tau_{\mathrm{int}}^2
  \partial_\alpha H_r
  (\tau_{\mathrm{int}},\alpha_{\mathrm{int}})
 }{\delta}
 \left[
  \left.
  \det
  \frac{\partial\left(
   \tau^2-\sigma^2-\dfrac{\tau^2}{\delta}H_r(\tau,\alpha),
   \ \alpha U_r(\tau,\alpha)^{(r-1)/r}
  \right)}
  {\partial(\tau,\alpha)}
  \right|_{(\tau,\alpha)
  =(\tau_{\mathrm{int}},\alpha_{\mathrm{int}})}
 \right]^{-1}.
 \label{eq:slope-tau-kappa}
\end{equation}
For the remainder of the proof, all quantities are evaluated at
$\tau=\tau_{\mathrm{int}}$ and $\alpha=\alpha_{\mathrm{int}}$. Gaussian integration by parts gives $D_r=\E[Gv]$.
Differentiating this identity using
\eqref{eq:slope-prox-derivatives} and the proximal equation yields
\begin{align*}
 \partial_\alpha D_r
 &=-\E[Gv\omega]
 =\frac12\partial_\alpha H_r-\alpha\E[|v|^r\omega],\\
 \partial_\tau D_r
 &=-\frac1\tau
   \E\left[G\frac B\tau
   \eta'_{\alpha,r}\left(\frac B\tau+G\right)\right]=\frac12\partial_\tau H_r-\frac{\delta\sigma^2}{\tau^3}
   +\frac\alpha\tau
    \left(\E|v|^r-\E\left[\frac B\tau v\omega\right]\right).
\end{align*}
For the second identity, we also used
\eqref{eq:slope-prox-complement} and
\[
 \E\left[\frac B\tau v\right]
 -\E\left[\left(\frac B\tau\right)^2\right]
 +\alpha\E|v|^r
 =D_r-H_r=\frac{\delta\sigma^2}{\tau^2},
\]
where the first equality follows from
\eqref{eq:prox-eq}, the definition of $H_r$ in \eqref{eq:Hdef},
and $D_r=\E[Gv]$, while the second follows from \eqref{eq:int-state-mse}.
The bounds $|\eta'_{\alpha,r}|\leq1$,
$\omega\leq(\alpha(r-1))^{-1}$, and
$|v|\leq|B/\tau+G|$ justify differentiation under the expectations
and continuity of these derivatives near the interpolation point.

Substituting these identities and
\eqref{eq:slope-H-alpha-direct},
\eqref{eq:slope-H-tau-direct}, and
\eqref{eq:slope-U-identity-direct} give
\begin{align}
 &\left(\partial_\tau H_r-\frac{2\delta\sigma^2}{\tau^3}\right)
    \partial_\alpha D_r
   -(\partial_\alpha H_r)(\partial_\tau D_r)
 \notag\\
 &\quad=\frac{2\alpha}{\tau}
 \Bigg[
  \left(\E\left[\left(\frac B\tau-v\right)v\omega\right]\right)^2
  +\E[|v|^r\omega]
   \left(
    \frac{\delta\sigma^2}{\tau^2}
    +\alpha(r-1)
      \E\left[\left(\frac B\tau-v\right)^2\omega\right]
   \right)
 \Bigg]>0.
 \label{eq:slope-residual-denominator-positive}
\end{align}
%Indeed, $\E[|v|^r\omega]>0$ since $v\neq0$ almost surely. Moreover, $\E[(B/\tau-v)^2\omega]=0$ would imply $v=B/\tau$ almost surely, so the proximal equation would give $G=\alpha\sign(B/\tau)|B/\tau|^{r-1}$ almost surely. This contradicts the independent, nondegenerate Gaussian law of $G$ conditional on $B$. Thus the second expectation is also strictly positive, and all quantities above are finite under $\E B^2<\infty$.
Finally, we differentiate
\[
 \overline{\mathsf{MSE}}_r(\kappa)
 =\delta(\tau_\kappa^2-\sigma^2),
 \qquad
 R_\kappa
 =\tau_\kappa
   \left(1-\frac{D_r(\tau_\kappa,\alpha_\kappa)}{\delta}\right).
\]
Since $D_r=\delta$ at interpolation, 
$\lim_{\kappa\downarrow\kappa_r^\star}dR_\kappa/d\kappa$
equals the LHS of
\eqref{eq:slope-residual-denominator-positive} multiplied by
$\tau^3/\delta^2$ and by the reciprocal of the Jacobian in
\eqref{eq:slope-jacobian}.
This limit is therefore finite and strictly positive.
Combining this expression with \eqref{eq:slope-tau-kappa} and
$d\overline{\mathsf{MSE}}_r/d\kappa
=2\delta\tau_\kappa\,d\tau_\kappa/d\kappa$
proves \eqref{eq:mse-slope-formula}.
\end{proof}

\subsection{Proof of Theorem~\ref{thm:sparse-decreasing-slope}}
\label{app:proof-sparse-decreasing-slope}

\begin{proof}
Let $s=1/(r-1)$ and rescale the interpolation parameters as
\[
 t=\sqrt\delta\,\tau,\qquad a=\delta^{r-1}\alpha.
\]
We first let $\delta\downarrow0$ with $\lambda>0$ and $\sigma$
fixed. All convergence statements below are uniform over
$r\in[r_-,r_+]$ and locally uniform in positive $t,a$.
The notation $o_{C^k}(1)$ means that the remainder and its
partial derivatives of total order at most $k$ tend uniformly
to zero; before evaluating at interpolation these are derivatives
in $(t,a,r)$, and afterwards they are derivatives in $r$.

Define the rescaled proximal function
\[
 v_\delta(z):=\delta^{-1}\eta_{a\delta^{1-r},r}(z),
\]
suppressing its dependence on $a,r$.
By \eqref{eq:prox-eq}, it satisfies
\[
 z=\delta v_\delta(z)
   +a\sign(v_\delta(z))|v_\delta(z)|^{r-1}.
\]
Consequently, $|v_\delta(z)|\leq a^{-s}|z|^s$,
$\delta|v_\delta(z)|\leq|z|$, and
\begin{align}
 \left|v_\delta(z)-a^{-s}z|z|^{s-1}\right|
 &\leq C|z|^s\min\{1,\delta|z|^{s-1}\},
 \label{eq:sparse-uniform-prox-bound}\\
 \delta v_\delta(z)^2
 &\leq C|z|^{s+1}\min\{1,\delta|z|^{s-1}\}.
 \label{eq:sparse-quadratic-cutoff}
\end{align}
Implicit differentiation gives
\begin{align}
 \partial_zv_\delta
 &=\frac{|v_\delta|^{2-r}}
         {a(r-1)+\delta|v_\delta|^{2-r}},\,\,
 \partial_av_\delta
 &=-\frac{v_\delta}
          {a(r-1)+\delta|v_\delta|^{2-r}},\,\,
 \partial_rv_\delta
 &=-\frac{av_\delta\log|v_\delta|}
          {a(r-1)+\delta|v_\delta|^{2-r}}.
 \label{eq:sparse-v-derivatives}
\end{align}
We now explain why these approximations may be differentiated twice
under the expectations used below. Differentiating
\eqref{eq:sparse-v-derivatives} once more gives, for $z\neq0$,
\[
 |\partial_zv_\delta|\leq C|z|^{s-1},
 \qquad
 |\partial_{zz}v_\delta|\leq C|z|^{s-2}.
\]
Derivatives in $a$ preserve the polynomial growth, while each
derivative in $r$ adds at most one factor
$1+|\log|z||$.
To identify the expectations being differentiated, expand
the definition of $H_r$ in \eqref{eq:Hdef}:
\begin{align*}
 H_r(\tau,\alpha)
 &=\E\left[
   \left(\delta v_\delta(B/\tau+G)-B/\tau\right)^2
   \right]=\frac{\delta}{t^2}
   -2\delta\,\E\left[
     \frac B\tau v_\delta(B/\tau+G)
   \right]
   +\delta^2\,\E\left[v_\delta(B/\tau+G)^2\right],
\end{align*}
where we used $\E B^2=1$ and $\tau=t/\sqrt\delta$.
Similarly,
\begin{align*}
 D_r(\tau,\alpha)
 &=\delta\,\E\left[
   \partial_zv_\delta(B/\tau+G)
   \right]=\delta\,\E\left[Gv_\delta(B/\tau+G)\right].
\end{align*}
On $\{\xi_\rho=1\}$, we have $B/\tau=W/(t\sqrt\lambda)$.
Consider first the two integrands
\[
 \frac{W}{t\sqrt\lambda}\,
 v_\delta\left(\frac{W}{t\sqrt\lambda}+G\right),
 \qquad
 Gv_\delta\left(\frac{W}{t\sqrt\lambda}+G\right),
\]
which appear in the second term of the expansion of $H_r$
and in the expression for $D_r$, respectively.
Using the derivative bounds above and integrating over $G$,
the conditional expectations given $W$ of their absolute
values, and of the absolute values of all their first and
second partial derivatives in $(t,a,r)$, are bounded by
\[
 C\left(1+|W|^{s+1}
                 (1+\log^2(2+|W|))\right).
\]
The preceding bounds, the moment assumption, and
\eqref{eq:sparse-quadratic-cutoff} justify differentiation
under the expectations and the uniform $C^2(t,a,r)$
expansions below, by dominated convergence with truncation
near $z=0$.

%This bound is integrable uniformly in $r$. Furthermore, the possible local singularity $|z|^{s-2}$ is uniformly integrable near zero under the conditional distribution of $W/(t\sqrt\lambda)+G$ given $W$. For the quadratic terms, implicit differentiation and \eqref{eq:sparse-quadratic-cutoff} give the same integrable bound for $\delta v_\delta^2$ and its first two parameter derivatives, all of which tend pointwise to zero. On $\{\xi_\rho=0\}$, all required Gaussian moments are finite. Dominated convergence, with truncation near $z=0$ when necessary, therefore gives the uniform $C^2$ limits used below.

Using \eqref{eq:Hdef}, conditioning on $\xi_\rho$, and applying
Gaussian integration by parts to $D_r$, we obtain
\begin{align}
 \frac{H_r(t/\sqrt\delta,a\delta^{1-r})}{\delta}
 &=\frac1{t^2}
   +\delta a^{-2s}\E|G|^{2s}
 -\frac{2\delta\sqrt\lambda}{t a^s}
   \E\left[
    W\left(\frac W{t\sqrt\lambda}+G\right)
     \left|\frac W{t\sqrt\lambda}+G\right|^{s-1}
   \right]
   +o_{C^2}(\delta),
 \notag\\ \frac{D_r(t/\sqrt\delta,a\delta^{1-r})}{\delta}
 &=a^{-s}\E|G|^{s+1}+o_{C^2}(1),
 \label{eq:sparse-rescaled-system-limit}
\end{align}
where $o_{C^2}(\delta)$ denotes $\delta$ times an
$o_{C^2}(1)$ remainder.
%For the first expansion, the quadratic contribution from
%$\{\xi_\rho=1\}$ is negligible by
%\eqref{eq:sparse-quadratic-cutoff}; the cross term is contributed
%entirely by this event.
%For the second expansion, use the exact identity
%\[
% \frac{D_r(t/\sqrt\delta,a\delta^{1-r})}{\delta}
% =\E\left[Gv_\delta\left(\frac B\tau+G\right)\right].
%\]
%
In the variables $(t,a)$, 
\eqref{eq:int-state-mse} corresponds to
\begin{equation}\label{eq:resc-inter}
 t^2\left(1-
 \frac{H_r(t/\sqrt\delta,a\delta^{1-r})}{\delta}\right)
 -\delta\sigma^2=0,
 \qquad
 \frac{D_r(t/\sqrt\delta,a\delta^{1-r})}{\delta}-1=0.
\end{equation}
By \eqref{eq:sparse-rescaled-system-limit}, as
$\delta\downarrow0$ with $\lambda,\sigma$ fixed, the two
left-hand sides converge to $t^2-1$ and
$a^{-s}\E|G|^{s+1}-1$, respectively.
Then, the equations obtained by setting these limits equal to zero are
\[
 t^2-1=0,\qquad a^{-s}\E|G|^{s+1}-1=0,
\]
with the corresponding Jacobian in $(t, a)$ given by
\begin{equation}
 \begin{pmatrix}
 2&0\\
 0&-s/a
 \end{pmatrix},
 \label{eq:sparse-limiting-jacobian}
\end{equation}
which is nonsingular for $r\in[r_-,r_+]$.
The implicit-function theorem thus gives
a positive solution, continuously differentiable in $r$, with
\begin{equation}
 t=1+o_{C^1}(1),\qquad
 a^s=\E|G|^{s+1}+o_{C^1}(1).
 \label{eq:sparse-boundary-limit}
\end{equation}
By Lemma~\ref{prop:branch-attachment}, this solution corresponds to the interpolation endpoint, also when $\sigma=0$.
Thus, we can apply Lemma~\ref{prop:mse-slope-formula} and obtain
\begin{equation}
 S_r=
 \frac{2\delta^2\partial_\alpha H_r}
 {\left(\partial_\tau H_r-\dfrac{2\delta\sigma^2}{\tau^3}\right)
       \partial_\alpha D_r
  -(\partial_\alpha H_r)(\partial_\tau D_r)},
 \label{eq:slope-moment-formula}
\end{equation}
where all quantities are evaluated at interpolation.

To evaluate the ratio, we differentiate
\eqref{eq:sparse-rescaled-system-limit} using
$\partial_\tau=\sqrt\delta\,\partial_t$ and
$\partial_\alpha=\delta^{r-1}\partial_a$.
Together with \eqref{eq:sparse-boundary-limit}, this gives
\[
 \delta^{-3/2}\partial_\tau H_r=-2+o_{C^1}(1),\qquad
 \delta^{-r}\partial_\alpha D_r=-s/a+o_{C^1}(1),\qquad
 \delta^{-3/2}\partial_\tau D_r=o_{C^1}(1),
\]
and
\begin{align*}
 \delta^{-(r+1)}\partial_\alpha H_r
 =2sa^{-s-1}\Bigg(&
 \sqrt\lambda\,
 \E\left[
 W\left(\frac W{\sqrt\lambda}+G\right)
  \left|\frac W{\sqrt\lambda}+G\right|^{s-1}
 \right]-a^{-s}\E|G|^{2s}\Bigg)+o_{C^1}(1).
\end{align*}
Since $2\delta\sigma^2/\tau^3=O(\delta^{5/2})$, substitution
into \eqref{eq:slope-moment-formula} yields
\begin{align}
 \delta^{-3/2}S_r
 &=\frac{2\sqrt\lambda}{\E|G|^{s+1}}
 \E\left[
 W\left(\frac W{\sqrt\lambda}+G\right)
  \left|\frac W{\sqrt\lambda}+G\right|^{s-1}
 \right]
 -\frac{2\E|G|^{2s}}{(\E|G|^{s+1})^2}
       +o_{C^1}(1).
 \label{eq:sparse-fixed-lambda-limit}
\end{align}
The $C^2$ estimates above justify the differentiated expansions
and show that $r\mapsto S_r$ is continuously differentiable.

We finally choose $\lambda$ sufficiently small. As $\lambda\to 0$, we have
\begin{align*}
 &\lambda^{(s-1)/2}\sqrt\lambda\,
 \E\left[
 W\left(\frac W{\sqrt\lambda}+G\right)
  \left|\frac W{\sqrt\lambda}+G\right|^{s-1}
 \right]=\E|W|^{s+1}+o_{C^1}(1).
\end{align*}
Since $\E|W|^{s+1}\geq1$ and $s>1$, combining the preceding
estimate with \eqref{eq:sparse-fixed-lambda-limit} gives
\begin{equation}
 \lim_{\delta\downarrow0}\delta^{-3/2}S_r
 =\frac{2\E|W|^{s+1}}{\E|G|^{s+1}}\,
   \lambda^{-(s-1)/2}(1+o_{C^1}(1)),
 \label{eq:sparse-slope-asymptotic}
\end{equation}
for all sufficiently small $\lambda>0$. Here the limit in
$\delta$ is taken at fixed $\lambda$, and $o_{C^1}(1)$ denotes
a remainder whose value and first derivative in $r$ converge
uniformly to zero on $[r_-,r_+]$ as $\lambda\downarrow0$.%The Gaussian term in \eqref{eq:sparse-fixed-lambda-limit} is negligible here, including after differentiation, since $\lambda^{(s-1)/2}(1+|\log\lambda|)\to0$ uniformly in $r$.
The coefficient in \eqref{eq:sparse-slope-asymptotic} is positive,
bounded away from zero, and has bounded logarithmic derivative.
As $ds/dr=-s^2$, it follows that, for small $\lambda$,
\[
 \frac d{dr}\log\left(
  \lim_{\delta\downarrow0}\delta^{-3/2}S_r
 \right)
 =-\frac{s^2}{2}\log(1/\lambda)+O(1),
\]
uniformly in $r$.
We may therefore choose $\lambda_0:=\lambda_0(r_-,r_+,W)>0$
so that the limit in \eqref{eq:sparse-fixed-lambda-limit}
is strictly positive and has strictly negative derivative
throughout $[r_-,r_+]$ for every $0<\lambda<\lambda_0$.
%For each such fixed $\lambda$, both inequalities have uniform strict margins on this compact interval.
The $C^1$ convergence in \eqref{eq:sparse-fixed-lambda-limit}
then gives
$\delta_0=\delta_0(r_-,r_+,W,\lambda,\sigma)>0$
such that $S_r>0$ and $dS_r/dr<0$ whenever
$0<\delta<\delta_0$ and $r\in[r_-,r_+]$, which concludes the argument.
\end{proof}

\subsection{Proof of Theorem~\ref{thm:r1-critical-infinite-slope}}
\label{app:proof-r1-critical-infinite-slope}

\begin{proof}
Define
\begin{equation}
 h(\alpha):=\rho(1+\alpha^2)
 +(1-\rho)\E(|G|-\alpha)_+^2,
 \qquad \alpha\geq0,
 \label{eq:r1-critical-h}
\end{equation}
so that $\delta_c=\min_{\alpha\geq0}h(\alpha)$ by
\eqref{eq:r1-critical-delta}. Computing derivatives gives
\[
 h''(\alpha)=2(\rho+(1-\rho)\Pp(|G|>\alpha))>0,
 \qquad
 h(\alpha)-\frac\alpha2h'(\alpha)
 =\rho+(1-\rho)\Pp(|G|>\alpha).
\]
Since $h'(0)<0$ and $h(\alpha)\to\infty$, its minimizer
$\alpha_c$ is unique and positive. In particular,
\[
 h(\alpha_c)=\delta_c,\qquad h'(\alpha_c)=0,\qquad
 h''(\alpha_c)=2\delta_c,
 \qquad \rho+(1-\rho)\Pp(|G|>\alpha_c)=\delta_c.
\]

Define
\[
 e(x,\alpha):=2\int_x^\infty z\,\Pp(|z+G|\leq\alpha)\,dz,
 \qquad x\geq0.
\]
Conditioning on the support of $B$ gives
\begin{align*}
 H_1(\tau,\alpha)
 &=h(\alpha)-\rho\E e\left(\frac{|W|}{\tau\sqrt\rho},\alpha\right),\\
 D_1(\tau,\alpha)
 &=\rho+(1-\rho)\Pp(|G|>\alpha)
   -\rho\Pp\left(\left|\frac{W}{\tau\sqrt\rho}+G\right|
                  \leq\alpha\right),
\end{align*}
%The first identity follows by integrating the derivative
%$2x\Pp(|x+G|\leq\alpha)$ of the scalar risk, whose limit
%as $x\to\infty$ is $1+\alpha^2$.
which implies that %These formulas also show that
$\partial_\tau H_1<0$, $\partial_\tau D_1<0$ and
$\partial_\alpha D_1<0$.

To identify the limit as $\tau\downarrow0$, note that
\eqref{eq:r1-critical-small-ball} implies $\Pp(W=0)=0$,
and hence $|W|/(\tau\sqrt\rho)\to\infty$ almost surely.
Since $0\leq e(x,\alpha)\leq e(0,\alpha)<\infty$ and
$e(x,\alpha)\to0$ as $x\to\infty$, dominated convergence gives
$H_1(\tau,\alpha)\to h(\alpha)$. For each $\alpha>\alpha_c$, the uniqueness of the minimizer of $h$ gives
\[
 \lim_{\tau\downarrow0}H_1(\tau,\alpha)
 =h(\alpha)>h(\alpha_c)=\delta_c.
\]
Furthermore, the limit as $\tau\to\infty$ is
$\E(|G|-\alpha)_+^2<\delta_c$. Thus, there exists a unique smooth solution
of $H_1(\tau,\alpha)=\delta_c$, which satisfies
$D_1(\tau,\alpha)<\delta_c$. For $r=1$ and $\sigma=0$, the fixed-point equations
\eqref{eq:se1} reduce, for $\tau_\kappa>0$, to
\[
 H_1(\tau_\kappa,\alpha_\kappa)=\delta_c,
 \qquad \alpha_\kappa=\kappa.
\]
Thus, for $\kappa=\alpha>\alpha_c$, the pair
$(\tau_\kappa,\alpha_\kappa)=(\tau(\alpha),\alpha)$
solves these equations. Moreover,
$D_1(\tau(\alpha),\alpha)<\delta_c$ ensures that
$R_\kappa>0$, as required by Lemma~\ref{thm:main-nonint}.
Applying that lemma gives the limiting training error
and MSE:
\[
 R_\alpha=\tau(\alpha)
 \left(1-\frac{D_1(\tau(\alpha),\alpha)}{\delta_c}\right)>0,
 \qquad
 \overline{\mathsf{MSE}}_1(\alpha)=\delta_c\tau(\alpha)^2.
\]
For every fixed $\tau>0$, $H_1(\tau,\alpha_c)<\delta_c$.
By continuity, this inequality holds for $\alpha$ close
to $\alpha_c$, so the solution $\tau(\alpha)$ must lie below
that fixed $\tau$. Thus $\tau(\alpha)\to0$ as
$\alpha\downarrow\alpha_c$, and $0<R_\alpha\leq\tau(\alpha)$
implies $R_\alpha\to0$. For any fixed $\kappa\leq\alpha_c$, the training error is bounded by  Lemma~\ref{lem:residual-monotonicity} and hence an application of Lemma~\ref{thm:main-nonint} gives that the training error at
$\kappa$ converges almost surely to zero, which in turn implies $\kappa_1^\star=\alpha_c$.

We next estimate the partial derivatives of $H_1(\tau,\alpha)$
and $D_1(\tau,\alpha)$ as $\tau\downarrow0$, evaluating these estimates at $\tau=\tau(\alpha)$, where
$H_1(\tau(\alpha),\alpha)=\delta_c$. In the calculations below, $\tau\downarrow0$ and
$\alpha$ ranges over a fixed compact neighborhood of $\alpha_c$
contained in $(0,\infty)$. %All remainders are uniform in $\alpha$.
Gaussian tail estimates give, for $x\geq0$,
\begin{align*}
 e(x,\alpha)
 +\frac{e_\alpha(x,\alpha)^2+x^2e_x(x,\alpha)^2}
        {e(x,\alpha)}
 &\leq C e^{-c x^2},\qquad
 \Pp(|x+G|\leq\alpha)
 +x\left|\partial_x\Pp(|x+G|\leq\alpha)\right|
 \leq C e^{-c x^2}.
\end{align*}
%For the first bound, $e(x,\alpha)$ is comparable to $x^{-1}\exp(-(x-\alpha)^2/2)$ for large $x$, and $|e_\alpha|+|xe_x|\leq C(1+x)^2e$. On bounded $x$-intervals, continuity and positivity of $e$ give the same bound. The second follows directly from the Gaussian density formula.
Furthermore, \eqref{eq:r1-critical-small-ball}
implies that, for each $c>0$,
\[
 \frac1{\tau^2}\E\exp\left(-\frac{cW^2}{\rho\tau^2}\right)
 =\int_0^\infty 2cx e^{-cx^2}
   \frac{\Pp(|W|\leq\tau\sqrt\rho\,x)}{\tau^2}\,dx
 \longrightarrow0.
\]
%The integrand tends to zero at each $x>0$ and is bounded by $Cx^3e^{-cx^2}$, since $\Pp(|W|\leq t)\leq Ct^2$ for all $t>0$. This justifies dominated convergence.
%The formula for $H_1$ and the bound for $e$ now give
As a result, we have
\begin{equation}
 \begin{aligned}
 0<h(\alpha)-H_1(\tau,\alpha)
 &\leq C\E\exp\left(-\frac{cW^2}{\rho\tau^2}\right)
 =o(\tau^2),\\
 \frac{|h'(\alpha)-\partial_\alpha H_1|
       +\tau|\partial_\tau H_1|}
      {\sqrt{h(\alpha)-H_1(\tau,\alpha)}}
 &\leq C\left(\E\exp\left(-\frac{cW^2}{\rho\tau^2}\right)\right)^{1/2}
 =o(\tau).
 \end{aligned}
 \label{eq:r1-critical-direct-estimates}
\end{equation}
%For the second inequality, differentiate the expectation defining $H_1$ and apply Cauchy--Schwarz with weight $e$; the absolute value of the $\tau$ derivative contributes $|xe_x|/\tau$, evaluated at $x=|W|/(\tau\sqrt\rho)$. The displayed Gaussian bounds also justify these differentiations under the expectation.
Since $h$ is three times continuously differentiable near
$\alpha_c$, with
$h(\alpha_c)=\delta_c$, $h'(\alpha_c)=0$ and
$h''(\alpha_c)=2\delta_c$, a Taylor expansion gives 
\begin{align*}
 h(\alpha)-\delta_c
 &=\delta_c(\alpha-\alpha_c)^2
   +O\bigl((\alpha-\alpha_c)^3\bigr),\qquad
 h'(\alpha)
 =2\delta_c(\alpha-\alpha_c)
   +O\bigl((\alpha-\alpha_c)^2\bigr),
\end{align*}
which combined with
\eqref{eq:r1-critical-direct-estimates} yields
\[
 \alpha-\alpha_c=o(\tau),\qquad
 \partial_\alpha H_1\sim2\delta_c(\alpha-\alpha_c)>0,
 \qquad
 \partial_\tau H_1=o(\alpha-\alpha_c).
\]
Since $\partial_\tau H_1<0$, $\tau(\alpha)$ is strictly increasing near $\alpha_c$ and it therefore has a local inverse $\alpha(\tau)$, satisfying
\[
 0<\alpha'(\tau)
 =-\frac{\partial_\tau H_1}{\partial_\alpha H_1}=o(1),
\]
as $\tau\downarrow0$.

Note that
\[
 \left|D_1-\rho-(1-\rho)\Pp(|G|>\alpha)\right|
 +\tau|\partial_\tau D_1|
 \leq C\E\exp\left(-\frac{cW^2}{\rho\tau^2}\right)
 =o(\tau^2),
\]
which implies that
\[
D_1=\rho+(1-\rho)\Pp(|G|>\alpha)+o(\tau^2),\qquad \partial_\tau D_1=o(\tau).
\]
Furthermore, $|\partial_\alpha D_1|\leq2/\sqrt{2\pi}$, since
$-\partial_\alpha D_1$ is an average of sums of two Gaussian densities.

We now evaluate $D_1$ and its partial derivatives at
$(\tau,\alpha(\tau))$.
As $\rho+(1-\rho)\Pp(|G|>\alpha_c)=\delta_c$ and
$\alpha(\tau)-\alpha_c=o(\tau)$, we have
$$1-D_1/\delta_c=o(\tau).$$
Differentiating $R_{\alpha(\tau)}=\tau(1-D_1/\delta_c)$
therefore gives
\[
 0<\frac{d}{d\tau}R_{\alpha(\tau)}
 =1-\frac{D_1}{\delta_c}
   -\frac{\tau}{\delta_c}
    \left(\partial_\tau D_1+
          \partial_\alpha D_1\,\alpha'(\tau)\right)
 =o(\tau).
\]
%Positivity follows from $D_1<\delta_c$ and the derivative signs established above. The three terms are respectively $o(\tau)$, $o(\tau^2)$ and $o(\tau)$, using $\alpha'(\tau)=o(1)$ for the last term.
Finally, as $\kappa=\alpha(\tau)$,
$\kappa\downarrow\kappa_1^\star=\alpha_c$ implies $\tau\downarrow0$.
%For every sufficiently small $\tau>0$, the chain rule cancels the common nonzero factor $\alpha'(\tau)$ in the derivative ratio. 
Using $\overline{\mathsf{MSE}}_1(\alpha(\tau))
=\delta_c\tau^2$, we therefore obtain
\[
 S_1
 =\lim_{\kappa\downarrow\kappa_1^\star}
 \frac{d\overline{\mathsf{MSE}}_1/d\kappa}{dR_\kappa/d\kappa}
 =\lim_{\tau\downarrow0}
 \frac{2\delta_c\tau}{\dfrac{d}{d\tau}R_{\alpha(\tau)}}
 =+\infty,
\]
which gives the desired result.
\end{proof}

\subsection{Exact interpolation hurts generalization for noisy data and $\ell_1$ regularization}
\label{app:proof-r1-negative-slope}

\begin{proposition}[$S_1< 0$ for $\ell_1$ regularization and $\sigma>0$]
\label{prop:r1-negative-slope}
Let $r=1$, $\delta\in(0,1)$, $\sigma>0$, and $0<\E B^2<\infty$. Then the limit $S_1$ in \eqref{eq:slope-definition} exists and is strictly negative. 
\end{proposition}

\begin{proof}
For every $\tau\geq\sigma$, the equation
$D_1(\tau,\alpha)=\Pp(|B/\tau+G|>\alpha)=\delta$
yields a unique positive $\alpha$, which is continuous in $\tau$.
Furthermore,
$H_1(\tau,\alpha)-\delta(1-\sigma^2/\tau^2)$ is positive at
$\tau=\sigma$. As $\tau\to\infty$, $\alpha$ converges to the
positive solution of $\Pp(|G|>\alpha)=\delta$, and the same
expression has a strictly negative limit by
Lemma~\ref{lem:null-prior-gap}. Hence, by the intermediate-value theorem, there exists a solution
$(\tau_{\mathrm{int}},\alpha_{\mathrm{int}})$ of
\eqref{eq:int-state-mse}, with $\tau_{\mathrm{int}}>\sigma$.

In the calculations below, we write
$(\tau,\alpha)=(\tau_{\mathrm{int}},\alpha_{\mathrm{int}})$
and recall $Z=B/\tau+G$ from \eqref{eq:sdes}.
Using~\eqref{eq:Hprime} and Gaussian integration by parts, we have
\begin{align*}
 \partial_\tau H_1-\frac{2\delta\sigma^2}{\tau^3}
 &=-\frac2{\tau^3}
   \left(\delta\sigma^2+
    \E\left[B^2\mathbf1_{\{|Z|\leq\alpha\}}\right]\right)<0,\\
 \partial_\alpha H_1
 &=-\frac2\alpha
   \left(\frac{\delta\sigma^2}{\tau^2}+
    \E\left[\frac B\tau Z
       \mathbf1_{\{|Z|\leq\alpha\}}\right]\right)<0.
\end{align*}
%For the second inequality, the expectation is nonnegative: conditional on $B$, the truncated mean
%$\E[Z\mathbf1_{\{|Z|\leq\alpha\}}\mid B]$ has the same sign
%as $B$, by symmetry and monotonicity of the Gaussian density.
By continuity of the partial derivatives, we can apply the implicit-function theorem to the first equation
in \eqref{eq:se1} and obtain a continuously differentiable $\tau(\alpha)$ satisfying %which gives a continuously differentiable local branch $\tau(\alpha)$ through the endpoint, with
\[
 \tau'(\alpha_{\mathrm{int}})
 =-\frac{\partial_\alpha H_1}
 {\partial_\tau H_1-2\delta\sigma^2/\tau^3}<0.
\]
Differentiating $D_1$ and using \eqref{eq:int-state-mse} gives %,
%at the endpoint,
\[
 \partial_\alpha D_1=\frac12\partial_\alpha H_1-\alpha\delta,
 \qquad
 \partial_\tau D_1
 =\frac12\left(\partial_\tau H_1
                 -\frac{2\delta\sigma^2}{\tau^3}\right)
  -\frac\alpha{2\tau}\partial_\alpha H_1.
\]
Consequently,
\[
 \left.\frac{d}{d\alpha}D_1(\tau(\alpha),\alpha)
 \right|_{\alpha=\alpha_{\mathrm{int}}}
 =-\alpha\delta
  -\frac\alpha{2\tau}\partial_\alpha H_1\,
    \tau'(\alpha_{\mathrm{int}})<0.
\]
For $\alpha>\alpha_{\mathrm{int}}$ sufficiently close to it,
we therefore have $D_1<\delta$. Since $\kappa=\alpha$ for $r=1$,
Lemma~\ref{thm:main-nonint} gives
\[
 R_\alpha=\tau(\alpha)
 \left(1-\frac{D_1(\tau(\alpha),\alpha)}\delta\right)>0,
 \qquad
 \overline{\mathsf{MSE}}_1(\alpha)
 =\delta\bigl(\tau(\alpha)^2-\sigma^2\bigr).
\]
As $\alpha\downarrow\alpha_{\mathrm{int}}$, $R_\alpha\to0$ and, following the same argument as for Theorem~\ref{thm:r1-critical-infinite-slope}, we obtain that $\kappa_1^\star=\alpha_{\mathrm{int}}$.
Moreover,
\[
 \left.\frac{d}{d\alpha}R_\alpha
 \right|_{\alpha=\alpha_{\mathrm{int}}}
 =\alpha\tau+\frac\alpha{2\delta}
   \partial_\alpha H_1\,\tau'(\alpha_{\mathrm{int}})>0.
\]
This allows us to conclude that
\[
 S_1
 =\lim_{\alpha\downarrow\alpha_{\mathrm{int}}}
 \frac{2\delta\tau(\alpha)\tau'(\alpha)}
      {\dfrac{d}{d\alpha}R_\alpha}
 =\frac{2\delta\tau_{\mathrm{int}}
          \tau'(\alpha_{\mathrm{int}})}
 {\left.\dfrac{d}{d\alpha}R_\alpha
    \right|_{\alpha=\alpha_{\mathrm{int}}}}<0,
\]
which proves the claim.
\end{proof}

\subsection{Exact interpolation hurts generalization for $\ell_2$ regularization}
\label{app:proof-r2-negative-slope}

\begin{proposition}[$S_2 \le 0$ for $\ell_2$ regularization]
\label{prop:r2-negative-slope} Let $r=2$,
$\delta\in(0,1)$, and $\sigma\geq0$.
Let $B\sim P_B$ have second moment $m:=\E B^2$, where
$0<m+\sigma^2<\infty$.
Then the limit $S_2$ in \eqref{eq:slope-definition} exists and
\begin{equation}
 S_2=-\frac{2\delta^2\sigma^2}
 {(1-\delta)^2\tau_{\mathrm{int}}}\leq0,
 \qquad
 \tau_{\mathrm{int}}^2
 =\frac{\sigma^2}{1-\delta}+\frac{1-\delta}{\delta}m.
 \label{eq:r2-residual-slope}
\end{equation}
Moreover, the limit
\[
 \lim_{\kappa\downarrow\kappa_2^\star}
 \frac{d}{d\kappa}\overline{\mathsf{MSE}}_2(\kappa)
\]
exists and is strictly negative when $\sigma>0$ and zero
when $\sigma=0$.
\end{proposition}

\begin{proof}
Since $\eta_{\alpha,2}(z)=z/(1+\alpha)$, we have
$D_2(\tau,\alpha)=1/(1+\alpha)$, so the interpolation
condition $D_2=\delta$ gives
$\alpha_{\mathrm{int}}=(1-\delta)/\delta$.
Substituting $\eta_{\alpha,2}(z)=z/(1+\alpha)$ into
the definitions of $H_2$ and $U_2$ in \eqref{eq:Hdef},
and then into the fixed-point equations \eqref{eq:se1}, gives
\begin{equation}
 \tau(\alpha)^2
 =\frac{\delta\sigma^2(1+\alpha)^2+m\alpha^2}
        {\delta(1+\alpha)^2-1},
 \qquad
 \kappa(\alpha)=\frac{\alpha}{1+\alpha}
 \sqrt{1+\frac m{\tau(\alpha)^2}}.
 \label{eq:r2-branch-parametrization}
\end{equation}
At $\alpha=\alpha_{\mathrm{int}}$, the formula for $\tau$ gives
\[
 \tau(\alpha_{\mathrm{int}})^2
 =\tau_{\mathrm{int}}^2
 =\frac{\sigma^2}{1-\delta}+\frac{1-\delta}{\delta}m>0,
\]
ensuring that $\tau(\alpha)$ and $\kappa(\alpha)$ are smooth near
$\alpha_{\mathrm{int}}$.
For $\alpha>\alpha_{\mathrm{int}}$ and sufficiently close to it,
$(\tau(\alpha),\alpha)$ solves \eqref{eq:se1} at
$\kappa=\kappa(\alpha)$ and satisfies $D_2<\delta$.
By Lemma~\ref{thm:main-nonint}, we thus have
\[
 R_{\kappa(\alpha)}
 =\tau(\alpha)\left(1-\frac1{\delta(1+\alpha)}\right)>0,
 \qquad
 \overline{\mathsf{MSE}}_2(\kappa(\alpha))
 =\delta\bigl(\tau(\alpha)^2-\sigma^2\bigr).
\]
Differentiating at $\alpha_{\mathrm{int}}$ gives
\[
 \tau'(\alpha_{\mathrm{int}})
 =-\frac{\delta^2\sigma^2}{(1-\delta)^2\tau_{\mathrm{int}}},
 \qquad
 \left.\frac{d}{d\alpha}R_{\kappa(\alpha)}
 \right|_{\alpha=\alpha_{\mathrm{int}}}
 =\delta\tau_{\mathrm{int}}>0.
\]
Since $\tau'(\alpha_{\mathrm{int}})\leq0$, differentiating
the formula for $\kappa$ also gives
$\kappa'(\alpha_{\mathrm{int}})
\geq\delta^2\sqrt{1+m/\tau_{\mathrm{int}}^2}>0$.
Thus, $\kappa(\alpha)$ is locally invertible and increasing.
As $\alpha\downarrow\alpha_{\mathrm{int}}$, $R_{\kappa(\alpha)}\to 0$ and, following the same argument as for Theorem~\ref{thm:r1-critical-infinite-slope}, we obtain that
$\kappa_2^\star=\kappa(\alpha_{\mathrm{int}})$. Consequently, $\kappa\downarrow\kappa_2^\star$ corresponds
to $\alpha\downarrow\alpha_{\mathrm{int}}$ and we have
\[
 S_2
 =\lim_{\alpha\downarrow\alpha_{\mathrm{int}}}
 \frac{2\delta\tau(\alpha)\tau'(\alpha)}
      {\dfrac{d}{d\alpha}R_{\kappa(\alpha)}}
 =-\frac{2\delta^2\sigma^2}
 {(1-\delta)^2\tau_{\mathrm{int}}}.
\]
As $\kappa\downarrow\kappa_2^\star$, the MSE derivative tends to
$2\delta\tau_{\mathrm{int}}\tau'(\alpha_{\mathrm{int}})
/\kappa'(\alpha_{\mathrm{int}})$, which is strictly negative
when $\sigma>0$ and zero when $\sigma=0$, thus concluding the argument.
\end{proof}

\subsection{Proof of Theorem~\ref{prop:sparsity-monotonicity}}
\label{app:proof-sparsity-monotonicity}

\begin{proof}
Let $s=1/(r-1)>1$, define
\[
 F(x):=\E\left[(x+G)|x+G|^{s-1}\right],
\]
and set
\begin{equation}
 \mathcal S(\lambda)
 :=\frac{2}{\E|G|^{s+1}}
 \left(
  \sqrt\lambda\,\E\left[WF\left(\frac W{\sqrt\lambda}\right)\right]
  -\frac{\E|G|^{2s}}{\E|G|^{s+1}}
 \right).
 \label{eq:sparsity-limiting-slope}
\end{equation}
As in the proof of Theorem~\ref{thm:sparse-decreasing-slope},
we rescale the interpolation parameters and proximal function as
\[
 t=\sqrt\delta\,\tau,\qquad
 a=\delta^{r-1}\alpha,\qquad
 v_\delta(z)=\delta^{-1}\eta_{a\delta^{1-r},r}(z).
\]
Here $r$ is fixed, and the partial derivatives in the expansions
below are taken with respect to $(t,a,\lambda)$.
Differentiation in $\lambda$ has the same polynomial growth
as differentiation in $t$, and it introduces no logarithmic factors.
Thus, using \eqref{eq:sparse-uniform-prox-bound},
\eqref{eq:sparse-v-derivatives}, and
\eqref{eq:sparse-quadratic-cutoff}, the same argument gives
\eqref{eq:sparse-rescaled-system-limit} in $C^2(t,a,\lambda)$,
uniformly on compact positive parameter sets.
Indeed, after integration over $G$, the bounds for the linear
terms and their derivatives are $C(1+|W|^{s+1})$.
The cutoff \eqref{eq:sparse-quadratic-cutoff} similarly controls
$\delta v_\delta^2$ and its derivatives on $\{\xi_\rho=1\}$,
whose expectations over $W,G$ tend to zero by dominated convergence.
%Hence the assumption $\E|W|^{s+1}<\infty$ suffices.

The argument based on the implicit function theorem which leads to
\eqref{eq:sparse-boundary-limit} gives a positive solution
$(\tau_{\mathrm{int}},\alpha_{\mathrm{int}})$ of the
fixed point equations in \eqref{eq:int-state-mse}, namely
\[
 \tau_{\mathrm{int}}^2
 =\sigma^2+\frac{\tau_{\mathrm{int}}^2}{\delta}
 H_r(\tau_{\mathrm{int}},\alpha_{\mathrm{int}}),
 \qquad
 D_r(\tau_{\mathrm{int}},\alpha_{\mathrm{int}})=\delta,
\]
also when $\sigma=0$.
By Lemma~\ref{prop:branch-attachment}, this solution is unique
and $(\tau_\kappa,\alpha_\kappa)\to
(\tau_{\mathrm{int}},\alpha_{\mathrm{int}})$ as
$\kappa\downarrow\kappa_r^\star$.
Substituting into \eqref{eq:mse-slope-formula}, as in the
derivation of \eqref{eq:sparse-fixed-lambda-limit}, gives
\begin{equation}
 \delta^{-3/2}S_r
 =\mathcal S(\lambda)+o_{C^1}(1),
 \label{eq:sparsity-lambda-C1-limit}
\end{equation}
as $\delta\downarrow0$ with $\rho=\lambda\delta$ and $\sigma$
fixed. Here the remainder and its first derivative in $\lambda$
converge uniformly to zero on $[\lambda_-,\lambda_+]$.

We next show that $\mathcal S'(\lambda)<0$ for every $\lambda>0$.
By symmetry, $F$ is odd, so
\begin{equation}
 \sqrt\lambda\,\E\left[WF\left(\frac W{\sqrt\lambda}\right)\right]
 =\E\left[
   W^2\frac{F(|W|/\sqrt\lambda)}{|W|/\sqrt\lambda}
 \right],
 \label{eq:sparsity-size-biased-identity}
\end{equation}
where the ratio at zero is defined as $F'(0)$.
Moreover, $F'(x)=s\E|x+G|^{s-1}$ is strictly increasing for $x>0$:
for $x,u>0$,
\[
 \frac{\partial}{\partial x}\Pp(|x+G|>u)
 =\frac{e^{-(u-x)^2/2}-e^{-(u+x)^2/2}}{\sqrt{2\pi}}>0.
\]
Since $F(0)=0$, it follows that
\begin{equation}
 F'(x)-\frac{F(x)}x
 =F'(x)-\frac1x\int_0^xF'(u)\,du>0,
 \qquad x>0.
 \label{eq:sparsity-positive-bracket}
\end{equation}
The bound $|F'(x)|+|F(x)|/x\leq C(1+x^{s-1})$
and $\E|W|^{s+1}<\infty$ justify differentiation under the
expectation in \eqref{eq:sparsity-size-biased-identity}, yielding
\begin{equation}
 \mathcal S'(\lambda)
 =-\frac1{\lambda\E|G|^{s+1}}
 \E\left[
 W^2\left(
 F'\left(\frac{|W|}{\sqrt\lambda}\right)
 -\frac{F(|W|/\sqrt\lambda)}{|W|/\sqrt\lambda}
 \right)\right]<0,
 \label{eq:sparsity-limiting-derivative}
\end{equation}
where the term in brackets is defined as zero when $W=0$.
The strict inequality follows from \eqref{eq:sparsity-positive-bracket}
and $\E W^2=1$, and the continuity of $\mathcal S'$ follows by dominated convergence.

For fixed $\delta$, \eqref{eq:sparsity-lambda-C1-limit}
and $\partial_\rho=\delta^{-1}\partial_\lambda$ give
\[
 \frac{\partial S_r}{\partial\rho}
 =\delta^{1/2}\left\{\mathcal S'(\rho/\delta)+o(1)\right\},
 \qquad
 \rho\in[\lambda_-\delta,\lambda_+\delta].
\]
Since $\mathcal S'$ has a strictly negative maximum on
$[\lambda_-,\lambda_+]$, this derivative is negative throughout
the required interval for all sufficiently small $\delta$.
Finally, the same argument used for \eqref{eq:sparse-slope-asymptotic} gives
\[
 \lim_{\lambda\downarrow0}
 \lambda^{(s-1)/2}\mathcal S(\lambda)
 =\frac{2\E|W|^{s+1}}{\E|G|^{s+1}}>0.
\]
Let us choose $\lambda_c=\lambda_c(r,W)>0$ so that
$\mathcal S(\lambda)>0$ for $0<\lambda<\lambda_c$.
If $\lambda_+<\lambda_c$, then
$\mathcal S(\lambda)\geq\mathcal S(\lambda_+)>0$
on $[\lambda_-,\lambda_+]$.
The uniform convergence in \eqref{eq:sparsity-lambda-C1-limit}
therefore gives $S_r>0$ throughout the required interval, concluding the argument. 
%All choices of $\delta_0$ depend only on $r,W,\lambda_-,\lambda_+,\sigma$, and we choose $\delta_0<\min\{1,1/\lambda_+\}$ to ensure $0<\rho<1$.
\end{proof}

\subsection{Exact interpolation hurts generalization without structure}
\label{app:proof-gaussian-negative-slope}

\begin{proposition}[$S_r<0$ for Gaussian $B$]
\label{prop:gaussian-negative-slope}
Let 
$\delta\in(0,1)$ and $\sigma\geq0$.
Suppose that $B\sim N(0,m)$ for some $m>0$.
Then, for every $1\leq r<2$, the limits below exist,
are finite, and satisfy
\[
 \lim_{\kappa\downarrow\kappa_r^\star}
 \frac{d}{d\kappa}\overline{\mathsf{MSE}}_r(\kappa)<0,
 \qquad
 S_r<0,
\]
where $S_r$ is defined in \eqref{eq:slope-definition}.
\end{proposition}

\begin{proof}
Fix $1\leq r<2$ and write $v_a:=\eta_{a,r}(G)$ for $a>0$,
with $G\sim N(0,1)$. From the proximal characterization, we have  that
$a\mapsto\E[Gv_a]$ is continuously differentiable with strictly
negative derivative, decreasing from $1$ to $0$.
Thus, there is a unique $a_*>0$ such that
\[
 \E[Gv_{a_*}]=\E\eta'_{a_*,r}(G)=\delta,
 \qquad
 \delta-\E v_{a_*}^2=a_*\E|v_{a_*}|^r>0,
\]
where the first equality follows from Gaussian integration by parts
and the second identity  from the proximal equation.
For $a$ near $a_*$, define
\begin{align}
 \tau(a)^2
 &=\frac{m\E(G-v_a)^2+\delta\sigma^2}{\delta-\E v_a^2},
\qquad \alpha(a)=a\left(1+\frac m{\tau(a)^2}\right)^{(2-r)/2},\notag\\ \kappa(a)
 &=a\sqrt{1+\frac m{\tau(a)^2}}
       \bigl(\E|v_a|^r\bigr)^{(r-1)/r}.
 \label{eq:gaussian-slope-parametrization}
\end{align}
These functions are positive and continuously differentiable. %
%the denominator defining $\tau^2$ is positive at $a_*$, and
%$m\E(G-v_{a_*})^2>0$.
From the definition of $\alpha(a)$, we readily have
\[
 \eta_{\alpha(a),r}\left(\sqrt{1+\frac m{\tau(a)^2}}\,G\right)
 =\sqrt{1+\frac m{\tau(a)^2}}\,v_a.
\]
Since $B/\tau+G$ is Gaussian with variance $1+m/\tau^2$ and
$\E[B/\tau\mid B/\tau+G]=m(B/\tau+G)/(m+\tau^2)$, we have
\[
 \begin{aligned}
 D_r(\tau(a),\alpha(a))&=\E[Gv_a],\\
 H_r(\tau(a),\alpha(a))
 &=\E v_a^2+\frac m{\tau(a)^2}\E(G-v_a)^2,\\
 U_r(\tau(a),\alpha(a))
 &=\left(1+\frac m{\tau(a)^2}\right)^{r/2}\E|v_a|^r.
 \end{aligned}
\]
Consequently, $(\tau(a),\alpha(a))$ solves the fixed-point
equations \eqref{eq:se1} at $\kappa=\kappa(a)$.
For $a>a_*$ and sufficiently close to it, $D_r<\delta$, so
Lemma~\ref{thm:main-nonint} gives
\[
 R_{\kappa(a)}=\tau(a)\left(1-\frac{\E[Gv_a]}\delta\right)>0,
 \qquad
 \overline{\mathsf{MSE}}_r(\kappa(a))
 =\delta\bigl(\tau(a)^2-\sigma^2\bigr).
\]

We next prove that $\tau'(a_*)<0$.
Suppose first that $1<r<2$. At $a=a_*$, \eqref{eq:prox-eq} and
$\E[Gv_{a_*}]=\delta$ imply
\[
 \delta-\E v_{a_*}^2=a_*\E|v_{a_*}|^r,
 \qquad
 \E(G-v_{a_*})^2=a_*^2\E|v_{a_*}|^{2r-2}.
\]
Differentiating the formula for $\tau(a)$ in
\eqref{eq:gaussian-slope-parametrization}, then using these
identities and $\sigma^2\geq0$, gives
\begin{equation}
 \frac{\tau'(a_*)}{\tau(a_*)}
 \leq\frac1{a_*}\left(
 \frac{\E[|v_{a_*}|^{2r-2}\eta'_{a_*,r}(G)]}
      {\E|v_{a_*}|^{2r-2}}
 -
 \frac{\E[|v_{a_*}|^r\eta'_{a_*,r}(G)]}
      {\E|v_{a_*}|^r}
 \right)<0.
 \label{eq:gaussian-tau-derivative}
\end{equation}
To justify the strict inequality, observe that
\[
 \eta'_{a,r}(G)=\frac1{1+a(r-1)|v_a|^{r-2}}.
\]
Under the probability distribution obtained by weighting
the law of $|v_a|$ by $|v_a|^{2r-2}$, the functions
$x^{2-r}$ and $1/(1+a(r-1)x^{r-2})$ have strictly positive
covariance, since both are strictly increasing and the distribution
is nondegenerate. This is precisely the strict inequality
between the two weighted averages in
\eqref{eq:gaussian-tau-derivative}.

For $r=1$, we have
$v_a=\operatorname{sign}(G)(|G|-a)_+$ and
$\E[Gv_a]=\Pp(|G|>a)$.
Differentiating the formula for $\tau(a)$ at $a_*$ gives
\[
 \tau(a_*)\tau'(a_*)
 =-\frac{m\E[G^2\mathbf1_{\{|G|\leq a_*\}}]+\delta\sigma^2}
 {a_*^2\E(|G|-a_*)_+}<0,
\]
where the inequality is strict because $m>0$ and $a_*>0$.

For every $1\leq r<2$, by computing the derivative, we have
\[
 \frac{d}{da}\left[a\bigl(\E|v_a|^r\bigr)^{(r-1)/r}\right]
 =\bigl(\E|v_a|^r\bigr)^{-1/r}
   \E[|v_a|^r\eta'_{a,r}(G)]>0.
\]
Together with $\tau'(a_*)<0$, this implies
$\kappa'(a_*)>0$ in \eqref{eq:gaussian-slope-parametrization}.
Thus, $\kappa(a)$ is locally invertible and increasing.
As $a\downarrow a_*$, $R_{\kappa(a)}\to0$ and, following the same argument as for Theorem~\ref{thm:r1-critical-infinite-slope}, we obtain
$\kappa_r^\star=\kappa(a_*)$.
At $a_*$, the fixed point equations in \eqref{eq:int-state-mse}
also hold, so the limiting MSE of the MNI equals $\delta(\tau(a_*)^2-\sigma^2)$ by Lemma~\ref{thm:min-norm-interpolator}. Finally,
\[
 \left.\frac{d}{da}R_{\kappa(a)}\right|_{a=a_*}
 =-\frac{\tau(a_*)}\delta
   \left.\frac{d}{da}\E[Gv_a]\right|_{a=a_*}>0.
\]
Since $\kappa\downarrow\kappa_r^\star$ corresponds to
$a\downarrow a_*$, the chain rule gives
\[
 \lim_{\kappa\downarrow\kappa_r^\star}
 \frac{d}{d\kappa}\overline{\mathsf{MSE}}_r(\kappa)
 =\frac{2\delta\tau(a_*)\tau'(a_*)}{\kappa'(a_*)}<0,
 \qquad
 S_r=\frac{2\delta\tau(a_*)\tau'(a_*)}
 {\left.\dfrac{d}{da}R_{\kappa(a)}\right|_{a=a_*}}<0,
\]
which proves the desired result.
\end{proof}

\section{Additional numerical experiments and implementation details}
\label{app:Experiment}

\subsection{Early-stopped gradient descent as $Q_k$-regularized regression}

\citet{azulay2021implicit} show that, when gradient flow on a diagonal linear network converges to an interpolating solution $\hat\beta(\infty)$, this solution satisfies $\hat\beta(\infty)=\argmin_{X\beta=y}Q_k(\beta)$, where $Q_k$ is the initialization-dependent regularizer defined in Section~\ref{sec:Numerical_Experiments}.
Here, we show experimentally that this correspondence extends beyond interpolation: the early-stopped iterates $\hat\beta(t)$ closely track the solutions of
\[
 \min_{\beta\in\R^p}\ \|y-X\beta\|_2+\frac{\kappa}{\sqrt n}\,Q_k(\beta)
\]
along the regularization path, with larger stopping times corresponding to smaller $\kappa$.

\paragraph{Experimental Setup.}
We draw $\beta\in\R^p$ from the three-point prior of Section~\ref{sec:Numerical_Experiments} with $\rho=0.1$, i.e., $B$ follows \eqref{eq:sparse-prior-family} with $W$ a Rademacher random variable.
The design has i.i.d.\ entries $X_{ij}\sim\mathcal N(0,1/n)$, and $y=X\beta+\sigma z$ with $z\sim\mathcal N(0,I_n)$, where $n=2000$, $p=4000$, and $\sigma=0.03$.
We run paired simulations of the two methods described below using the same sampled prior and design matrix, and repeat this procedure 15 times.

\paragraph{Gradient descent.}
We train the two-layer diagonal linear networks $\beta = u_+\odot v_+ - u_-\odot v_-$ with the square-root loss $\|y-X\beta\|_2/\sqrt n$.
All four weight vectors are initialized at the same constant,
$u_+(0)=u_-(0)=v_+(0)=v_-(0)=\tfrac{k^{1/4}}{2}\mathbf 1$.
The induced regularizer is then $Q_k$, and we sweep $k\in\{0.01,0.1,1,10,100\}$.
We use a fixed step size $\eta=0.01/\|X\|_2^2$.
Time is measured in gradient-flow units, $t=\eta\times(\text{number of iterations})$.

\paragraph{Explicit regularization.}
For each $k$, we solve the $Q_k$-regularized program above with CVXPY at ten values of $\kappa$.
These are fixed multiples $\{0.5,0.95,1,1.02,1.05,1.1,1.2,1.4,1.7,2.2\}$ of the interpolation threshold $\kappa_*(k)$ predicted by state evolution; for $\kappa<\kappa_*$, the estimator interpolates the data.

\paragraph{Results.}
In \Cref{fig:azulay_early_stop}, we overlay the gradient-descent trajectories and the solutions to the $Q_k$-regularized regression problems, plotting parameter MSE as a function of training error.
Each solid line is a gradient-descent trajectory for one value of $k$, averaged over seeds, with a standard-deviation band.
Training starts at $\hat\beta(0)=0$ on the right and ends at the interpolating solution on the left, so time flows from right to left.
Each square represents one explicitly regularized solution.
The star at the left edge, labeled ``Azulay-norm SE plateau'', marks the interpolating branch prediction from Lemma \ref{thm:min-norm-interpolator}.

For every $k$, the gradient-descent curve passes through the squares: stopping gradient descent at a given training error yields a parameter MSE close to that obtained by solving the explicit program at the value of $\kappa$ producing the same training error.
To make the correspondence between stopping time and regularization strength explicit, we annotate each square on the $k=100$ curve with its $\kappa$ as a multiple of $\kappa_*$ and with the time $t$ at which gradient descent reaches the same training error.

\begin{figure}[t]
\centering
\includegraphics[width=0.8\linewidth]{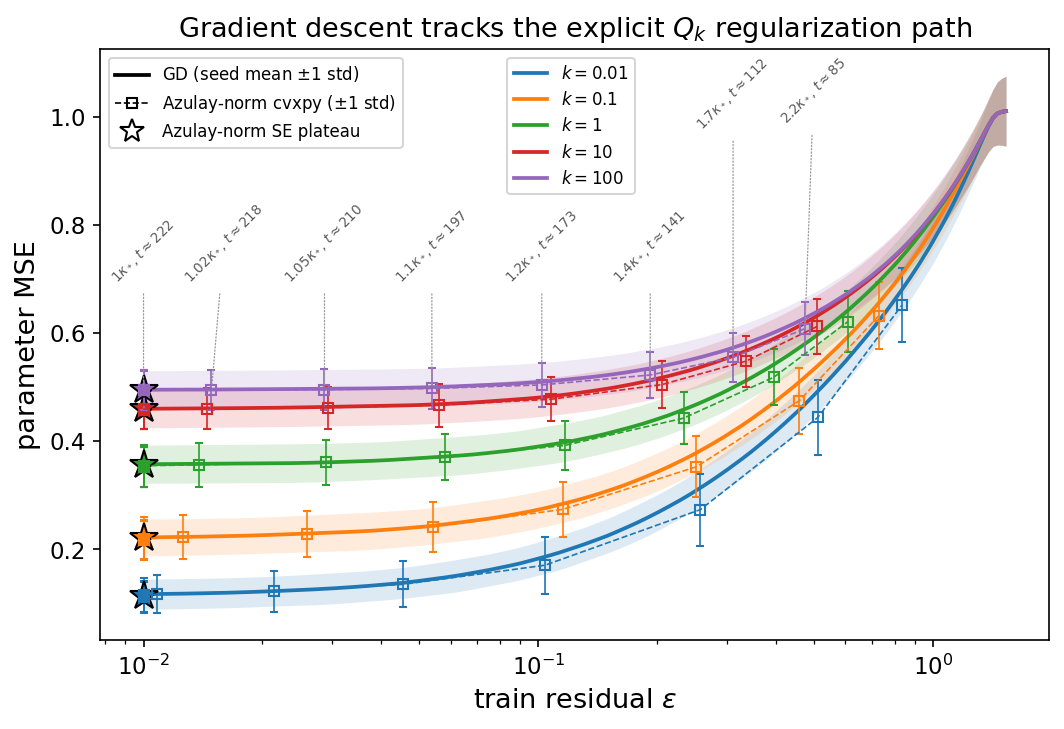}
\caption{Early-stopped gradient descent on diagonal linear networks (solid lines) versus $Q_k$-regularized regression solved with CVXPY (squares), for several initialization scales $k$. Both are plotted as parameter MSE against training error; time along the gradient-descent trajectories flows from right to left. Stars mark the state-evolution prediction for the minimum-$Q_k$ interpolator, the common endpoint of both methods. For $k=100$, each square is annotated with its $\kappa$ (as a multiple of the interpolation threshold $\kappa_*$) and the gradient-descent time $t$ at which it reaches the same training error.}
\label{fig:azulay_early_stop}
\end{figure}

\subsection{Transformer on a modular arithmetic task}

% Requires \usepackage{amsmath,amssymb}

\paragraph{Architecture.}
We use a causal decoder-only transformer following the setup of \cite{power2022grokking}, with one transformer block in place of two.
Each block applies pre-layer normalization followed by causal multi-head self-attention with four heads and a residual connection, then pre-layer normalization followed by an MLP with another residual
connection. The attention mask excludes future positions
(the strictly upper-triangular entries of the attention matrix).
The MLP has the form $\operatorname{Linear}(128 \to 512)
\;\to\;
\operatorname{GELU}
\;\to\;
\operatorname{Linear}(512 \to 128).$
Token embeddings and learned positional embeddings each have dimension \(128\). A final layer normalization precedes a bias-free linear output head over a vocabulary of \(99\) tokens: \(97\) residues, \texttt{<op>}, and \texttt{<=>}.
The model has \(224{,}512\) trainable parameters, of which
\(13{,}312\) (\(5.9\%\)) belong to the embeddings.

\paragraph{Task.}
We consider modular arithmetic over \(\mathbb{Z}/p\mathbb{Z}\)
with \(p=97\), and compute $c = xy \bmod p$. Each input is the token sequence $x \;\; \texttt{<op>} \;\; y \;\; \texttt{<=>}$, 
the model predicts the target \(c\) at the final position, and we evaluate cross-entropy loss and accuracy only at this position. We split the
\(p(p-1)=9{,}312\) pairs equally between training and validation, with \(4{,}656\) pairs in each split.

\paragraph{Optimizer and training.}
We use AdamW with cross-entropy loss, base learning rate \(10^{-3}\), weight decay \(1\),
\(\beta_1=0.9\), \(\beta_2=0.98\), and batch size \(512\).
The learning rate is warmed up linearly over the first \(10\) optimizer steps. Training runs for at most \(3\times10^5\) optimizer steps, with early stopping when validation accuracy exceeds \(0.99\).
Each configuration uses five seeds. Within each repetition, the same seed is used for the data split
and model initialization.

\paragraph{Output-scale parametrization.}
Following the rich/lazy output-scale axis of \cite{kunin2024get}, we multiply the logits by
\(\alpha\) immediately before computing the loss: $\widetilde{\mathbf{z}} = \alpha\mathbf{z}.$
The learning rate is adjusted according to
$\eta(\alpha) = \frac{10^{-3}}{\alpha^2}.$
We refer to this as the original recipe and we run it for $\alpha>1$. This recipe compensates the
learning rate but leaves the weight-decay coefficient unchanged.
Consequently, smaller \(\alpha\) is associated with both richer
training dynamics and stronger effective weight decay.

To control for this confounder, for $\alpha\le1$, we run with a compensated weight-decay coefficient $\lambda(\alpha)=\alpha^2$, while keeping the learning rate at $\eta(\alpha)=10^{-3}/\alpha^2$ as before. This keeps  $\eta(\alpha)\lambda(\alpha)=10^{-3}$ equal to its value at $\alpha=1$, allowing the output-scale effect to be examined at fixed effective decay.

\paragraph{Training and generalization error curves.} We use the validation cross-entropy loss at its final recorded epoch $T$ as an empirical proxy for $\mathcal{L}(\hat\beta(\infty))$, $\mathcal L$ being the cross-entropy loss. The residual loss is therefore zero at epoch $T$. Let $a^{\mathrm{train}}_{t}$ denote the training accuracy at epoch $t$. We define the training error as $\epsilon_{t}:=1-\max_{t'\le t}a^{\mathrm{train}}_{t'}$. When $\epsilon_{t}$ reaches the tolerance value of $0.01$ for some $t=t_1$, we set it to $0$ for all remaining epochs $t>t_1$.
\end{document}